\documentclass[11pt,letterpaper]{article}
\usepackage[margin=1in,top=1.05in,bottom=1.05in,includefoot]{geometry}

\usepackage[T1]{fontenc}
\usepackage{lmodern}

\usepackage{amsmath,amsfonts,bm}

\def\eqref#1{equation~\ref{#1}}

\def\1{\bm{1}}

\DeclareMathAlphabet{\mathsfit}{\encodingdefault}{\sfdefault}{m}{sl}
\SetMathAlphabet{\mathsfit}{bold}{\encodingdefault}{\sfdefault}{bx}{n}

\usepackage{url}
\usepackage{booktabs}

\usepackage{graphicx}
\usepackage[authoryear,round]{natbib}
\usepackage{caption}
\DeclareCaptionFont{tabletext}{\fontfamily{ptm}\fontsize{10}{12}\selectfont}
\usepackage{amsmath,amssymb,amsthm}

\theoremstyle{plain}  
\newtheorem{theorem}{Theorem}

\newtheorem{proposition}[theorem]{Proposition}
\newtheorem{corollary}[theorem]{Corollary}

\usepackage{amsfonts}
\usepackage{multirow}
\usepackage{array}
\usepackage{amsmath}
\usepackage{float}

\title{Beyond Static Positional Encoding: Bayesian Time-aware Attention for Non-Stationary Long-term Time Series Forecasting}

\author{Xinjie Lan$^{1}$ \quad Yunbo Mei$^{1}$ \quad Leheng Zhou$^{2}$ \quad Hao Qi$^{3}$\\[0.5em]
{\small $^{1}$Rice University}\\
{\small $^{2}$University of Hong Kong}\\
{\small $^{3}$University of North Carolina at Charlotte}}
\date{}
\usepackage{hyperref}
\hypersetup{hidelinks, pdfauthor={Xinjie Lan, Yunbo Mei, Leheng Zhou, Hao Qi}, pdftitle={Beyond Static Positional Encoding: Bayesian Time-aware Attention for Non-Stationary Long-term Time Series Forecasting}}

\usepackage{longtable}
\begin{document}

\maketitle

\begin{abstract}
Transformers have demonstrated strong performance in long-term time series forecasting (LTSF), yet modeling non-stationary temporal dependencies remains challenging. In this paper, we interpret self-attention as a Bayesian posterior over attended positions, combining a content-based prior induced by scaled dot-product similarity with a temporal-evidence likelihood induced by positional encoding. This formulation provides a probabilistic interpretation of how positional encoding shapes the temporal dependencies captured by self-attention. Building on this perspective, we derive a population-level evidence lower bound (ELBO) that characterizes the role of positional encoding in forecasting weakly stationary time series. Our analysis further highlights a limitation of static positional encoding: when temporal dependence structures vary across input windows, a fixed temporal bias must accommodate heterogeneous dynamics and can therefore be suboptimal for non-stationary forecasting. Motivated by this analysis, we propose Bayesian Time-aware Attention (BTA), which employs a mixture-of-experts (MoE) router to adaptively combine temporal kernels and generate input-conditioned, head-specific positional biases. Extensive experiments on standard LTSF benchmarks demonstrate that BTA consistently improves forecasting performance across multiple Transformer backbones and achieves competitive overall performance with modest computational overhead and stable training.
\end{abstract}

\section{Introduction}

Long-term time series forecasting (LTSF) aims to predict future trajectories over extended horizons from historical observations \citep{chen2023long}.
Transformers have emerged as a prominent architecture for LTSF largely because self-attention \citep{vaswani2017attention} enables the modeling of long-range temporal dependencies \citep{wen2022transformers}.
However, real-world time series often exhibit non-stationarity \citep{kim2021reversible}, including evolving trends \citep{liu2025timebridge}, time-varying seasonality \citep{sun2026penguin}, and distribution shifts \citep{li2026apt}.
Such changes can alter which historical observations are informative for future prediction, making it challenging for Transformer-based models to capture temporal dependencies that vary across historical input windows.

Existing approaches to non-stationary forecasting broadly follow two strategies: normalization and architectural adaptation.
First, normalization methods, such as RevIN \citep{kim2021reversible}, standardize historical input windows to mitigate shifts in marginal statistics and restore the original scale and location to the predictions.
However, local statistical normalization may be insufficient to address distribution shifts \citep{li2026apt}; even after normalization, windows with similar marginal statistics can exhibit different temporal dependence structures \citep{liupulse}.
Second, architectural approaches introduce specialized mechanisms to capture evolving temporal patterns, as exemplified by Non-stationary Transformer \citep{liu2022non}, TimeBridge \citep{liu2025timebridge}, and DTAF \citep{lu2026towards}.
While these approaches improve temporal modeling, architecture-specific designs can introduce additional complexity and limit their applicability across different Transformer backbones.
These limitations motivate a broadly applicable mechanism that enables Transformers to adapt to changing temporal dependencies without extensive architectural changes.

Beyond these methodological advances, a fundamental question remains: how does a Transformer forecaster model changing temporal dependencies?
Without positional information, canonical self-attention is permutation-equivariant \citep{zeng2023transformers} and therefore cannot distinguish temporal positions from content alone.
Positional encoding (PE) addresses this limitation by introducing information about token positions or relative distances \citep{shaw2018self}, enabling self-attention to account for temporal order.
This motivates a more specific question: \emph{how does PE interact with self-attention to support forecasting when temporal dependencies vary across historical input windows?}
Answering this question can clarify the temporal modeling behavior of Transformers under non-stationarity and provide a principled foundation for designing adaptive positional mechanisms.

In this paper, we address this question from a Bayesian perspective.
We first interpret position-aware self-attention as a Bayesian posterior over attended positions: content-only scaled dot-product attention defines a prior over historical keys, while PE provides temporal-evidence likelihood that updates this prior according to temporal relationships.
Building on this interpretation, we derive a population-level evidence lower bound (ELBO) \citep{bishop2006pattern} that characterizes the effect of PE on forecasting weakly-stationary time series.
To extend the analysis to non-stationarity, we introduce a regime-conditioned
weakly-stationary formulation and show that a static PE, which induces shared
temporal evidence across input windows, can be strictly suboptimal when
heterogeneous regimes induce incompatible regime-specific optima.
This theoretical result motivates positional mechanisms that adapt their temporal evidence to each historical input window.

Motivated by our theoretical analysis, we propose \emph{Bayesian Time-aware Attention} (BTA), a plug-and-play module that learns input-conditioned, head-specific temporal biases over relative lags. BTA aggregates attention scores by lag, augments them with Fourier features, and uses a mixture-of-experts (MoE) \citep{cai2025survey} router to combine periodic and non-periodic temporal kernels. This enables different heads and input windows to capture distinct local, periodic, and long-range dependencies. Extensive experiments show that BTA generalizes across conventional Transformers \citep{zhou2021informer,nie2023time} and time-series foundation models \citep{kottapalli2025foundation}, achieving competitive performance on non-stationary LTSF with high computational efficiency.

Our main contributions are summarized as follows:
\begin{itemize}
    \item \textbf{Theory.} We interpret position-aware self-attention from a Bayesian perspective and derive a population-level ELBO that characterizes the role of PE in time-series forecasting, proving that static PE can be suboptimal under changing temporal dependencies.

    \item \textbf{Methodology.} We propose \emph{Bayesian Time-aware Attention} (BTA), a plug-and-play module that learns input-conditioned, head-specific temporal biases through an MoE router combining complementary temporal kernels.

    \item \textbf{Experiments.} Extensive experiments across conventional Transformers and time-series foundation models demonstrate that BTA consistently improves non-stationary LTSF while maintaining computational efficiency.
\end{itemize}

\section{Related Work}

\textbf{Non-stationary TSF methods} can be broadly organized into two directions: representation-level stationarization and model-level adaptation. Representation-level methods reduce distributional variation through statistical normalization \citep{kim2021reversible,fan2023dish,liu2023adaptive}, frequency-domain processing \citep{ye2024frequency,piao2024frednormer}, or learned transformations \citep{fan2025inflow,liupulse,li2026apt}. Model-level methods instead modify the forecasting mechanism to capture evolving distributions, changing dynamics, and heterogeneous temporal patterns \citep{du2021adarnn,duan2022combating,liu2024dynaconf,wang2023koopman,pham2023learning,zhang2023onenet,qiu2025duet,sun2025learning}. In Transformer-based forecasting, these methods adapt attention to distribution shifts while preserving non-stationary information that is predictive of future behavior. The Non-stationary Transformer framework injects learned de-stationary factors into attention \citep{liu2022non}, whereas GBT decomposes forecasting into two Transformer stages to reduce discrepancies between historical and future distributions \citep{shen2023gbt}. Distribution-routing methods such as DUET and TFPS assign heterogeneous temporal patterns to specialized extractors or experts \citep{qiu2025duet,sun2025learning}. Attention-oriented architectures such as TimeBridge and TimeStacker separately model stable, cointegrated, or multilevel temporal dependencies \citep{liu2025timebridge,liu2025timestacker}, while DTAF combines temporal expert filtering with frequency differencing \citep{lu2026towards}. BTA follows this model-level direction but serves as a plug-and-play attention module with an explicit probabilistic interpretation, allowing it to enhance existing Transformers without redesigning their forecasting backbones.

\textbf{Positional encoding} is commonly introduced to represent the order and temporal distance between tokens because self-attention is permutation-equivariant \citep{vaswani2017attention}. Representative relative positional encoding methods in NLP include RPE \citep{shaw2018self}, Transformer-XL \citep{dai2019transformerxl}, T5 relative bias \citep{raffel2020exploring}, ALiBi \citep{press2021train}, KERPLE \citep{chi2022kerple}, and RoPE \citep{su2024roformer}. In time-series modeling, eRPE and STTRE introduce relative embeddings tailored to temporal and spatio-temporal dependencies \citep{foumani2024improving,deihim2023sttre}, whereas Powerformer and Dozerformer impose power-law locality and structured local-periodic biases, respectively \citep{hegazy2025powerformer,zhang2024dozerformer}. These methods generally use predefined or globally learned positional structures shared across samples. Although PAREformer constructs adaptive positional representations by integrating positional encodings with seasonal--trend features \citep{tong2025pareformer}, BTA differs by directly constructing an input-conditioned relative attention bias as a mixture of probabilistic temporal kernels. Specifically, a mixture-of-experts router dynamically combines these kernels for each input series, enabling the positional bias to adapt to heterogeneous periodicities and temporal scales.

\section{Theory}

\subsection{Preliminaries and Notations}
\label{section:Preliminaries}

\paragraph{Weakly stationary time series (Gaussian process).}
Let $(\Omega,\mathcal{F},P)$ be a probability space and
$\{X_t\}_{t\in\mathbb{Z}}$ be a zero-mean weakly stationary Gaussian
process (GP),
$X(t)\sim\mathcal{GP}(0,k)$, with
$\mathbb{E}[X_t]=0$ and stationary covariance kernel
$k(t,t')=\Gamma(t-t')$, where $\Gamma(\tau)$ denotes the autocovariance
at lag $\tau = t - t'$. For forecasting at time $n$, define the historical and
future random vectors as
$\boldsymbol{X}^{-}
=[X_{n-T+1},\ldots,X_n]^\top\in\mathbb{R}^{T}$
and
$\boldsymbol{Y}
=[X_{n+1},\ldots,X_{n+L}]^\top\in\mathbb{R}^{L}$.
Let
$\boldsymbol{\Sigma}_{-}
=\operatorname{Cov}(\boldsymbol{X}^{-})$,
$\boldsymbol{\Sigma}_{+-}
=\operatorname{Cov}(\boldsymbol{Y},\boldsymbol{X}^{-})$, and
$\boldsymbol{\Sigma}_{+}
=\operatorname{Cov}(\boldsymbol{Y})$,
where $\boldsymbol{\Sigma}_{-}$ is assumed to be positive definite.
By stationarity, these covariance matrices are fully determined by
$\Gamma(\tau)$.

\textbf{Self-attention.} Based on the kernel interpretation
\citep{tsai2019transformer} of scaled dot-product
self-attention, the attention weight assigned
by the $i$th query
$\boldsymbol{q}_i\in\mathbb{R}^{d_k}$
to the $j$th key position is
\begin{equation}
\label{equ_attn_prob}
\alpha_{ij}
=
\operatorname{Softmax}_{j\in\mathcal{C}(i)}
\left(
\boldsymbol{q}_i^{\top}\boldsymbol{k}_j/\sqrt{d_k}
+
b_{ij}
\right),
\end{equation}
where $i\in\{1,\ldots,T\}$,
$\boldsymbol{k}_j\in\mathbb{R}^{d_k}$
denotes the $j$th key, and $b_{ij}$ denotes the PE-induced positional
score associated with the query--key pair $(i,j)$. Depending on the PE
mechanism, $b_{ij}$ may encode temporal order, relative lag, or other
predefined positional structure. For a fixed query position $i$, let
$\boldsymbol{b}_i=(b_{ij})_{j\in\mathcal{C}(i)}$, where
$\mathcal{C}(i)=\{1,\ldots,i-1\}$ denotes the set of accessible key
positions under strict causal attention; otherwise,
$\mathcal{C}(i)=\{1,\ldots,T\}$.
Without PE, i.e., $b_{ij}=0$, self-attention is permutation-equivariant
with respect to token order and therefore cannot distinguish temporal
positions from content alone
\citep{tsai2019transformer,kim2024self,irani2025positional}.

\paragraph{Simplified Transformer.}
Consider a point-wise Transformer consisting of a single self-attention
encoder followed by average pooling over the $T$ historical positions
and a linear forecasting decoder. Its output can be written as
\begin{equation}
\label{equ_simplified_transformer}
f_{\theta,\phi}
(\boldsymbol{x}^{-})
=
\boldsymbol{w}
\boldsymbol{a}_{\phi}^{\top}(\boldsymbol{x}^{-})\cdot
\boldsymbol{x}^{-}
=
\boldsymbol{M}_{\theta,\phi}(\boldsymbol{x}^{-})\cdot
\boldsymbol{x}^{-}
\in\mathbb{R}^{L},
\end{equation}
where $\boldsymbol{x}^{-}\in\mathbb{R}^{T}$ denotes the historical input,
$\boldsymbol{a}_{\phi}^{\top}(\boldsymbol{x}^{-})
\in\mathbb{R}^{1\times T}$ denotes the averaged attention weights induced
by PE method $\phi$, and
$
\boldsymbol{M}_{\theta,\phi}(\boldsymbol{x}^{-})
:=
\boldsymbol{w}
\boldsymbol{a}_{\phi}^{\top}(\boldsymbol{x}^{-})
\in\mathbb{R}^{L\times T}
$
denotes the effective forecasting operator, with
$\boldsymbol{w}\in\mathbb{R}^{L\times 1}$ denoting the effective linear
decoder that absorbs the embedding and value projections. Although simplified, this architecture retains the core components of a Transformer---content-based self-attention, positional encoding, representation aggregation, and linear prediction---while remaining sufficiently tractable for theoretical analysis. 
A detailed specification of the simplified Transformer is provided in Appendix~\ref{app:simplifed_transformer}.

\subsection{Bayesian Explanations for Transformer in non-stationary LTSF}

\begin{proposition}[PE provides latent temporal evidence]
\label{proposition_position_encoding}
For a fixed query position $i$, let
$S_i\in\mathcal{C}(i)$ be a categorical random variable representing the
selected key position. For each candidate position
$j\in\mathcal{C}(i)$, let $Z_{ij}$ denote a latent random variable
representing the temporal evidence associated with the query--key pair
$(i,j)$, and let $z_{ij}$ denote a realization of $Z_{ij}$.
Let
$
\boldsymbol{Z}_i
=
\{Z_{ij}\}_{j\in\mathcal{C}(i)}
$
and
$
\boldsymbol{z}_i
=
\{z_{ij}\}_{j\in\mathcal{C}(i)}.
$
Assuming that the candidate-specific contribution to
$p(\boldsymbol{Z}_i=\boldsymbol{z}_i\mid S_i=j)$ is characterized by
$p(Z_{ij}=z_{ij}\mid S_i=j)$ up to a multiplicative factor independent
of $j$, the attention distribution in Equation~\ref{equ_attn_prob} admits the Bayesian interpretation
$\alpha_{ij}=p(S_i=j\mid\boldsymbol{Z}_i=\boldsymbol{z}_i)$,
\begin{equation}
\label{equ_attn_bayes_factorization}
p(S_i=j\mid\boldsymbol{Z}_i=\boldsymbol{z}_i)
\propto
\underbrace{
p(S_i=j)
}_{\text{content-based prior}}\cdot
\underbrace{
p(Z_{ij}=z_{ij}\mid S_i=j)
}_{\text{temporal-evidence likelihood}},
\end{equation}
where the content-based prior
$
p(S_i=j)
=
\operatorname{Softmax}_{j\in\mathcal{C}(i)}
\left(
\boldsymbol{q}_i^{\top}\boldsymbol{k}_j/\sqrt{d_k}
\right)
$
is induced by scaled dot-product self-attention without PE, and PE can be
interpreted as specifying a candidate-specific temporal-evidence
likelihood such that
$
b_{ij}\overset{c}{=}
\log p(Z_{ij}=z_{ij}\mid S_i=j),
$
where $\overset{c}{=}$ denotes equality up to an additive term independent
of $j$. The proof is in
Appendix~\ref{app:prob_exp_attention_with_PE}.
\end{proposition}

Proposition~\ref{proposition_position_encoding} provides a
probabilistic interpretation of representative PE mechanisms. In
particular, different PE mechanisms can be explained as different probabilistic assumptions of the temporal-evidence likelihood. Detailed probabilistic interpretations are provided in
Appendix~\ref{app:prob_exp_PE}. Throughout the following analysis, we call a PE \emph{static} if its
induced positional score $b_{ij}^{\phi}$ is independent of the historical
realization $\boldsymbol{x}^{-}$; this definition refers to the induced
attention bias rather than whether the PE parameters themselves are fixed.
Consequently,
$q_\phi(\boldsymbol{Z}\mid\boldsymbol{X}^{-})=q_\phi(\boldsymbol{Z})$.

\begin{theorem}[Population-level ELBO of PE on weakly-stationary time series]
\label{theorem_population_elbo}
Let $\mathcal{D}=p(\boldsymbol{X}^{-},\boldsymbol{Y})$ be the
weakly-stationary Gaussian distribution in
Section~\ref{section:Preliminaries}. For fixed model parameters $\theta$,
under the frozen-attention Gaussian forecasting model
$
p_{\theta,\phi}
(\boldsymbol{Y}\mid\boldsymbol{X}^{-},\boldsymbol{Z}=\boldsymbol{z})
=
\mathcal{N}
(\boldsymbol{Y};
\boldsymbol{M}_{\phi,\boldsymbol{z}}\boldsymbol{X}^{-},
\sigma_\epsilon^2\boldsymbol{I}_L),
$
let $p(\boldsymbol{Z}\mid\boldsymbol{X}^{-})$ and
$q_\phi(\boldsymbol{Z}\mid\boldsymbol{X}^{-})$ denote the GP-derived
temporal-evidence prior and PE-induced variational posterior,
respectively. Then
$
\mathbb{E}_{\mathcal D}
[\log p_{\theta,\phi}(\boldsymbol{Y}\mid\boldsymbol{X}^{-})]
\geq
\mathcal{L}_{\mathrm{pop}}(\phi)
:=
\mathbb{E}_{\mathcal D}
\mathbb{E}_{q_\phi(\boldsymbol{Z}\mid\boldsymbol{X}^{-})}
[\log p_{\theta,\phi}
(\boldsymbol{Y}\mid\boldsymbol{X}^{-},\boldsymbol{Z})]
-
\mathbb{E}_{\boldsymbol{X}^{-}}
D_{\mathrm{KL}}
\left(
q_\phi(\boldsymbol{Z}\mid\boldsymbol{X}^{-})
\,\Vert\,
p(\boldsymbol{Z}\mid\boldsymbol{X}^{-})
\right).
$
For a static PE,
$q_\phi(\boldsymbol{Z}\mid\boldsymbol{X}^{-})=q_\phi(\boldsymbol{Z})$,
and the population-level ELBO reduces to
\begin{equation}
\label{equ_theorem_population_elbo}
\begin{aligned}
\mathcal{L}_{\mathrm{pop}}(\phi)
=
C_0
&+
\frac{1}{\sigma_\epsilon^2}
\mathbb{E}_{q_\phi}
\operatorname{tr}
\left(
\boldsymbol{M}_{\phi,\boldsymbol{Z}}
\boldsymbol{\Sigma}_{+-}^{\top}
\right)
-
\frac{1}{2\sigma_\epsilon^2}
\mathbb{E}_{q_\phi}
\operatorname{tr}
\left(
\boldsymbol{M}_{\phi,\boldsymbol{Z}}
\boldsymbol{\Sigma}_{-}
\boldsymbol{M}_{\phi,\boldsymbol{Z}}^\top
\right)
\\
&-
\mathbb{E}_{\boldsymbol{X}^{-}}
D_{\mathrm{KL}}
\left(
q_\phi(\boldsymbol{Z})
\,\Vert\,
p(\boldsymbol{Z}\mid\boldsymbol{X}^{-})
\right),
\end{aligned}
\end{equation}
where
$
C_0=
-\operatorname{tr}(\boldsymbol{\Sigma}_{+})/(2\sigma_\epsilon^2)
-
L\log(2\pi\sigma_\epsilon^2)/2
$ is independent of $\phi$.
The general input-dependent covariance form and detailed proof are provided in
Appendix~\ref{app:theorem_bayesian_transformer_tsf}.
\end{theorem}

General non-stationary time series are difficult to characterize
theoretically because their temporal dependence structures may vary over
time. To enable tractable analysis, we model non-stationarity through a
regime-conditioned weakly-stationary process, where each forecasting
window is weakly stationary conditional on its regime. This allows
Theorem~\ref{theorem_population_elbo} to be applied within each regime,
providing a tractable characterization of the effect of PE under
non-stationarity.

\begin{corollary}[Potential suboptimality of static PE under non-stationarity]
\label{corollary_regime_population_elbo}
Let $R\in\{1,\ldots,J\}$ denote the regime of a forecasting window, with
$\pi_r=p(R=r)>0$ and $\sum_{r=1}^{J}\pi_r=1$. Conditional on $R=r$,
assume
$
\mathcal{D}_r
=
p(\boldsymbol{X}^{-},\boldsymbol{Y}\mid R=r)
$
is a weakly-stationary Gaussian distribution with regime-specific
autocovariance function $\Gamma_r(\tau)$ and corresponding
temporal-evidence prior
$p_r(\boldsymbol{Z}\mid\boldsymbol{X}^{-})$.
By Theorem~\ref{theorem_population_elbo}, the population-level ELBO of PE
method $\phi$ under the regime-conditioned process is
$
\mathcal{L}_{\mathrm{reg}}(\phi)
=
\sum_{r=1}^{J}
\pi_r\mathcal{L}_r(\phi),
$
where $\mathcal{L}_r(\phi)$ denotes the population ELBO under regime $r$.
For a static PE, the PE-induced posterior
$q_\phi(\boldsymbol{Z})$ is shared across regimes, whereas the
regime-specific prior
$p_r(\boldsymbol{Z}\mid\boldsymbol{X}^{-})$
and covariance structure may vary with $\Gamma_r(\tau)$.
Let $\Phi$ denote the admissible PE class and
$\Phi_{\mathrm{static}}\subseteq\Phi$ its static subclass. Therefore,
\begin{equation}
\label{equ_static_pe_upper_bound_main}
\sup_{\phi\in\Phi_{\mathrm{static}}}
\mathcal{L}_{\mathrm{reg}}(\phi)
\leq
\sum_{r=1}^{J}
\pi_r
\sup_{\phi_r\in\Phi}
\mathcal{L}_r(\phi_r).
\end{equation}
The detailed
derivation is provided in
Appendix~\ref{app:corollary_regime_population_elbo}.
\end{corollary}

The right-hand side of Equation~\ref{equ_static_pe_upper_bound_main} represents a
regime-specific oracle that may select a different PE for each regime.
Thus, Corollary~\ref{corollary_regime_population_elbo} identifies a potential limitation
of static PE: when heterogeneous regimes induce incompatible regime-specific
optima, a single shared positional mechanism can be strictly suboptimal
relative to regime-specific solutions. We illustrate Corollary ~\ref{corollary_regime_population_elbo} with a binary-regime forecasting example in Section~\ref{app:theorem_demonstration_synthetic}, where windows alternate between smooth and periodic GP dependence and BTA achieves a higher estimated ELBO than the evaluated PE baselines. It motivates input-adaptive PE mechanisms that can vary their temporal evidence across forecasting windows.

\begin{figure}[!htbp]
  \centering
  \includegraphics[width=0.9\textwidth]{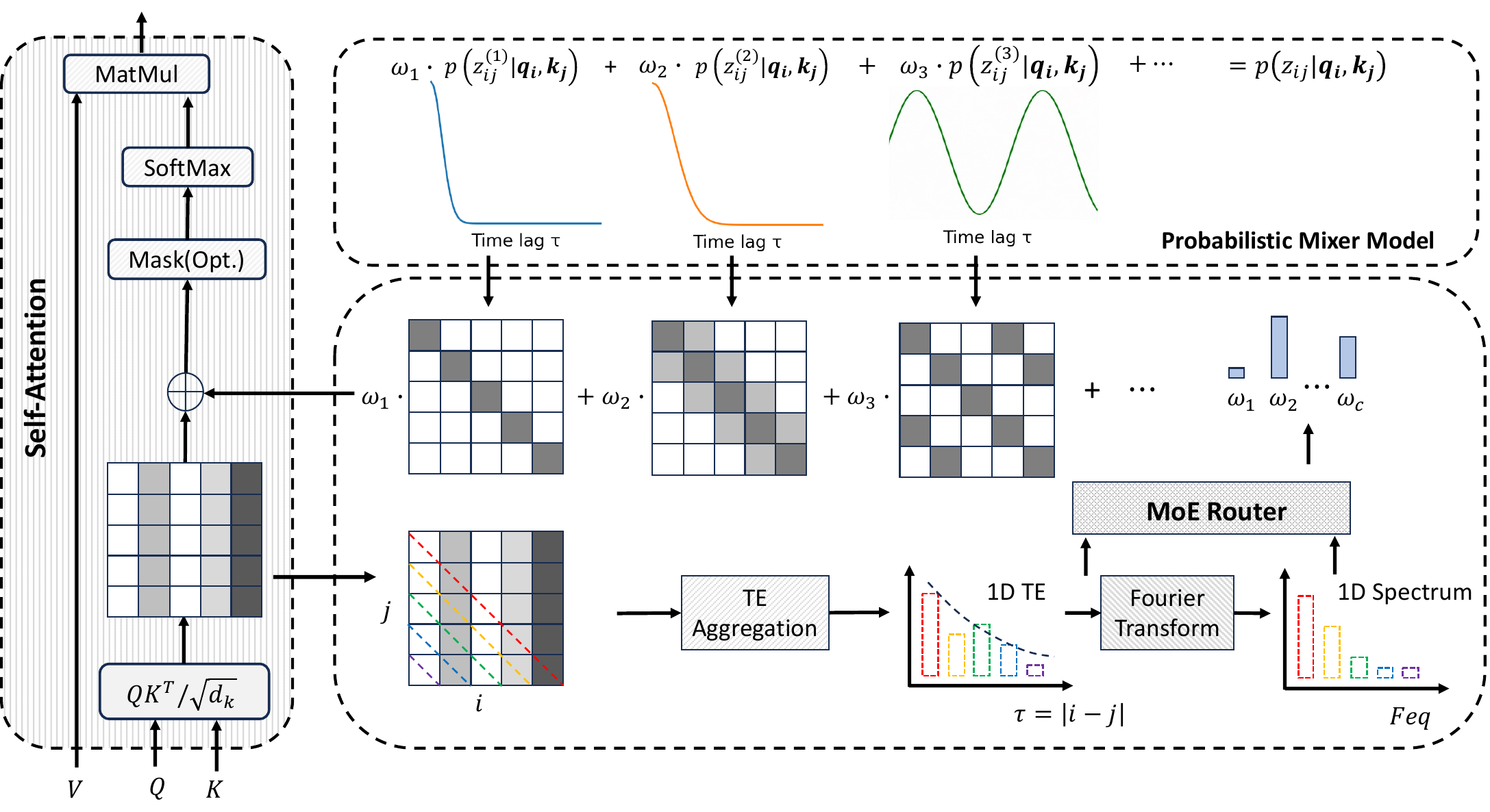}
  \caption{BTA consists of two components: a probabilistic mixer and a MoE router. The probabilistic mixer defines an over-complete set of candidate temporal patterns using time-aware kernels
$\psi^{(c)}\!\left(\tau^{(c)}_{ij};\boldsymbol{\theta}^{(c)}\right)$.
The MoE router learns the kernel weights $\{\omega_c\}_{c=1}^M$ with an MLP that processes the 1D temporal-evidence (TE) profile, obtained by aggregating the original 2D attention scores over relative lags, together with its Fourier spectrum. This plug-and-play design can be seamlessly integrated into standard multi-head self-attention.
}
  
  \label{Img_BTA}
\end{figure}

\section{Methodology}
\label{section_methodology}

Motivated by Corollary~\ref{corollary_regime_population_elbo}, we seek an adaptive PE mechanism that adjusts its temporal evidence to the dependence structure of each input window. Since different windows may exhibit distinct temporal patterns, a mixture-of-experts (MoE) formulation provides a natural mechanism for representing multiple candidate structures and adaptively weighting them without requiring explicit regime labels. Based on this idea, we propose \emph{Bayesian Time-aware Attention} (BTA), consisting of two components: (A) a probabilistic mixture of time-aware kernels that provides an over-complete set of candidate temporal patterns, and (B) an MoE router that produces input- and head-specific weights for combining these kernels. The resulting temporal-evidence likelihood is incorporated into the attention logits as an adaptive positional bias, as illustrated in Fig.~\ref{Img_BTA}.

\subsection{Probabilistic Mixture of Temporal Kernels}

For notational simplicity, we omit the input and attention-head indices $(n,h)$ throughout this subsection. Following Proposition~\ref{proposition_position_encoding}, we model the temporal-evidence likelihood as a mixture of $M$ component likelihoods:
\begin{equation}
\label{equ_temporal_mixture}
p_{\phi}(Z_{ij}=z_{ij}\mid S_i=j)
=
\sum_{c=1}^{M}
\omega_c\,
p_c(Z_{ij}=z_{ij}\mid S_i=j),
\end{equation}
where $\omega_c\geq0$, $\sum_{c=1}^{M}\omega_c=1$, and each component represents a distinct temporal-dependence pattern, such as locality, periodicity, or long-range dependence. The mixture weights are subsequently generated by an MoE router for each input and attention head.

We parameterize the $c$th component through a time-aware log-kernel
$\psi^{(c)}(\tau_{ij};\boldsymbol{\theta}^{(c)})$ such that
$
\log p_c(Z_{ij}=z_{ij}\mid S_i=j)
\overset{c}{=}
\psi^{(c)}
\left(
\tau_{ij};\boldsymbol{\theta}^{(c)}
\right),
$
where $\tau_{ij}=|i-j|$ and $\overset{c}{=}$ denotes equality up to an additive normalization term independent of $j$. Equivalently,
$p_c(Z_{ij}=z_{ij}\mid S_i=j)
\propto
\exp[\psi^{(c)}(\tau_{ij};\boldsymbol{\theta}^{(c)})]$.
Any component-specific normalization constant can be absorbed into the mixture weights without affecting the attention softmax. 
Therefore, the PE-induced temporal bias is
$
b_{ij}
\overset{c}{=}
\log
\left[
\sum_{c=1}^{M}
\omega_c
\exp\!\left(
\psi^{(c)}
\left(
\tau_{ij};\boldsymbol{\theta}^{(c)}
\right)
\right)
\right]
$.
Different choices of $\psi^{(c)}$ encode complementary assumptions about how temporal relevance varies with relative lag. The resulting kernel bank provides an over-complete set of candidate temporal patterns rather than imposing a single positional structure. 
The MoE router then determines their input- and head-specific contributions, allowing BTA to exploit useful temporal priors while adapting the positional evidence across forecasting windows.

\subsection{MoE Router}

To adapt the temporal-kernel mixture to each input, we design an MoE router that extracts lag-dependent structure from the content-based attention scores, as illustrated in Fig.~\ref{Img_BTA}. For the $h$th attention head of the $n$th input, the content-based attention score is
$
s_{ij}^{(n,h)}
=
\left[
(\boldsymbol{q}_i^{(n,h)})^{\top}
\boldsymbol{k}_j^{(n,h)}\right]
/
\sqrt{d_k}.
$
We aggregate scores associated with the same relative lag to construct a one-dimensional lag-wise routing profile:
$
e_{\tau}^{(n,h)}
=
\frac{1}{|\mathcal{I}_{\tau}|}
\sum_{(i,j)\in\mathcal{I}_{\tau}}
s_{ij}^{(n,h)}, 
\mathcal{I}_{\tau}
=
\{(i,j):|i-j|=\tau\}.
$
This aggregation summarizes how strongly the content representations associate query--key pairs at different temporal lags.
To capture both lag-domain structure and periodicity, we concatenate the routing profile with the magnitude of its Fourier spectrum and feed them into an MLP, namely
$
\boldsymbol{\omega}^{(n,h)}
=
\operatorname{Softmax}
\left[
\operatorname{MLP}
\left(
\left[
\boldsymbol{e}^{(n,h)};
\left|
\mathcal{F}
(\boldsymbol{e}^{(n,h)})
\right|
\right]
\right)
\right],
$
Before the softmax, all but the Top-$k$ router logits are masked to $-\infty$, so unselected kernels receive zero weight, as detailed in Appendix~\ref{app:bta_implementation}. Here $\mathcal{F}(\cdot)$ denotes the discrete Fourier transform and
$\boldsymbol{\omega}^{(n,h)}
=
[\omega_1^{(n,h)},\ldots,\omega_M^{(n,h)}]$
denote the input- and head-specific mixture weights.
Finally, BTA constructs the adaptive positional bias from the temporal-kernel mixture in Equation~\ref{equ_temporal_mixture}:
\begin{equation}
\label{equ_bta_bias}
b_{ij}^{(n,h)}
\overset{c}{=}
\log
\left[
\sum_{c=1}^{M}
\omega_c^{(n,h)}
\exp
\left(
\psi^{(c)}
\left(
\tau_{ij};
\boldsymbol{\theta}^{(c)}
\right)
\right)
\right].
\end{equation}
The resulting bias is added to the content-based attention logits before the softmax, allowing different inputs and attention heads to emphasize distinct local, periodic, or long-range temporal dependencies.
Because the candidate kernels and their hyperparameters are fixed, only the router learns how to combine the candidate temporal patterns. This yields a lightweight, plug-and-play mechanism that converts static PE kernels into input- and head-adaptive positional biases and can be incorporated into standard multi-head self-attention without modifying the forecasting backbone.

\section{Experiments}

\subsection{Experimental Setup}

\textbf{Dataset.} We evaluate on eight widely used multivariate time-series forecasting benchmarks: ETTh1, ETTh2, ETTm1, ETTm2, Weather, Electricity, Traffic, and Solar. They span diverse domains and temporal resolutions, with 7--862 variables and Augmented Dickey--Fuller (ADF) statistics ranging from $-37.23$ to $-4.13$ \citep{elliott1992efficient}, providing a broad testbed with varying stationarity and temporal patterns. Detailed descriptions are provided in Appendix~\ref{app:datasets}.

\textbf{Transformer backbones and Non-stationary LTSF references.}
We select four representative Transformer architectures -- Informer~\citep{zhou2021informer}, PatchTST~\citep{nie2023time}, TimeBridge~\citep{liu2025timebridge}, and Timer-XL~\citep{liu2025timer} -- as backbones and integrate BTA into each model. These backbones are chosen to cover complementary stages and design paradigms of Transformer-based LTSF. This diversity allows us to evaluate whether BTA generalizes across substantially different Transformer architectures rather than benefiting a particular backbone. In addition, we compare our method to ten LTSF baselines: PULSE \citep{liupulse}, DTAF \citep{lu2026towards}, DUET \citep{qiu2025duet}, TimeBridge \citep{liu2025timebridge}, iTransformer \citep{liu2024itransformer}, TimeMixer \citep{wang2024timemixer}, PatchTST \citep{nie2023time}, TimesNet \citep{wu2023timesnet}, DLinear \citep{zeng2023transformers}, and Nonstationary Transformer \citep{liu2022non}. Detailed descriptions of backbone and reference methods are provided in Appendix~\ref{app:backbone_architecture} and Appendix~\ref{app:ltsf_reference_methods}, respectively.

\textbf{Implementation Details.}
We use the official implementations of all the aforementioned models.
The implementation details of BTA are provided in Appendix~\ref{app:bta_implementation}.
We evaluate LTSF performance using mean squared error (MSE) and mean absolute error (MAE), and report FLOPs and the number of learnable parameters to quantify model complexity.
All experiments are implemented in PyTorch and executed on an NVIDIA A100 GPU with 80~GB of memory.

\subsection{Main Results}\label{sec:main_results}

We integrate BTA into three backbones. 
Table~\ref{tab:average_forecasting_results_table1} shows that BTA consistently improves all three backbones across the four datasets with different ADFs. Averaged across datasets, BTA reduces MSE by $9.70\%$, $10.35\%$, and $5.70\%$ for Informer, PatchTST, and Timer-XL, respectively, with corresponding MAE reductions of $5.99\%$, $6.84\%$, and $2.68\%$. The improvements across heterogeneous architectures and datasets demonstrate that BTA provides a broadly effective plug-and-play enhancement to Transformer-based models for non-stationary LTSF. Additional three-seed experiments in Table~\ref{tab:backbone_bta_mse_sensitivity} show consistently lower mean MSE with limited variation across random initializations.

\begin{table}[!htbp]
\centering
\caption{\textbf{BTA generalizes across three Transformer backbones}. Input lengths are 96 for Informer and PatchTST and 672 for Timer-XL; backbone-specific settings are in Table~\ref{tab:bta_configuration}. Each dataset entry is averaged over
$H\in\{96,192,336,720\}$. Gain(\%) denotes the MSE and MAE decrease from Vanilla attention to BTA. Full results in Table~\ref{tab:backbone_bta_forecasting_results}.
}
\label{tab:average_forecasting_results_table1}
\fontfamily{ptm}\fontsize{10}{12}\selectfont
\setlength{\tabcolsep}{1.25pt}
\renewcommand{\arraystretch}{1.02}
\resizebox{0.9\linewidth}{!}{%
\begin{tabular}{cc|cc|cc|cc|cc|cc|cc|cc|cc|cc}
\toprule
\multicolumn{2}{c|}{\multirow{1}{*}{\textbf{Model}}}
& \multicolumn{6}{c|}{\textbf{Informer} \citep{zhou2021informer}}
& \multicolumn{6}{c|}{\textbf{PatchTST} \citep{nie2023time}}
& \multicolumn{6}{c}{\textbf{Timer-XL} \citep{liu2025timer}} \\
\midrule
\multicolumn{2}{c|}{\multirow{1}{*}{\textbf{Self-Attention}}}
& \multicolumn{2}{c|}{{Vanilla}}
& \multicolumn{2}{c|}{{BTA}}
& \multicolumn{2}{c|}{{Gain(\%)}}
& \multicolumn{2}{c|}{{Vanilla}}
& \multicolumn{2}{c|}{{BTA}}
& \multicolumn{2}{c|}{Gain(\%)}
& \multicolumn{2}{c|}{{Vanilla}}
& \multicolumn{2}{c|}{{BTA}}
& \multicolumn{2}{c}{Gain(\%)} \\
\midrule
\textbf{Data} & \textbf{ADF}
& \textbf{MSE} & \textbf{MAE}
& \textbf{MSE} & \textbf{MAE}
& \textbf{$\Delta$MSE} & \textbf{$\Delta$MAE}
& \textbf{MSE} & \textbf{MAE}
& \textbf{MSE} & \textbf{MAE}
& \textbf{$\Delta$MSE} & \textbf{$\Delta$MAE}
& \textbf{MSE} & \textbf{MAE}
& \textbf{MSE} & \textbf{MAE}
& \textbf{$\Delta$MSE} & \textbf{$\Delta$MAE} \\
\midrule
ETTh1 & -5.91 & 1.096 & 0.835 & \textbf{0.981} & \textbf{0.769} & 10.50 & 7.88 & 0.516 & 0.477 & \textbf{0.425} & \textbf{0.429} & 17.51 & 10.02 & 0.409 & 0.430 & \textbf{0.379} & \textbf{0.407} & 7.28 & 5.35 \\
\midrule
ETTm1 & -14.98 & 0.826 & 0.690 & \textbf{0.756} & \textbf{0.652} & 8.47 & 5.47 & 0.406 & 0.408 & \textbf{0.371} & \textbf{0.395} & 8.80 & 3.07 & 0.360 & 0.383 & \textbf{0.353} & \textbf{0.380} & 1.74 & 0.65 \\
\midrule
Weather & -26.68 & 0.463 & 0.459 & \textbf{0.398} & \textbf{0.424} & 14.05 & 7.73 & 0.265 & 0.286 & \textbf{0.254} & \textbf{0.278} & 4.34 & 2.80 & 0.240 & 0.273 & \textbf{0.235} & \textbf{0.266} & 1.98 & 2.48 \\
\midrule
Solar & -37.23 & 0.233 & 0.268 & \textbf{0.220} & \textbf{0.261} & 5.79 & 2.89 & 0.270 & 0.307 & \textbf{0.241} & \textbf{0.272} & 10.74 & 11.48 & 0.218 & 0.247 & \textbf{0.193} & \textbf{0.242} & 11.80 & 2.22 \\
\bottomrule
\end{tabular}%
}
\end{table}

We integrate BTA into TimeBridge and compare the resulting model with ten reference algorithms for non-stationary LTSF. As shown in Table~\ref{tab:average_forecasting_results}, BTA consistently improves the TimeBridge backbone across all eight datasets, reducing the dataset-averaged MSE and MAE by $2.39\%$ and $1.89\%$, respectively. These percentages are computed as $100(1-\bar e_{\mathrm{BTA}}/\bar e_{\mathrm{TimeBridge}})$, where $\bar e$ is the mean of the eight dataset-average entries in Table~\ref{tab:average_forecasting_results}. Moreover, TimeBridge+BTA achieves the best MSE on 7 of 8 datasets and the best MAE on 5, outperforming both general LTSF baselines and methods specifically designed for non-stationarity. These results demonstrate that adaptive temporal evidence can further strengthen a competitive non-stationary forecasting backbone and generalize effectively across diverse temporal dynamics.

\begin{table}[!htbp]
\centering
\caption{Comparison with reference methods for non-stationary LTSF. Ours denotes TimeBridge+BTA, and ELC denotes Electricity. Sequence length $T=720$
for all models. Each dataset entry is averaged over
$H\in\{96,192,336,720\}$. Best and second-best
results are shown in bold and underlined, respectively. The 1st Count row counts dataset-level wins, including ties.
Full results are available in Table~\ref{tab:long_term_forecasting}.}
\label{tab:average_forecasting_results}
\fontfamily{ptm}\fontsize{10}{12}\selectfont
\setlength{\tabcolsep}{1.25pt}
\renewcommand{\arraystretch}{1.02}
\resizebox{\linewidth}{!}{%
\begin{tabular}{c|cc|cc|cc|cc|cc|cc|cc|cc|cc|cc|cc}
\toprule
\multicolumn{1}{c}{\multirow{2}{*}{\textbf{Model}}}
& \multicolumn{2}{c}{\multirow{2}{*}{\qquad\textbf{Ours}\qquad}}
& \multicolumn{2}{c}{\textbf{PULSE}}
& \multicolumn{2}{c}{\textbf{DTAF}}
& \multicolumn{2}{c}{\textbf{DUET}}
& \multicolumn{2}{c}{\textbf{TimeBridge}}
& \multicolumn{2}{c}{\textbf{iTransformer}}
& \multicolumn{2}{c}{\textbf{TimeMixer}}
& \multicolumn{2}{c}{\textbf{PatchTST}}
& \multicolumn{2}{c}{\textbf{TimesNet}}
& \multicolumn{2}{c}{\textbf{DLinear}}
& \multicolumn{2}{c}{\textbf{Stationary}} \\

\multicolumn{1}{c}{}
& \multicolumn{2}{c}{}
& \multicolumn{2}{c}{\small\cite{liupulse}}
& \multicolumn{2}{c}{\small\cite{lu2026towards}}
& \multicolumn{2}{c}{\small\cite{qiu2025duet}}
& \multicolumn{2}{c}{\small\cite{liu2025timebridge}}
& \multicolumn{2}{c}{\small\cite{liu2024itransformer}}
& \multicolumn{2}{c}{\small\cite{wang2024timemixer}}
& \multicolumn{2}{c}{\small\cite{nie2023time}}
& \multicolumn{2}{c}{\small\cite{wu2023timesnet}}
& \multicolumn{2}{c}{\small\cite{zeng2023transformers}}
& \multicolumn{2}{c}{\small\cite{liu2022non}} \\

\midrule

\multicolumn{1}{c|}{\textbf{Metric}}
& \textbf{MSE} & \textbf{MAE}
& \textbf{MSE} & \textbf{MAE}
& \textbf{MSE} & \textbf{MAE}
& \textbf{MSE} & \textbf{MAE}
& \textbf{MSE} & \textbf{MAE}
& \textbf{MSE} & \textbf{MAE}
& \textbf{MSE} & \textbf{MAE}
& \textbf{MSE} & \textbf{MAE}
& \textbf{MSE} & \textbf{MAE}
& \textbf{MSE} & \textbf{MAE}
& \textbf{MSE} & \textbf{MAE} \\
\midrule
ETTh1 & \textbf{0.394} & \underline{0.423} & 0.412 & \textbf{0.422} & 0.416 & 0.435 & 0.404 & 0.427 & \underline{0.399} & 0.426 & 0.439 & 0.448 & 0.427 & 0.441 & 0.413 & 0.434 & 0.488 & 0.452 & 0.423 & 0.437 & 0.644 & 0.554 \\
\midrule
ETTh2 & \textbf{0.338} & \textbf{0.382} & 0.361 & 0.391 & 0.363 & 0.396 & \underline{0.344} & 0.390 & \underline{0.344} & \underline{0.384} & 0.374 & 0.406 & 0.349 & 0.397 & 0.351 & 0.395 & 0.414 & 0.427 & 0.431 & 0.447 & 0.521 & 0.487 \\
\midrule
ETTm1 & \underline{0.337} & 0.371 & 0.363 & 0.388 & 0.338 & \textbf{0.368} & \textbf{0.336} & \underline{0.370} & 0.348 & 0.379 & 0.362 & 0.391 & 0.356 & 0.380 & 0.353 & 0.381 & 0.399 & 0.406 & 0.357 & 0.379 & 0.527 & 0.472 \\
\midrule
ETTm2 & \textbf{0.246} & \textbf{0.305} & 0.262 & 0.314 & \underline{0.249} & \underline{0.307} & 0.255 & 0.316 & 0.250 & \underline{0.307} & 0.269 & 0.329 & 0.257 & 0.318 & 0.256 & 0.317 & 0.314 & 0.333 & 0.267 & 0.332 & 0.567 & 0.462 \\
\midrule
Weather & \textbf{0.218} & \textbf{0.248} & 0.239 & 0.271 & 0.222 & 0.251 & 0.230 & 0.264 & \underline{0.220} & \underline{0.250} & 0.233 & 0.271 & 0.226 & 0.264 & 0.226 & 0.264 & 0.259 & 0.287 & 0.240 & 0.300 & 0.295 & 0.315 \\
\midrule
ELC & \textbf{0.152} & \textbf{0.245} & 0.160 & 0.255 & 0.160 & \underline{0.248} & 0.162 & 0.253 & \underline{0.154} & 0.249 & 0.164 & 0.261 & 0.185 & 0.284 & 0.159 & 0.253 & 0.193 & 0.295 & 0.166 & 0.264 & 0.194 & 0.296 \\
\midrule
Traffic & \textbf{0.346} & \textbf{0.242} & 0.460 & 0.286 & 0.403 & \underline{0.249} & 0.386 & 0.254 & \underline{0.360} & 0.253 & 0.397 & 0.282 & 0.409 & 0.279 & 0.391 & 0.264 & 0.620 & 0.336 & 0.434 & 0.295 & 0.624 & 0.340 \\
\midrule
Solar & \textbf{0.176} & \underline{0.218} & 0.194 & 0.243 & 0.195 & 0.235 & 0.194 & \textbf{0.213} & \underline{0.186} & 0.233 & 0.200 & 0.260 & 0.193 & 0.252 & 0.194 & 0.245 & 0.244 & 0.334 & 0.247 & 0.309 & 0.391 & 0.389 \\
\midrule
1st Count & 7 & 5 & 0 & 1 & 0 & 1 & 1 & 1 & 0 & 0 & 0 & 0 & 0 & 0 & 0 & 0 & 0 & 0 & 0 & 0 & 0 & 0 \\
\bottomrule
\end{tabular}%
}
\end{table}

\subsection{Case Study}

To illustrate how BTA adapts positional information to heterogeneous temporal patterns, we select two ETTm1 windows with 720 historical time steps and use TimeBridge to forecast the subsequent 96 steps. Figure~\ref{fig:ettm1_input_conditioned_kernels} visualizes the temporal kernels selected by BTA and their corresponding mixture weights. The left window exhibits a strong weekly periodic pattern, for which the router assigns $0.75$ of the total weight to periodic kernels, including a maximum weight of $0.24$ on a 7-day periodic kernel. In contrast, the right window exhibits stronger local dependence, and BTA assigns $0.66$ of the total weight to local RBF kernels, with the largest weight ($0.41$) placed on the $0.7$-day local kernel. Consequently, BTA induces different positional biases for the two input windows and enables TimeBridge to better capture their distinct temporal dependencies, reducing the forecasting MSE by approximately $37\%$ in both cases. This case study demonstrates the advantage of input-adaptive temporal evidence over a fixed positional structure uniformly across forecasting windows. More case studies in Appendix~\ref{app:solar_qualitative_analysis} further confirm this advantage in different datasets and backbones.

\begin{figure}[!htbp]
  \centering
  \includegraphics[width=0.9\textwidth]{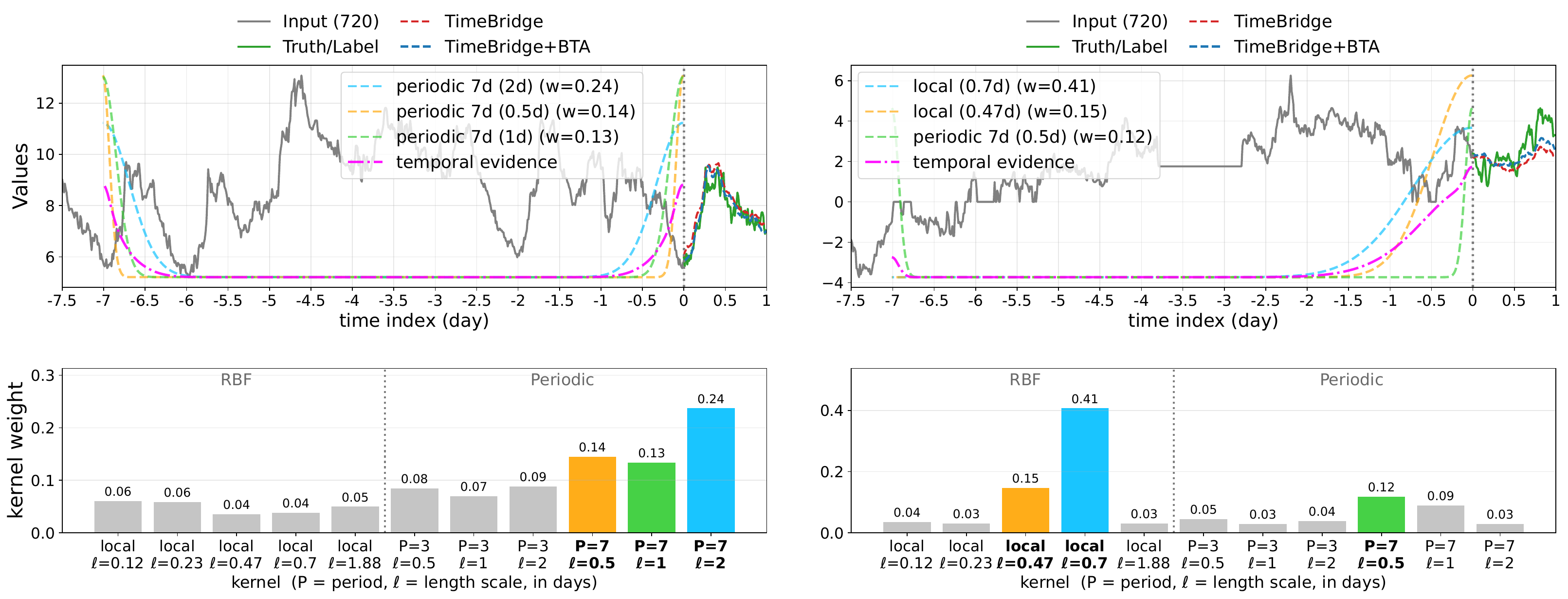}
  \caption{\textbf{Input-adaptive temporal kernel selection in BTA.}
Two ETTm1 windows with a 720-step input and 96-step forecasting horizon.
\emph{Top:} TimeBridge and TimeBridge+BTA forecasts, together with the three highest-weight temporal kernels and their weighted combination.
\emph{Bottom:} MoE router weights over candidate kernels; for periodic kernels, $\ell$ denotes the peak width $\ell'_{r,e}p_e$.
The left window emphasizes weekly periodicity, assigning $0.75$ of the total weight to periodic kernels, whereas the right window emphasizes local dependence, assigning $0.66$ to RBF kernels.
BTA reduces the OT-channel MSE by approximately $37\%$ in both cases ($0.219\rightarrow0.139$ and $0.970\rightarrow0.608$).}
  \label{fig:ettm1_input_conditioned_kernels}
  
\end{figure}

\subsection{Representation Analysis}

Figure~\ref{fig:kernel_weights_by_prediction_length} further shows that BTA
improves not only non-stationary LTSF but also the interpretability of
Transformer temporal behavior. The learned kernel weights vary systematically
with the forecasting horizon: shorter horizons emphasize relatively local
dependencies, whereas longer horizons assign greater importance to broader
local and periodic patterns. Thus, BTA exposes how the temporal inductive bias
of self-attention changes with the prediction horizon, providing an
interpretable characterization of the temporal dependencies used for LTSF.
Extra analysis on Solar in Appendix~\ref{app:solar_kernel_weights}.

\begin{figure}[!htbp]
  \centering
  \includegraphics[width=0.8\textwidth]{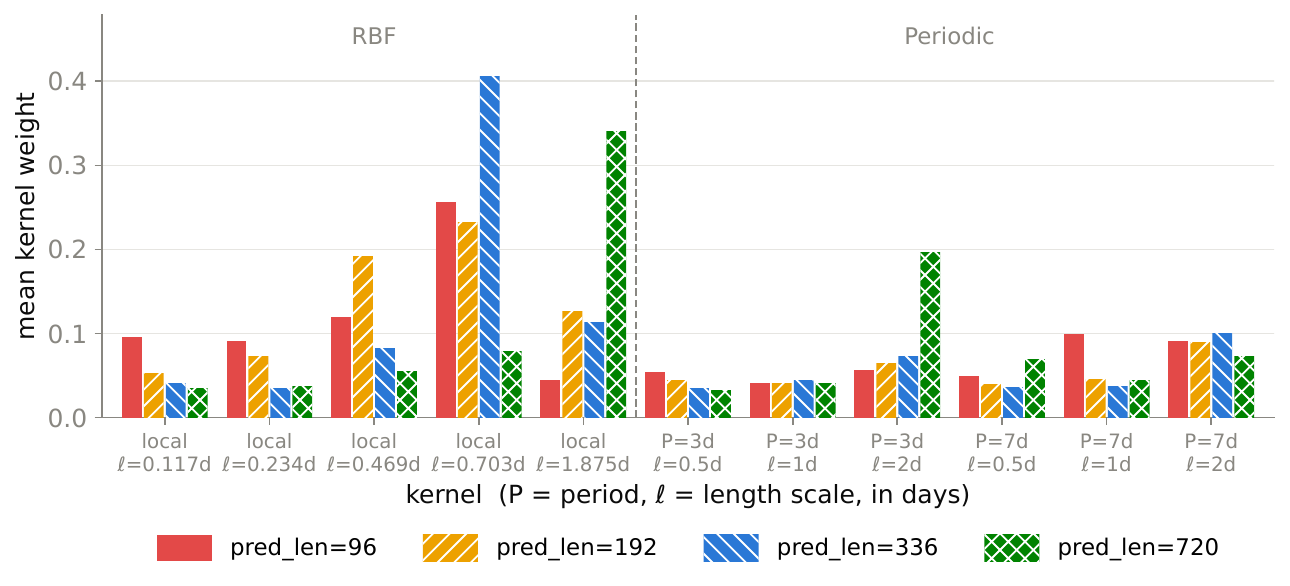}
  \caption{\textbf{BTA reveals horizon-dependent temporal dependencies.}
Mean kernel weights across ETTm1 test samples and attention heads for
$H\in\{96,192,336,720\}$. The dashed line separates local RBF and periodic
kernels. As the forecasting horizon changes, BTA shifts weight across temporal
scales, with longer horizons emphasizing broader local and periodic patterns.}
  \label{fig:kernel_weights_by_prediction_length}
\end{figure}

\subsection{Hyperparameter Analysis}

We analyze two MoE-router hyperparameters: Top-$k$ and hidden dimension $d_r$. Figure~\ref{fig:hyperparameter_analysis} shows the ETTm1 results, with ETTh1 and Solar reported in Appendix~\ref{app:additional_hyperparameter_analysis}. Across datasets, BTA generally favors sparse expert selection: $k=3$ performs best on ETTm1 and Solar, while ETTh1 prefers $k=1$. Router capacity is also dataset-dependent and non-monotonic, with the best $d_r$ equal to $192$, $160$, and $96$ for ETTh1, ETTm1, and Solar, respectively. Overall, a small number of experts and moderate router capacity are sufficient, with $d_r$ selected by validation performance. Additional efficiency analysis is provided in Appendix~\ref{app:additional_efficiency_analysis}.

\begin{figure}[!htbp]
    \centering
    \includegraphics[width=0.7\textwidth,
    trim=0 10 0 0,
    clip]{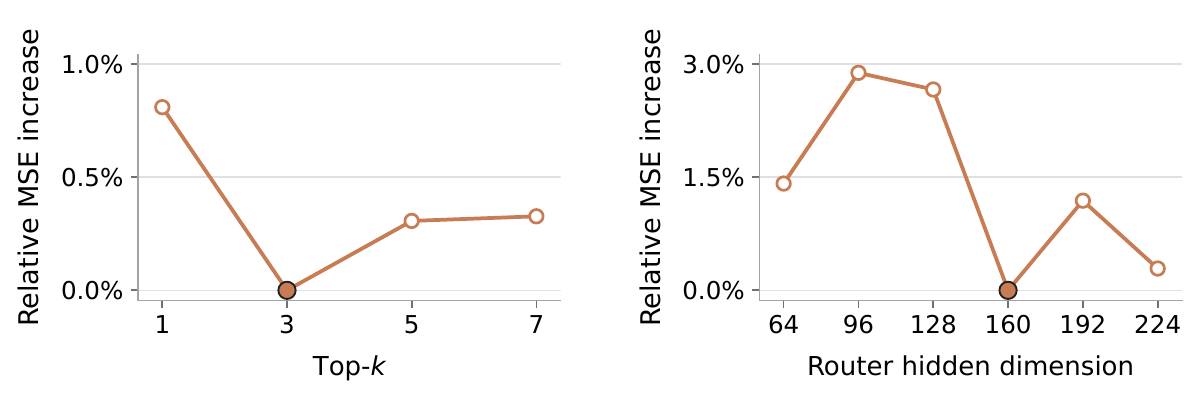}
    \caption{Hyperparameter sensitivity on ETTm1: Top-$k$ (left) and router hidden dimension (right). Each panel shows the percentage MSE increase relative to its lowest MSE.} 
    \label{fig:hyperparameter_analysis}
    
\end{figure}

\subsection{Ablation Study} 

Table~\ref{table_ablation_study} shows the contribution of each component in BTA. Adding local kernels improves vanilla self-attention, while combining local and periodic kernels provides further gains, demonstrating the benefit of a richer temporal kernel bank. Replacing static kernel weights with an MoE router using TE profile further reduces MSE and MAE, highlighting the value of input-adaptive kernel selection. Incorporating the Fourier spectrum of TE achieves the largest gains $5.65\%$ and $6.55\%$, respectively, indicating that frequency-domain information provides complementary guidance for temporal-kernel routing. The ablation study of BTA on Weather in Appendix~\ref{app:additional_ablation_study}.

\begin{table}[!htbp]
\centering
\caption{Ablation study of BTA on Solar (TimeBridge as the backbone). The full BTA variant is also reported as Ours in Table~\ref{tab:long_term_forecasting}.}
\label{table_ablation_study}
\fontfamily{ptm}\fontsize{10}{12}\selectfont
\setlength{\tabcolsep}{3pt}
\renewcommand{\arraystretch}{1.10}
\begin{tabular*}{\linewidth}{@{\extracolsep{\fill}}cc|cc|cc|cc|cc|cc@{}}
\toprule
\multicolumn{2}{c|}{\textbf{Attention}} & \multicolumn{2}{c|}{Self-attention} & \multicolumn{8}{c}{Bayesian Time-aware Attention} \\
\midrule
\multicolumn{2}{c|}{\textbf{Optimization}} & \multicolumn{2}{c|}{Static} & \multicolumn{4}{c|}{Static} & \multicolumn{2}{c|}{\shortstack{MoE\\(TE)}} & \multicolumn{2}{c}{\shortstack{MoE\\(TE + Fourier)}} \\
\midrule
\multicolumn{2}{c|}{\textbf{Kernels}} & \multicolumn{2}{c|}{None} & \multicolumn{2}{c|}{Local} & \multicolumn{2}{c|}{\shortstack{Local +\\periodic}} & \multicolumn{2}{c|}{\shortstack{Local +\\periodic}} & \multicolumn{2}{c}{\shortstack{Local +\\periodic}} \\
\midrule
\textbf{Dataset} & $H$ & MSE & MAE & MSE & MAE & MSE & MAE & MSE & MAE & MSE & MAE \\
\midrule
\multirow{6}{*}{\rotatebox{90}{Solar}} & 96 & 0.161 & 0.224 & 0.158 & 0.219 & 0.157 & 0.213 & 0.152 & 0.201 & \textbf{0.150} & \textbf{0.198} \\
 & 192 & 0.177 & 0.237 & 0.176 & 0.225 & 0.176 & 0.219 & 0.173 & 0.214 & \textbf{0.171} & \textbf{0.211} \\
 & 336 & 0.196 & 0.232 & 0.192 & 0.230 & 0.191 & 0.227 & 0.187 & 0.227 & \textbf{0.184} & \textbf{0.226} \\
 & 720 & 0.210 & 0.239 & 0.207 & 0.238 & 0.205 & 0.236 & 0.201 & \textbf{0.235} & \textbf{0.197} & 0.236 \\
 & \textit{Avg.} & 0.186 & 0.233 & 0.183 & 0.228 & 0.182 & 0.224 & 0.178 & 0.219 & \textbf{0.176} & \textbf{0.218} \\
 & \shortstack{Gain\\(\%)} & -- & -- & 1.48 & 2.15 & 2.02 & 3.97 & 4.17 & 5.90 & $\boldsymbol{5.65}$ & $\boldsymbol{6.55}$ \\
\bottomrule
\end{tabular*}
\par\smallskip\begin{minipage}{\linewidth}
Both MoE variants use local and periodic kernels. TE denotes the temporal-evidence profile; Fourier denotes its Fourier spectrum. Gain is relative to vanilla self-attention.
\end{minipage}
\end{table}

\section{Conclusions}

We studied how positional encoding shapes Transformer forecasting under
non-stationary temporal dependencies. By interpreting position-aware
self-attention as Bayesian inference over attended positions, we derived a
population-level ELBO and showed that static PE can be suboptimal when
temporal dependence varies across input windows. Motivated by this analysis,
we proposed BTA, a plug-and-play mechanism
that adaptively combines complementary temporal kernels through an
input- and head-specific MoE router. Experiments across multiple Transformer
backbones and eight LTSF benchmarks demonstrate that BTA consistently improves
forecasting accuracy and remains competitive with recent methods designed for
non-stationarity. Overall, our results suggest that adaptive temporal
evidence offers a principled, effective, and interpretable alternative to
static positional structures for non-stationary time-series forecasting.

\section*{Disclosure of AI Use}

During this work, we used ChatGPT-5.6 Sol solely to polish the language and improve the readability of the manuscript. We did not use generative AI to develop the theoretical framework, formulate mathematical claims, derive the ELBO or write proofs, design or implement BTA, design experiments, generate or process datasets, interpret results, or translate the manuscript. The authors reviewed all AI-assisted edits to ensure that they preserved the intended technical meaning and introduced no unsupported claims. The authors take full responsibility for the final manuscript, including its mathematical claims, experimental results, code, and conclusions.


\appendix

\section{The Simplified Transformer Architecture}
\label{app:simplifed_transformer}

Here we provide the detailed configuration of the simplified Transformer used in the theoretical analysis. Consider a point-wise Transformer consisting of a single self-attention encoder, average pooling over the $T$ historical positions, and a linear forecasting decoder. 

Given a historical input
$\boldsymbol{x}^{-}\in\mathbb{R}^{T}$, each scalar observation is first mapped to a $d_x$-dimensional representation through a linear embedding,
\begin{equation}
\boldsymbol{H}
=
\boldsymbol{x}^{-}\boldsymbol{w}_e
\in\mathbb{R}^{T\times d_x},
\end{equation}
where $\boldsymbol{w}_e\in\mathbb{R}^{1\times d_x}$ denotes the embedding vector.

The query, key, and value representations are
\begin{equation}
\boldsymbol{Q}
=
\boldsymbol{H}\boldsymbol{W}_q,
\qquad
\boldsymbol{K}
=
\boldsymbol{H}\boldsymbol{W}_k,
\qquad
\boldsymbol{V}
=
\boldsymbol{H}\boldsymbol{W}_v,
\end{equation}
where
$\boldsymbol{W}_q,\boldsymbol{W}_k\in\mathbb{R}^{d_x\times d_k}$
and
$\boldsymbol{W}_v\in\mathbb{R}^{d_x\times d_v}$
are learnable projection matrices. Thus, the $i$th query and $j$th key are
\[
\boldsymbol{q}_i
=
x_i^-\boldsymbol{w}_e\boldsymbol{W}_q
\in\mathbb{R}^{1\times d_k},
\qquad
\boldsymbol{k}_j
=
x_j^-\boldsymbol{w}_e\boldsymbol{W}_k
\in\mathbb{R}^{1\times d_k}.
\]

For a positional encoding method $\phi$, let
$\boldsymbol{\alpha}_{\phi}(\boldsymbol{x}^{-})
\in\mathbb{R}^{T\times T}$
denote the corresponding attention matrix, whose $(i,j)$th entry is
\begin{equation}
\label{equ_appendix_attention}
\alpha_{\phi,ij}(\boldsymbol{x}^{-})
=
\frac{
\exp\left(
\boldsymbol{q}_i\boldsymbol{k}_j^\top/\sqrt{d_k}
+
b_{ij}^{\phi}(\boldsymbol{x}^{-})
\right)
}{
\displaystyle
\sum_{\ell\in\mathcal{C}(i)}
\exp\left(
\boldsymbol{q}_i\boldsymbol{k}_\ell^\top/\sqrt{d_k}
+
b_{i\ell}^{\phi}(\boldsymbol{x}^{-})
\right)
},
\qquad
j\in\mathcal{C}(i),
\end{equation}
and $\alpha_{\phi,ij}(\boldsymbol{x}^{-})=0$ for
$j\notin\mathcal{C}(i)$.
Here, $b_{ij}^{\phi}(\boldsymbol{x}^{-})$ denotes the PE-induced positional score (this dependence on \(\boldsymbol{x}^{-}\) is absent for static PE) and
$\mathcal{C}(i)$ is the set of key positions accessible to query $i$,
as defined in Section~\ref{section:Preliminaries}.

The output of the self-attention encoder is
\begin{equation}
\boldsymbol{H}_{\mathrm{attn}}
=
\boldsymbol{\alpha}_{\phi}(\boldsymbol{x}^{-})
\boldsymbol{V}
\in\mathbb{R}^{T\times d_v}.
\end{equation}
Following the convention \citep{nie2023time}, we aggregate the $T$ output representations using average pooling,
\begin{equation}
\overline{\boldsymbol{h}}_{\phi}^{\top}(\boldsymbol{x}^{-})
=
\frac{1}{T}
\boldsymbol{1}_T^\top
\boldsymbol{\alpha}_{\phi}(\boldsymbol{x}^{-})
\boldsymbol{V}
\in\mathbb{R}^{1\times d_v},
\end{equation}
and apply a linear decoder
$\boldsymbol{W}_d\in\mathbb{R}^{L\times d_v}$ to obtain the
$L$-step forecast,
\begin{equation}
\label{equ_appendix_transformer}
f_{\theta,\phi}(\boldsymbol{x}^{-})
=
\boldsymbol{W}_d
\overline{\boldsymbol{h}}_{\phi}(\boldsymbol{x}^{-})
\in\mathbb{R}^{L},
\end{equation}
where
$\theta=
\{\boldsymbol{w}_e,
\boldsymbol{W}_q,
\boldsymbol{W}_k,
\boldsymbol{W}_v,
\boldsymbol{W}_d\}$
collects the learnable model parameters.

Since
\[
\boldsymbol{V}
=
\boldsymbol{x}^{-}\boldsymbol{w}_e\boldsymbol{W}_v,
\]
Equation~\ref{equ_appendix_transformer} can be rewritten as
\begin{align}
f_{\theta,\phi}(\boldsymbol{x}^{-})
&=
\boldsymbol{W}_d
(\boldsymbol{w}_e\boldsymbol{W}_v)^\top
\left[
\frac{1}{T}
\boldsymbol{1}_T^\top
\boldsymbol{\alpha}_{\phi}(\boldsymbol{x}^{-})
\right]
\boldsymbol{x}^{-}.
\end{align}
Define
\begin{equation}
\boldsymbol{w}
:=
\boldsymbol{W}_d
(\boldsymbol{w}_e\boldsymbol{W}_v)^\top
\in\mathbb{R}^{L\times 1}
\end{equation}
and
\begin{equation}
\boldsymbol{a}_{\phi}^{\top}(\boldsymbol{x}^{-})
:=
\frac{1}{T}
\boldsymbol{1}_T^\top
\boldsymbol{\alpha}_{\phi}(\boldsymbol{x}^{-})
\in\mathbb{R}^{1\times T}.
\end{equation}
The simplified Transformer therefore admits the compact representation
\begin{equation}
\label{equ_appendix_effective_operator}
f_{\theta,\phi}(\boldsymbol{x}^{-})
=
\boldsymbol{w}
\boldsymbol{a}_{\phi}^{\top}(\boldsymbol{x}^{-})
\boldsymbol{x}^{-}
=
\boldsymbol{M}_{\theta,\phi}(\boldsymbol{x}^{-})
\boldsymbol{x}^{-},
\end{equation}
where
\begin{equation}
\boldsymbol{M}_{\theta,\phi}(\boldsymbol{x}^{-})
:=
\boldsymbol{w}
\boldsymbol{a}_{\phi}^{\top}(\boldsymbol{x}^{-})
\in\mathbb{R}^{L\times T}.
\end{equation}
This recovers the effective forecasting operator introduced in
Equation~\ref{equ_simplified_transformer}.

Because each row of
$\boldsymbol{\alpha}_{\phi}(\boldsymbol{x}^{-})$
is an attention distribution over its accessible key positions,
$\boldsymbol{a}_{\phi}^{\top}(\boldsymbol{x}^{-})$
represents the attention coefficients averaged across the $T$ query
positions. Thus, the effect of PE method $\phi$ on the simplified
Transformer enters the forecasting operator through its influence on the
attention distribution.

For analytical clarity, we omit layer normalization, feed-forward
sublayers, and residual connections. Including a residual connection in
the attention block replaces
$\boldsymbol{\alpha}_{\phi}(\boldsymbol{x}^{-})$
by
$\boldsymbol{\alpha}_{\phi}(\boldsymbol{x}^{-})+\boldsymbol{I}_T$
in the derivation above and does not alter the resulting effective
linear-operator representation.

\section{Proof of Proposition~\ref{proposition_position_encoding}}
\label{app:prob_exp_attention_with_PE}

\textit{Proof.}
For a fixed query position $i$, let
$S_i\in\mathcal{C}(i)$ denote the selected key position.
We treat the query $\boldsymbol{q}_i$ and candidate keys
$\{\boldsymbol{k}_j:j\in\mathcal{C}(i)\}$ as fixed representations.
Without positional information, scaled dot-product self-attention induces
the content-based prior
\begin{equation}
\label{equ_content_prior}
p(S_i=j)
=
\frac{
\exp\!\left(
\boldsymbol{q}_i^\top\boldsymbol{k}_j/\sqrt{d_k}
\right)
}{
\sum_{\ell\in\mathcal{C}(i)}
\exp\!\left(
\boldsymbol{q}_i^\top\boldsymbol{k}_{\ell}/\sqrt{d_k}
\right)
}.
\end{equation}
Hence,
\begin{equation}
\label{equ_content_prior_log}
\log p(S_i=j)
\overset{c}{=}
\frac{\boldsymbol{q}_i^\top\boldsymbol{k}_j}{\sqrt{d_k}},
\end{equation}
where $\overset{c}{=}$ denotes equality up to an additive term independent
of $j$.

Let
$\boldsymbol{Z}_i=\{Z_{ij}\}_{j\in\mathcal{C}(i)}$
denote the full temporal evidence associated with query $i$, with observed
realization
$\boldsymbol{z}_i=\{z_{ij}\}_{j\in\mathcal{C}(i)}$.
By the candidate-specific assumption in
Proposition~\ref{proposition_position_encoding}, there exists a factor
$C_i(\boldsymbol{z}_i)>0$, independent of the candidate position $j$,
such that
\begin{equation}
\label{equ_candidate_evidence_factorization}
p(\boldsymbol{Z}_i=\boldsymbol{z}_i\mid S_i=j)
=
C_i(\boldsymbol{z}_i)\,
p(Z_{ij}=z_{ij}\mid S_i=j).
\end{equation}
Here and below, $p$ denotes a probability mass or density as appropriate.

Applying Bayes' rule and using
Equation~\ref{equ_candidate_evidence_factorization},
\begin{equation}
\label{equ_bayes_attention}
p(S_i=j\mid\boldsymbol{Z}_i=\boldsymbol{z}_i)
=
\frac{
p(S_i=j)\,
p(Z_{ij}=z_{ij}\mid S_i=j)
}{
\sum_{\ell\in\mathcal{C}(i)}
p(S_i=\ell)\,
p(Z_{i\ell}=z_{i\ell}\mid S_i=\ell)
},
\end{equation}
because the common factor $C_i(\boldsymbol{z}_i)$ cancels between the
numerator and denominator. Therefore,
\begin{equation}
\label{equ_attn_logposterior}
\log
p(S_i=j\mid\boldsymbol{Z}_i=\boldsymbol{z}_i)
\overset{c}{=}
\frac{\boldsymbol{q}_i^\top\boldsymbol{k}_j}{\sqrt{d_k}}
+
\log p(Z_{ij}=z_{ij}\mid S_i=j).
\end{equation}

On the other hand, self-attention with PE assigns
\begin{equation}
\label{equ_pe_attention_posterior}
\alpha_{ij}
=
\frac{
\exp\!\left(
\boldsymbol{q}_i^\top\boldsymbol{k}_j/\sqrt{d_k}
+b_{ij}
\right)
}{
\sum_{\ell\in\mathcal{C}(i)}
\exp\!\left(
\boldsymbol{q}_i^\top\boldsymbol{k}_{\ell}/\sqrt{d_k}
+b_{i\ell}
\right)
}.
\end{equation}
Thus, if PE specifies the temporal-evidence likelihood such that
\begin{equation}
\label{equ_pe_loglikelihood}
b_{ij}
\overset{c}{=}
\log p(Z_{ij}=z_{ij}\mid S_i=j),
\end{equation}
then Equations~\ref{equ_attn_logposterior} and
\ref{equ_pe_attention_posterior} have identical normalized logits.
Consequently,
\begin{equation}
\alpha_{ij}
=
p(S_i=j\mid\boldsymbol{Z}_i=\boldsymbol{z}_i).
\end{equation}
Therefore, position-aware self-attention can be interpreted as Bayesian
reweighting over candidate key positions: scaled dot-product similarity
defines the content-based prior, while PE contributes the
candidate-specific temporal-evidence likelihood.

\section{Probabilistic Interpretations of Positional Encoding}
\label{app:prob_exp_PE}

Proposition~\ref{proposition_position_encoding} provides a unified
probabilistic interpretation of representative positional encoding
(PE) mechanisms, including Gaussian/RBF positional bias
\citep{guo2019gaussian}, ALiBi \citep{press2021train}, KERPLE
\citep{chi2022kerple}, the power-law decay bias of Powerformer
\citep{hegazy2025powerformer}, and RoPE \citep{su2024roformer}.
Specifically, an additive positional term in the attention logits can be
interpreted as the log-likelihood contribution of temporal evidence, up to an
additive normalization term that is invariant to the key position $j$ and
therefore cancels under the softmax. Table~\ref{tab:temporal_positional_bias}
summarizes the resulting probabilistic interpretations and illustrates how
different PE mechanisms encode distinct assumptions about the dependence of
temporal evidence on the relative lag $\Delta_{ij}=i-j$.

\begin{table}[!htbp]
\centering
\caption{
Probabilistic interpretations of representative PE methods, where
$\Delta_{ij}=i-j$.
Here, $\sigma^2>0$ denotes the variance of the Gaussian/RBF kernel;
$m_h>0$ is the head-specific ALiBi slope;
$r_1,r_2>0$ are KERPLE parameters controlling the magnitude and scale of
its logarithmic decay, respectively; and
$\alpha>0$ is the power-law decay exponent in Powerformer.
Log-likelihood terms are shown up to additive normalization constants that
do not affect the softmax over $j$.
}
\label{tab:temporal_positional_bias}
{\fontfamily{ptm}\fontsize{9}{10}\selectfont
\setlength{\tabcolsep}{6pt}
\renewcommand{\arraystretch}{1}
\begin{tabular}{c | c c c}
\hline
PE method
& Domain
& $p(Z_{ij}=z_{ij}\mid S_i=j)$
& $\log p(Z_{ij}=z_{ij}\mid S_i=j)$ \\
\hline \\[-1.5ex]

Gaussian/RBF
& NLP
& $\displaystyle
   \propto
   \exp\!\left(
   -\frac{\Delta_{ij}^{2}}{2\sigma^{2}}
   \right)$
& $\displaystyle
   -\frac{\Delta_{ij}^{2}}{2\sigma^{2}}$ \\[2ex]

ALiBi
& NLP
& $\displaystyle
   \propto
   \exp\!\left(-m_h|\Delta_{ij}|\right)$
& $\displaystyle
   -m_h|\Delta_{ij}|$ \\[2ex]

KERPLE
& NLP
& $\displaystyle
   \propto
   \left(1+r_2|\Delta_{ij}|\right)^{-r_1}$
& $\displaystyle
   -r_1\log\!\left(1+r_2|\Delta_{ij}|\right)$ \\[2ex]

Powerformer
& TSF
& $\displaystyle
   \propto
   \left(1+|\Delta_{ij}|\right)^{-\alpha}$
& $\displaystyle
   -\alpha\log\!\left(1+|\Delta_{ij}|\right)$ \\[2ex]

RoPE
& NLP \& TSF
& ---
& $\displaystyle
   \{\boldsymbol{q}_i^{\top}
   \left(\boldsymbol{R}_{\Delta_{ij}}-\boldsymbol{I}\right)
   \boldsymbol{k}_j\}/\sqrt{d_k}$ \\[1ex]

\hline
\end{tabular}
}
\end{table}

The first four methods can therefore be understood as imposing different
assumptions on how temporal evidence decays with relative distance.
Gaussian/RBF bias yields a quadratic log-likelihood,
$
-\frac{\Delta_{ij}^2}{2\sigma^2},
$
which induces rapid decay and favors temporally local interactions. The
variance $\sigma^2$ controls the effective temporal range: a larger
$\sigma^2$ produces a broader likelihood and assigns relatively more weight
to distant positions. ALiBi instead yields the linear decay
\[
-m_h|\Delta_{ij}|,
\]
which corresponds to a Laplace-type temporal likelihood. Its head-specific
slope $m_h$ determines how strongly each attention head penalizes distant
positions.

KERPLE introduces the heavier-tailed logarithmic decay
\[
-r_1\log\!\left(1+r_2|\Delta_{ij}|\right),
\]
where $r_1$ controls the overall strength of the decay and $r_2$ determines
its distance scale. Similarly, the Powerformer bias
\[
-\alpha\log\!\left(1+|\Delta_{ij}|\right)
\]
corresponds to a power-law temporal likelihood, with $\alpha$ controlling
the decay rate. Compared with Gaussian or exponential decay, these
logarithmic forms assign relatively greater likelihood to distant positions
and therefore permit heavier-tailed temporal dependencies.

More generally, $\sigma^2$, $m_h$, $r_1$, $r_2$, and $\alpha$ determine the
temporal scale and tail behavior of their corresponding evidence
likelihoods. Thus, different PE mechanisms can be viewed not merely as
alternative deterministic position transformations, but as encoding
different probabilistic assumptions about the range and strength of
temporal dependencies.

RoPE admits a related but qualitatively different interpretation. Let
$\boldsymbol{R}_{\Delta_{ij}}$ denote the relative rotation matrix associated
with lag $\Delta_{ij}$. The RoPE attention score can be decomposed as
\[
(\boldsymbol{q}_i^{\top}
\boldsymbol{R}_{\Delta_{ij}}
\boldsymbol{k}_j)/\sqrt{d_k}
=
\{\boldsymbol{q}_i^{\top}\boldsymbol{k}_j
+
\boldsymbol{q}_i^{\top}
\left(
\boldsymbol{R}_{\Delta_{ij}}-\boldsymbol{I}
\right)
\boldsymbol{k}_j\}/\sqrt{d_k},
\]
where $\boldsymbol{I}$ denotes the identity matrix. Under Proposition~\ref{proposition_position_encoding}, with
$\boldsymbol{q}_i$ and $\boldsymbol{k}_j$ treated as fixed
representations, the second term can therefore be interpreted as the
candidate-specific log temporal-evidence contribution,
\[
\log p(Z_{ij}=z_{ij}\mid S_i=j)
\overset{c}{=}
\frac{
\boldsymbol{q}_i^{\top}
\left(
\boldsymbol{R}_{\Delta_{ij}}-\boldsymbol{I}
\right)
\boldsymbol{k}_j
}{\sqrt{d_k}},
\]
where $\overset{c}{=}$ denotes equality up to an additive normalization
term independent of $j$ and therefore irrelevant to the attention
softmax.
Unlike Gaussian/RBF, ALiBi, KERPLE, and Powerformer, whose positional terms
are determined primarily by relative distance, the temporal-evidence
contribution induced by RoPE depends jointly on the relative lag
$\Delta_{ij}$ and the content representations $\boldsymbol{q}_i$ and
$\boldsymbol{k}_j$. Consequently, the same temporal lag can contribute
different amounts of evidence for different query--key pairs. This
distinction highlights the generality of
Proposition~\ref{proposition_position_encoding}: both distance-dependent
positional biases and content-dependent relative rotations can be
interpreted within a common probabilistic framework as mechanisms that
modify the temporal evidence entering self-attention.

\section{Proof of Theorem~\ref{theorem_population_elbo}}
\label{app:theorem_bayesian_transformer_tsf}

We prove Theorem~\ref{theorem_population_elbo} in five steps:

\begin{enumerate}
    \item \textbf{Sample-level ELBO.}
    We derive the standard variational lower bound for a forecasting
    sample $(\boldsymbol{x}^{-},\boldsymbol{y})$ with latent temporal
    evidence $\boldsymbol{Z}$.

    \item \textbf{PE-induced future-free posterior.}
    We specialize the variational family to
    $q_{\phi}(\boldsymbol{Z}\mid\boldsymbol{x}^{-})$, reflecting that PE
    depends only on the historical input.

    \item \textbf{Population-level ELBO.}
    We average the sample-level bound over
    $\mathcal{D}=p(\boldsymbol{X}^{-},\boldsymbol{Y})$.

    \item \textbf{Covariance-based population form.}
    Under the Gaussian forecasting likelihood, we express the forecasting
    term through $q_\phi$-weighted second-order moments. For static PE,
    these reduce to the population covariance matrices
    $\boldsymbol{\Sigma}_{-}$ and $\boldsymbol{\Sigma}_{+-}$.

    \item \textbf{Prior and posterior of temporal evidence.}
    We derive the GP-based prior
    $p(\boldsymbol{Z}\mid\boldsymbol{X}^{-})$ and parameterize the
    PE-induced variational posterior
    $q_{\phi}(\boldsymbol{Z}\mid\boldsymbol{X}^{-})$ using the first- and
    second-order statistics of lag-wise PE profiles.
\end{enumerate}

\subsection{Step 1: Sample-level ELBO}
\label{app:sample_elbo}

Consider a forecasting sample
$(\boldsymbol{x}^{-},\boldsymbol{y})$ and latent temporal evidence
$\boldsymbol{Z}$. The conditional marginal likelihood is
\begin{equation}
\label{equ_sample_marginal_likelihood}
p_{\theta,\phi}
(\boldsymbol{y}\mid\boldsymbol{x}^{-})
=
\int
p_{\theta,\phi}
(\boldsymbol{y}\mid\boldsymbol{x}^{-},\boldsymbol{z})
p(\boldsymbol{z}\mid\boldsymbol{x}^{-})
\,d\boldsymbol{z}.
\end{equation}
For any variational distribution
$q(\boldsymbol{Z}\mid\boldsymbol{x}^{-},\boldsymbol{y})$
whose support is contained in that of
$p(\boldsymbol{Z}\mid\boldsymbol{x}^{-})$, Jensen's inequality gives
\begin{equation}
\label{equ_sample_elbo_general}
\begin{aligned}
\log p_{\theta,\phi}
(\boldsymbol{y}\mid\boldsymbol{x}^{-})
\geq{}&
\mathbb{E}_{q(\boldsymbol{Z}\mid\boldsymbol{x}^{-},\boldsymbol{y})}
\left[
\log
p_{\theta,\phi}
(\boldsymbol{y}\mid\boldsymbol{x}^{-},\boldsymbol{Z})
\right]
\\
&-
D_{\mathrm{KL}}
\left(
q(\boldsymbol{Z}\mid\boldsymbol{x}^{-},\boldsymbol{y})
\,\Vert\,
p(\boldsymbol{Z}\mid\boldsymbol{x}^{-})
\right).
\end{aligned}
\end{equation}
Thus, Equation~\ref{equ_sample_elbo_general} is the standard sample-level
variational lower bound and requires neither the Gaussian forecasting
assumption nor weak stationarity.

\subsection{Step 2: PE-induced Future-free Variational Posterior}
\label{app:future_free_posterior}

PE is determined from the historical input and cannot access the future
target at inference time. We therefore choose the future-free
PE-induced variational family
$q_{\phi}(\boldsymbol{Z}\mid\boldsymbol{x}^{-})$.
Substituting this family into
Equation~\ref{equ_sample_elbo_general} gives
\begin{equation}
\label{equ_sample_elbo_future_free}
\begin{aligned}
\log p_{\theta,\phi}
(\boldsymbol{y}\mid\boldsymbol{x}^{-})
\geq{}&
\mathbb{E}_{q_{\phi}(\boldsymbol{Z}\mid\boldsymbol{x}^{-})}
\left[
\log
p_{\theta,\phi}
(\boldsymbol{y}\mid\boldsymbol{x}^{-},\boldsymbol{Z})
\right]
\\
&-
D_{\mathrm{KL}}
\left(
q_{\phi}(\boldsymbol{Z}\mid\boldsymbol{x}^{-})
\,\Vert\,
p(\boldsymbol{Z}\mid\boldsymbol{x}^{-})
\right).
\end{aligned}
\end{equation}
Accordingly, define the future-free sample-level ELBO as
\begin{equation}
\label{equ_future_free_sample_elbo}
\begin{aligned}
\widetilde{\mathcal{L}}_{\mathrm{ELBO}}
(\theta,\phi;
\boldsymbol{x}^{-},\boldsymbol{y})
:={}&
\mathbb{E}_{q_{\phi}(\boldsymbol{Z}\mid\boldsymbol{x}^{-})}
\left[
\log
p_{\theta,\phi}
(\boldsymbol{y}\mid\boldsymbol{x}^{-},\boldsymbol{Z})
\right]
\\
&-
D_{\mathrm{KL}}
\left(
q_{\phi}(\boldsymbol{Z}\mid\boldsymbol{x}^{-})
\,\Vert\,
p(\boldsymbol{Z}\mid\boldsymbol{x}^{-})
\right).
\end{aligned}
\end{equation}
Although $\phi$ may be learned from forecasting errors involving future
targets during training, evaluating
$q_{\phi}(\boldsymbol{Z}\mid\boldsymbol{x}^{-})$
depends only on the historical input.

\subsection{Step 3: Population-level ELBO}
\label{app:population_elbo}

Let
$\mathcal{D}=p(\boldsymbol{X}^{-},\boldsymbol{Y})$
denote the population distribution of historical and future observations.
Taking expectation of
Equation~\ref{equ_future_free_sample_elbo}
with respect to $\mathcal{D}$ defines
\begin{equation}
\label{equ_population_elbo_definition}
\mathcal{L}_{\mathrm{pop}}(\theta,\phi)
:=
\mathbb{E}_{\mathcal{D}}
\left[
\widetilde{\mathcal{L}}_{\mathrm{ELBO}}
(\theta,\phi;
\boldsymbol{X}^{-},\boldsymbol{Y})
\right].
\end{equation}
Substituting Equation~\ref{equ_future_free_sample_elbo} gives
\begin{equation}
\label{equ_population_elbo_general}
\begin{aligned}
\mathcal{L}_{\mathrm{pop}}(\theta,\phi)
={}&
\mathbb{E}_{\mathcal{D}}
\mathbb{E}_{q_{\phi}(\boldsymbol{Z}\mid\boldsymbol{X}^{-})}
\left[
\log
p_{\theta,\phi}
(\boldsymbol{Y}\mid\boldsymbol{X}^{-},\boldsymbol{Z})
\right]
\\
&-
\mathbb{E}_{\boldsymbol{X}^{-}}
\left[
D_{\mathrm{KL}}
\left(
q_{\phi}(\boldsymbol{Z}\mid\boldsymbol{X}^{-})
\,\Vert\,
p(\boldsymbol{Z}\mid\boldsymbol{X}^{-})
\right)
\right].
\end{aligned}
\end{equation}
Since Equation~\ref{equ_sample_elbo_future_free} holds for every
forecasting sample,
\begin{equation}
\label{equ_population_elbo_bound}
\mathbb{E}_{\mathcal{D}}
\left[
\log
p_{\theta,\phi}
(\boldsymbol{Y}\mid\boldsymbol{X}^{-})
\right]
\geq
\mathcal{L}_{\mathrm{pop}}(\theta,\phi).
\end{equation}
Therefore,
$\mathcal{L}_{\mathrm{pop}}(\theta,\phi)$
is a population-level evidence lower bound.

\subsection{Step 4: Covariance-based Population Form}
\label{app:covariance_population_elbo}

We next characterize the forecasting term in
Equation~\ref{equ_population_elbo_general}.
Throughout this step, the model parameters $\theta$ are fixed, and their
dependence is suppressed for notational simplicity.
Conditioned on
$\boldsymbol{Z}=\boldsymbol{z}$, assume the Gaussian forecasting
likelihood
\begin{equation}
\label{equ_gaussian_forecasting_likelihood}
p_{\theta,\phi}
\left(
\boldsymbol{Y}
\mid
\boldsymbol{X}^{-},
\boldsymbol{Z}=\boldsymbol{z}
\right)
=
\mathcal{N}
\left(
\boldsymbol{Y};
\boldsymbol{M}_{\phi,\boldsymbol{z}}
\boldsymbol{X}^{-},
\sigma_\epsilon^2\boldsymbol{I}_L
\right),
\end{equation}
where
$\boldsymbol{M}_{\phi,\boldsymbol{z}}
\in\mathbb{R}^{L\times T}$
denotes the effective forecasting operator associated with temporal
evidence $\boldsymbol{z}$.
Relative to the original input-dependent operator
$\boldsymbol{M}_{\theta,\phi}(\boldsymbol{x}^{-})$,
Equation~\ref{equ_gaussian_forecasting_likelihood}
uses a frozen-attention representation in which the forecasting operator
is fixed conditional on $\boldsymbol{Z}$.
Importantly,
$q_{\phi}(\boldsymbol{Z}\mid\boldsymbol{X}^{-})$
remains input-dependent and is not replaced by a population marginal.

From Equation~\ref{equ_gaussian_forecasting_likelihood},
\begin{equation}
\label{equ_gaussian_log_likelihood}
\log
p_{\theta,\phi}
\left(
\boldsymbol{Y}
\mid
\boldsymbol{X}^{-},
\boldsymbol{Z}
\right)
=
-\frac{L}{2}\log(2\pi\sigma_\epsilon^2)
-
\frac{1}{2\sigma_\epsilon^2}
\left\|
\boldsymbol{Y}
-
\boldsymbol{M}_{\phi,\boldsymbol{Z}}
\boldsymbol{X}^{-}
\right\|_2^2.
\end{equation}
Expanding the squared error yields
\begin{equation}
\label{equ_forecasting_error_expansion}
\begin{aligned}
\left\|
\boldsymbol{Y}
-
\boldsymbol{M}_{\phi,\boldsymbol{Z}}
\boldsymbol{X}^{-}
\right\|_2^2
={}&
\boldsymbol{Y}^{\top}\boldsymbol{Y}
-
2\boldsymbol{Y}^{\top}
\boldsymbol{M}_{\phi,\boldsymbol{Z}}
\boldsymbol{X}^{-}
\\
&+
(\boldsymbol{X}^{-})^{\top}
\boldsymbol{M}_{\phi,\boldsymbol{Z}}^{\top}
\boldsymbol{M}_{\phi,\boldsymbol{Z}}
\boldsymbol{X}^{-}.
\end{aligned}
\end{equation}

To preserve the dependence between
$\boldsymbol{Z}$ and $\boldsymbol{X}^{-}$ induced by
$q_{\phi}(\boldsymbol{Z}\mid\boldsymbol{X}^{-})$, define the
$q_\phi$-weighted second-order moments
\begin{equation}
\label{equ_weighted_history_covariance}
\boldsymbol{\Omega}_{-,\phi}(\boldsymbol{z})
:=
\mathbb{E}_{\boldsymbol{X}^{-}}
\left[
q_{\phi}(\boldsymbol{z}\mid\boldsymbol{X}^{-})
\boldsymbol{X}^{-}(\boldsymbol{X}^{-})^\top
\right]
\end{equation}
and
\begin{equation}
\label{equ_weighted_cross_covariance}
\boldsymbol{\Omega}_{+-,\phi}(\boldsymbol{z})
:=
\mathbb{E}_{\mathcal{D}}
\left[
q_{\phi}(\boldsymbol{z}\mid\boldsymbol{X}^{-})
\boldsymbol{Y}(\boldsymbol{X}^{-})^\top
\right].
\end{equation}
These quantities retain the dependence between the PE-induced temporal
evidence and the historical input.

Because
$\int q_{\phi}(\boldsymbol{z}\mid\boldsymbol{X}^{-})
\,d\boldsymbol{z}=1$,
the first term in
Equation~\ref{equ_forecasting_error_expansion} satisfies
\begin{equation}
\label{equ_y_covariance_term}
\mathbb{E}_{\mathcal D}
\mathbb{E}_{q_\phi(\boldsymbol{Z}\mid\boldsymbol{X}^{-})}
\left[
\boldsymbol{Y}^{\top}\boldsymbol{Y}
\right]
=
\operatorname{tr}(\boldsymbol{\Sigma}_{+}).
\end{equation}
The cross term satisfies
\begin{equation}
\label{equ_cross_covariance_term}
\mathbb{E}_{\mathcal D}
\mathbb{E}_{q_\phi(\boldsymbol{Z}\mid\boldsymbol{X}^{-})}
\left[
\boldsymbol{Y}^{\top}
\boldsymbol{M}_{\phi,\boldsymbol{Z}}
\boldsymbol{X}^{-}
\right]
=
\int
\operatorname{tr}
\left(
\boldsymbol{M}_{\phi,\boldsymbol{z}}
\boldsymbol{\Omega}_{+-,\phi}^{\top}(\boldsymbol{z})
\right)
d\boldsymbol{z},
\end{equation}
and the quadratic term satisfies
\begin{equation}
\label{equ_history_covariance_term}
\mathbb{E}_{\mathcal D}
\mathbb{E}_{q_\phi(\boldsymbol{Z}\mid\boldsymbol{X}^{-})}
\left[
(\boldsymbol{X}^{-})^{\top}
\boldsymbol{M}_{\phi,\boldsymbol{Z}}^{\top}
\boldsymbol{M}_{\phi,\boldsymbol{Z}}
\boldsymbol{X}^{-}
\right]
=
\int
\operatorname{tr}
\left(
\boldsymbol{M}_{\phi,\boldsymbol{z}}
\boldsymbol{\Omega}_{-,\phi}(\boldsymbol{z})
\boldsymbol{M}_{\phi,\boldsymbol{z}}^\top
\right)
d\boldsymbol{z}.
\end{equation}

Substituting
Equations~\ref{equ_y_covariance_term}--\ref{equ_history_covariance_term}
into Equation~\ref{equ_gaussian_log_likelihood} gives
\begin{equation}
\label{equ_expected_gaussian_likelihood}
\begin{aligned}
&
\mathbb{E}_{\mathcal D}
\mathbb{E}_{q_\phi(\boldsymbol{Z}\mid\boldsymbol{X}^{-})}
\left[
\log
p_{\theta,\phi}
(
\boldsymbol{Y}
\mid
\boldsymbol{X}^{-},
\boldsymbol{Z}
)
\right]
\\
={}&
C_0
+
\frac{1}{\sigma_\epsilon^2}
\int
\operatorname{tr}
\left(
\boldsymbol{M}_{\phi,\boldsymbol{z}}
\boldsymbol{\Omega}_{+-,\phi}^{\top}(\boldsymbol{z})
\right)
d\boldsymbol{z}
\\
&-
\frac{1}{2\sigma_\epsilon^2}
\int
\operatorname{tr}
\left(
\boldsymbol{M}_{\phi,\boldsymbol{z}}
\boldsymbol{\Omega}_{-,\phi}(\boldsymbol{z})
\boldsymbol{M}_{\phi,\boldsymbol{z}}^\top
\right)
d\boldsymbol{z},
\end{aligned}
\end{equation}
where
\begin{equation}
\label{equ_population_constant}
C_0
=
-\frac{1}{2\sigma_\epsilon^2}
\operatorname{tr}(\boldsymbol{\Sigma}_{+})
-
\frac{L}{2}\log(2\pi\sigma_\epsilon^2).
\end{equation}

Combining Equation~\ref{equ_expected_gaussian_likelihood} with
Equation~\ref{equ_population_elbo_general} yields
\begin{equation}
\label{equ_population_elbo_weighted_covariance}
\begin{aligned}
\mathcal{L}_{\mathrm{pop}}(\phi)
={}&
C_0
+
\frac{1}{\sigma_\epsilon^2}
\int
\operatorname{tr}
\left(
\boldsymbol{M}_{\phi,\boldsymbol{z}}
\boldsymbol{\Omega}_{+-,\phi}^{\top}(\boldsymbol{z})
\right)
d\boldsymbol{z}
\\
&-
\frac{1}{2\sigma_\epsilon^2}
\int
\operatorname{tr}
\left(
\boldsymbol{M}_{\phi,\boldsymbol{z}}
\boldsymbol{\Omega}_{-,\phi}(\boldsymbol{z})
\boldsymbol{M}_{\phi,\boldsymbol{z}}^\top
\right)
d\boldsymbol{z}
\\
&-
\mathbb{E}_{\boldsymbol{X}^{-}}
\left[
D_{\mathrm{KL}}
\left(
q_{\phi}(\boldsymbol{Z}\mid\boldsymbol{X}^{-})
\,\Vert\,
p(\boldsymbol{Z}\mid\boldsymbol{X}^{-})
\right)
\right].
\end{aligned}
\end{equation}
Equation~\ref{equ_population_elbo_weighted_covariance} preserves the full
dependence of
$q_{\phi}(\boldsymbol{Z}\mid\boldsymbol{X}^{-})$
on the historical input and therefore requires no population-marginal
approximation.

For a static PE,
\begin{equation}
\label{equ_static_pe_posterior}
q_{\phi}(\boldsymbol{z}\mid\boldsymbol{X}^{-})
=
q_{\phi}(\boldsymbol{z}),
\end{equation}
so the weighted moments factor as
\begin{equation}
\label{equ_static_weighted_covariances}
\boldsymbol{\Omega}_{-,\phi}(\boldsymbol{z})
=
q_\phi(\boldsymbol{z})
\boldsymbol{\Sigma}_{-},
\qquad
\boldsymbol{\Omega}_{+-,\phi}(\boldsymbol{z})
=
q_\phi(\boldsymbol{z})
\boldsymbol{\Sigma}_{+-}.
\end{equation}
Equation~\ref{equ_population_elbo_weighted_covariance} therefore reduces
to
\begin{equation}
\label{equ_population_elbo_static_covariance}
\begin{aligned}
\mathcal{L}_{\mathrm{pop}}(\phi)
={}&
C_0
+
\frac{1}{\sigma_\epsilon^2}
\mathbb{E}_{q_\phi}
\left[
\operatorname{tr}
\left(
\boldsymbol{M}_{\phi,\boldsymbol{Z}}
\boldsymbol{\Sigma}_{+-}^{\top}
\right)
\right]
\\
&-
\frac{1}{2\sigma_\epsilon^2}
\mathbb{E}_{q_\phi}
\left[
\operatorname{tr}
\left(
\boldsymbol{M}_{\phi,\boldsymbol{Z}}
\boldsymbol{\Sigma}_{-}
\boldsymbol{M}_{\phi,\boldsymbol{Z}}^\top
\right)
\right]
\\
&-
\mathbb{E}_{\boldsymbol{X}^{-}}
\left[
D_{\mathrm{KL}}
\left(
q_{\phi}(\boldsymbol{Z})
\,\Vert\,
p(\boldsymbol{Z}\mid\boldsymbol{X}^{-})
\right)
\right].
\end{aligned}
\end{equation}
Thus, for static PE, the population ELBO admits a direct covariance-based
representation in terms of the historical covariance
$\boldsymbol{\Sigma}_{-}$ and the future--history cross-covariance
$\boldsymbol{\Sigma}_{+-}$.

\subsection{Step 5: Prior and Posterior of Temporal Evidence}
\label{app:temporal_evidence_prior_posterior}

We finally specify the prior and PE-induced variational posterior of the
latent temporal evidence. For the lag-wise population analysis, we
aggregate the pairwise temporal evidence in
Proposition~\ref{proposition_position_encoding} according to relative lag
and represent the resulting temporal structure through a latent process
$Z(\tau)$, where $\tau=|i-j|$. The prior is induced by the Gaussian
process governing the historical input, whereas the variational posterior
is parameterized directly by the lag-wise positional evidence supplied by
PE.

\paragraph{GP-derived prior over temporal evidence.}
Let
$\mathcal{T}=\{1,\ldots,K\}$, with $K\leq T-1$, denote the set of
relative lags considered in the theoretical analysis. For each
$\tau\in\mathcal{T}$, consider the empirical autocovariance statistic
\begin{equation}
\label{equ_empirical_temporal_evidence}
U_{\tau}
=
\frac{1}{T-\tau}
\sum_{t=1}^{T-\tau}
X_tX_{t+\tau}.
\end{equation}
Since $\{X_t\}$ is a zero-mean weakly stationary Gaussian process,
\begin{equation}
\label{equ_temporal_evidence_mean}
\mathbb{E}[U_{\tau}]
=
\Gamma(\tau).
\end{equation}
For $h,k\in\mathcal{T}$, Isserlis' theorem further gives
\begin{equation}
\label{equ_temporal_evidence_covariance}
\begin{aligned}
k_Z(h,k)
:={}&
\operatorname{Cov}(U_h,U_k)
\\
={}&
\frac{1}{(T-h)(T-k)}
\sum_{t=1}^{T-h}
\sum_{s=1}^{T-k}
\Big[
\Gamma(t-s)\Gamma(t+h-s-k)
\\
&\hspace{35mm}
+
\Gamma(t-s-k)\Gamma(t+h-s)
\Big].
\end{aligned}
\end{equation}
Thus, the first- and second-order statistics of the lag-wise temporal
evidence are determined by the autocovariance function
$\Gamma(\tau)$ of the underlying Gaussian process.

Motivated by these exact first- and second-order moments, we adopt a
moment-matched Gaussian-process prior for the latent temporal evidence on
the finite lag domain $\mathcal{T}$,
\begin{equation}
\label{equ_temporal_evidence_gp_prior}
Z(\tau)
\sim
\mathcal{GP}
\left(
\Gamma(\tau),
k_Z(\tau,\tau')
\right),
\qquad
\tau,\tau'\in\mathcal{T}.
\end{equation}
This Gaussian construction is a moment-matched prior: the mean and
covariance are inherited from the empirical autocovariance statistics
$U_\tau$, without requiring the statistics themselves to be jointly
Gaussian.
For the finite lag set $\mathcal{T}$, define
\begin{equation}
\label{equ_temporal_evidence_vector}
\boldsymbol{Z}
=
[Z(1),\ldots,Z(K)]^\top,
\qquad
\boldsymbol{\gamma}
=
[\Gamma(1),\ldots,\Gamma(K)]^\top,
\end{equation}
and let
$\boldsymbol{K}_Z\in\mathbb{R}^{K\times K}$
satisfy
$[\boldsymbol{K}_Z]_{hk}=k_Z(h,k)$. 
Motivated by these exact moments, we specify the conditional temporal-evidence prior used in the ELBO as
\begin{equation}
\label{equ_temporal_evidence_prior}
p(\boldsymbol{Z}\mid\boldsymbol{X}^{-})
=
\mathcal{N}
\left(
\boldsymbol{\gamma},
\boldsymbol{K}_0
\right),
\qquad
\boldsymbol{K}_0
=
\boldsymbol{K}_Z+\epsilon_0\boldsymbol{I}_K,
\end{equation}
where $\epsilon_0>0$ ensures positive definiteness.
Because $\boldsymbol{K}_Z$ is a covariance matrix, it is positive
semidefinite by construction.

We retain the conditional notation
$p(\boldsymbol{Z}\mid\boldsymbol{X}^{-})$
to emphasize that the temporal-evidence prior is associated with the
Gaussian process governing the historical input. Under a fixed weakly
stationary process, $\Gamma(\tau)$, and hence
$\boldsymbol{\gamma}$ and $\boldsymbol{K}_0$, are invariant across
historical realizations. Under regime-conditioned non-stationarity,
however, these quantities may vary across regimes through the
regime-specific autocovariance function.

\paragraph{Lag-wise PE profiles.}
We next characterize the temporal evidence represented by PE.
Let
$$
\mathcal{Q}_K
=
\left\{
i:
i-\tau\in\mathcal{C}(i),
\ \forall\,\tau\in\mathcal{T}
\right\}
$$
denote the query positions for which all $K$ relative lags are
accessible. For each $i\in\mathcal{Q}_K$, define the lag-wise PE profile
$$
\boldsymbol{b}_{\phi,i}(\boldsymbol{x}^{-})
=
\left[
b_{i,i-1}^{\phi}(\boldsymbol{x}^{-}),
\ldots,
b_{i,i-K}^{\phi}(\boldsymbol{x}^{-})
\right]^\top
\in\mathbb{R}^{K}.
$$
Since Proposition~\ref{proposition_position_encoding} interprets PE scores
as log-likelihood contributions rather than realizations of
$\boldsymbol{Z}$ themselves, we do not directly identify
$b_{ij}^{\phi}$ with $Z_{ij}$. Instead, we use the calibrated PE profile
as a variational parameterization of the latent temporal-evidence
posterior:
\begin{equation}
\label{equ_calibrated_pe_profile}
\boldsymbol{e}_{\phi,i}(\boldsymbol{x}^{-})
=
\frac{
\boldsymbol{b}_{\phi,i}(\boldsymbol{x}^{-})
-
c\boldsymbol{1}_K
}{a},
\qquad
a>0,
\end{equation}
where $a>0$ and $c\in\mathbb{R}$ are calibration constants aligning the
PE-score representation with the latent temporal-evidence coordinates.

\paragraph{PE-induced variational posterior.}
We approximate the temporal evidence represented by PE with the
multivariate Gaussian variational posterior
\begin{equation}
\label{equ_pe_posterior}
q_{\phi}
\left(
\boldsymbol{Z}\mid\boldsymbol{x}^{-}
\right)
=
\mathcal{N}
\left(
\boldsymbol{\mu}_{\phi}(\boldsymbol{x}^{-}),
\boldsymbol{K}_{\phi}(\boldsymbol{x}^{-})
\right).
\end{equation}
Let $N_K=|\mathcal{Q}_K|$. We parameterize the posterior mean by the
average lag-wise PE profile,
\begin{equation}
\label{equ_pe_posterior_mean}
\boldsymbol{\mu}_{\phi}(\boldsymbol{x}^{-})
=
\frac{1}{N_K}
\sum_{i\in\mathcal{Q}_K}
\boldsymbol{e}_{\phi,i}(\boldsymbol{x}^{-}),
\end{equation}
so its $\tau$th component represents the average PE-induced temporal
evidence associated with lag $\tau$.

The posterior covariance is parameterized by the variation of these
lag-wise PE profiles,
\begin{equation}
\label{equ_pe_posterior_covariance}
\boldsymbol{K}_{\phi}(\boldsymbol{x}^{-})
=
\frac{1}{N_K}
\sum_{i\in\mathcal{Q}_K}
\left(
\boldsymbol{e}_{\phi,i}(\boldsymbol{x}^{-})
-
\boldsymbol{\mu}_{\phi}(\boldsymbol{x}^{-})
\right)
\left(
\boldsymbol{e}_{\phi,i}(\boldsymbol{x}^{-})
-
\boldsymbol{\mu}_{\phi}(\boldsymbol{x}^{-})
\right)^\top
+
\epsilon_q\boldsymbol{I}_K,
\end{equation}
where $\epsilon_q>0$ ensures positive definiteness.
Therefore,
$q_{\phi}(\boldsymbol{Z}\mid\boldsymbol{x}^{-})$
summarizes the average lag-wise temporal evidence encoded by PE together with its variation across query positions.

For a static PE, the positional scores
$b_{ij}^{\phi}$ do not depend on the historical realization
$\boldsymbol{x}^{-}$. Consequently,
$\boldsymbol{\mu}_{\phi}(\boldsymbol{x}^{-})$
and
$\boldsymbol{K}_{\phi}(\boldsymbol{x}^{-})$
are shared across input windows, yielding
\begin{equation}
\label{equ_static_posterior_from_pe}
q_{\phi}
(\boldsymbol{Z}\mid\boldsymbol{x}^{-})
=
q_{\phi}(\boldsymbol{Z}).
\end{equation}
In contrast, an adaptive PE produces input-dependent positional scores
and therefore induces a window-specific variational posterior
$q_{\phi}(\boldsymbol{Z}\mid\boldsymbol{x}^{-})$.

\paragraph{Posterior--prior discrepancy.}
The GP-derived prior
$p(\boldsymbol{Z}\mid\boldsymbol{x}^{-})$
represents the temporal dependence supported by the underlying process,
whereas
$q_{\phi}(\boldsymbol{Z}\mid\boldsymbol{x}^{-})$
represents the temporal dependence encoded by PE.
Since both distributions are Gaussian,
\begin{equation}
\label{equ_temporal_evidence_gaussian_kl}
\begin{aligned}
&
D_{\mathrm{KL}}
\left(
q_{\phi}
(\boldsymbol{Z}\mid\boldsymbol{x}^{-})
\,\Vert\,
p(\boldsymbol{Z}\mid\boldsymbol{x}^{-})
\right)
\\
={}&
\frac{1}{2}
\Bigg[
\operatorname{tr}
\left(
\boldsymbol{K}_0^{-1}
\boldsymbol{K}_{\phi}(\boldsymbol{x}^{-})
\right)
+
\left(
\boldsymbol{\mu}_{\phi}(\boldsymbol{x}^{-})
-
\boldsymbol{\gamma}
\right)^\top
\boldsymbol{K}_0^{-1}
\left(
\boldsymbol{\mu}_{\phi}(\boldsymbol{x}^{-})
-
\boldsymbol{\gamma}
\right)
-K
+
\log
\frac{
\det(\boldsymbol{K}_0)
}{
\det(
\boldsymbol{K}_{\phi}(\boldsymbol{x}^{-})
)
}
\Bigg].
\end{aligned}
\end{equation}
Accordingly, the posterior--prior term in the population ELBO is
\begin{equation}
\label{equ_population_pe_prior_kl}
\mathbb{E}_{\boldsymbol{X}^{-}}
\left[
D_{\mathrm{KL}}
\left(
q_{\phi}
(\boldsymbol{Z}\mid\boldsymbol{X}^{-})
\,\Vert\,
p(\boldsymbol{Z}\mid\boldsymbol{X}^{-})
\right)
\right].
\end{equation}
This term quantifies the population-level discrepancy between the
temporal structure represented by PE and that supported by the Gaussian
process governing the historical input.

\paragraph{Conclusion of the proof.}
Combining Steps~1--5, the population-level variational formulation
provides a genuine lower bound on the expected forecasting
log-likelihood while preserving the input dependence of
$q_{\phi}(\boldsymbol{Z}\mid\boldsymbol{X}^{-})$.
The GP governing the historical input determines the temporal-evidence
prior
$p(\boldsymbol{Z}\mid\boldsymbol{X}^{-})$, whereas PE determines the
variational posterior
$q_{\phi}(\boldsymbol{Z}\mid\boldsymbol{X}^{-})$
through the first- and second-order statistics of its lag-wise
positional profiles.

The forecasting component is characterized by the
$q_\phi$-weighted historical and future--history second-order moments in
Equation~\ref{equ_population_elbo_weighted_covariance}.
For static PE, the posterior is shared across input windows, so these
weighted moments factor into the population covariance matrices
$\boldsymbol{\Sigma}_{-}$ and
$\boldsymbol{\Sigma}_{+-}$, yielding
Equation~\ref{equ_population_elbo_static_covariance}.
The posterior--prior term in
Equation~\ref{equ_population_pe_prior_kl}
further measures the discrepancy between the temporal structure encoded
by PE and that supported by the historical process.
Together, these results establish
Theorem~\ref{theorem_population_elbo}.

\section{Proof of Corollary~\ref{corollary_regime_population_elbo}}
\label{app:corollary_regime_population_elbo}

\textit{Proof.}
Conditional on regime $R=r$, the forecasting distribution
$$
\mathcal{D}_r
=
p(\boldsymbol{X}^{-},\boldsymbol{Y}\mid R=r)
$$
is weakly stationary Gaussian by assumption. Therefore,
Theorem~\ref{theorem_population_elbo} applies within each regime and
induces the regime-specific population ELBO
$\mathcal{L}_r(\phi)$, with temporal-evidence prior
$p_r(\boldsymbol{Z}|\boldsymbol{X}^{-})$ determined by the regime-specific autocovariance
function $\Gamma_r(\tau)$.

Taking expectation over the regime variable $R$ and using the law of
total expectation gives
\begin{equation}
\label{equ_regime_elbo_decomposition}
\mathcal{L}_{\mathrm{reg}}(\phi)
=
\sum_{r=1}^{J}
\pi_r\mathcal{L}_r(\phi).
\end{equation}
Thus, the population objective under regime-conditioned
non-stationarity is a weighted combination of the regime-specific
population ELBOs.

For any static PE
$\phi\in\Phi_{\mathrm{static}}$, we have, for every regime $r$,
$$
\mathcal{L}_r(\phi)
\leq
\sup_{\phi_r\in\Phi}
\mathcal{L}_r(\phi_r).
$$
Multiplying by $\pi_r>0$, summing over regimes, and then taking the
supremum over $\phi\in\Phi_{\mathrm{static}}$ yields
\begin{equation}
\label{equ_static_pe_upper_bound}
\sup_{\phi\in\Phi_{\mathrm{static}}}
\mathcal{L}_{\mathrm{reg}}(\phi)
\leq
\sum_{r=1}^{J}
\pi_r
\sup_{\phi_r\in\Phi}
\mathcal{L}_r(\phi_r).
\end{equation}

To characterize when the inequality is strict, define
$$
\mathcal{L}_r^\star
:=
\sup_{\phi_r\in\Phi}
\mathcal{L}_r(\phi_r)
$$
and, for a static PE $\phi$, its aggregate regime-wise optimality gap
$$
\Delta(\phi)
:=
\sum_{r=1}^{J}
\pi_r
\left[
\mathcal{L}_r^\star-\mathcal{L}_r(\phi)
\right]
\geq 0.
$$
Equation~\ref{equ_static_pe_upper_bound} is strict whenever
$$
\inf_{\phi\in\Phi_{\mathrm{static}}}
\Delta(\phi)
>
0.
$$
Equivalently, because $J$ is finite and $\pi_r>0$, equality can hold only
if there exists a sequence of static PE methods
$\{\phi_n\}\subset\Phi_{\mathrm{static}}$ satisfying
$$
\mathcal{L}_r(\phi_n)
\longrightarrow
\mathcal{L}_r^\star,
\qquad
\forall r\in\{1,\ldots,J\}.
$$
Hence, if no single sequence of static PE methods can simultaneously
approach the regime-specific optimal ELBOs, the inequality in
Equation~\ref{equ_static_pe_upper_bound} is strict.

For static PE, the induced temporal-evidence posterior
$q_\phi(\boldsymbol{Z})$ is shared across regimes, whereas the GP-derived
prior $p_r(\boldsymbol{Z}\mid\boldsymbol{X}^{-})$ and the covariance
structure entering $\mathcal{L}_r(\phi)$ may vary with the regime.
Consequently, heterogeneous regimes can induce incompatible
regime-specific ELBO objectives, forcing a static PE to compromise across
them. An input-adaptive PE can instead vary its temporal evidence across
forecasting windows, motivating the adaptive positional mechanism
developed in this work. This establishes
Corollary~\ref{corollary_regime_population_elbo}.

\section{The demonstration of Theorem \ref{theorem_population_elbo}}
\label{app:theorem_demonstration_synthetic}

\subsection{Binary-Regime Synthetic Dataset}
\label{sec:synthetic_data}

We generate independent trajectories from two zero-mean Gaussian process
(GP) regimes. For each trajectory, a regime
$r \sim \operatorname{Bernoulli}(1/2)$ is sampled once and held fixed
throughout the trajectory. The regime-specific covariance functions are
\begin{align}
k_{\mathrm{smooth}}(t,t')
&= \exp\!\left(-\frac{(t-t')^2}{2(28)^2}\right), \\
k_{\mathrm{periodic}}(t,t')
&= \exp\!\left(
-\frac{2\sin^2\!\bigl(\pi(t-t')/24\bigr)}{(0.55)^2}
\right).
\end{align}
Given the selected regime, we draw
$f_{0:67}\sim\mathcal{GP}(0,k_r)$ and observe
$x_t=f_t+\epsilon_t$, where
$\epsilon_t\overset{\mathrm{iid}}{\sim}\mathcal{N}(0,0.1^2)$.
The forecasting input is $x_{0:63}$, and the four-step target is
$x_{64:67}$.

We generate 256 training, 64 validation, and 64 test trajectories
\emph{per regime}, with independent trajectories across splits.
Inputs and targets are standardized using the mean and standard deviation
of the training inputs. Regime labels are retained for evaluation but
are not provided to the forecasting models.

\begin{figure}[!htbp]
  \centering
  \includegraphics[width=0.99\textwidth]{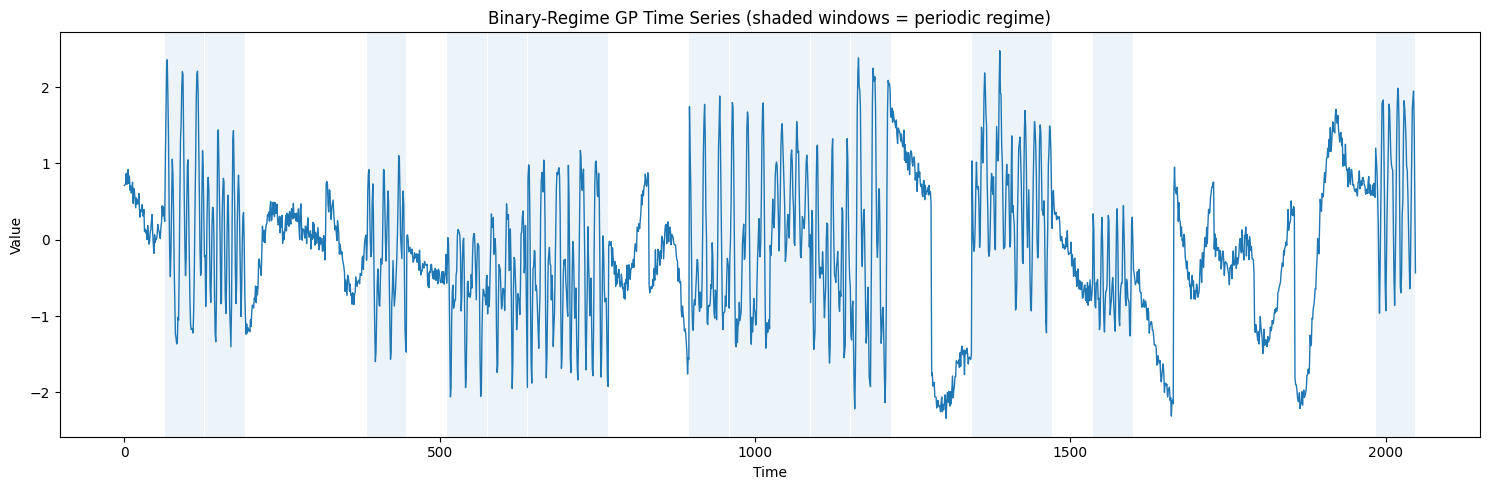}
  \caption{Synthetic binary-regime GP time series of 2,048 observations.
Each 64-step window is generated from either a smooth RBF GP
(unshaded) or a periodic GP (shaded). The periodic regime has
period 24.}
  \label{Img_synthetic_dataset}
\end{figure}

\subsection{Simplified Transformer for Forecasting}
\label{sec:binary_regime_transformer}

We use a single-layer, single-head Transformer with model dimension \(d_{\mathrm{model}}=32\). For a historical input \(\boldsymbol{x}^{-}\in\mathbb{R}^{64}\), a point-wise linear layer embeds each observation, and sinusoidal absolute positional encoding (APE) is added in every PE configuration considered below. The encoder computes
\[
\boldsymbol{A}_{\phi}(\boldsymbol{x}^{-})
=
\operatorname{Softmax}\!\left(
\frac{\boldsymbol{Q}\boldsymbol{K}^{\top}}{\sqrt{32}}
+\boldsymbol{B}_{\phi}(\boldsymbol{x}^{-})
\right),
\qquad
\boldsymbol{H}'=
\boldsymbol{H}
+\boldsymbol{W}_{O}\!\left(\boldsymbol{A}_{\phi}\boldsymbol{V}\right),
\]
where the softmax is applied across keys and \(\boldsymbol{B}_{\phi}\) denotes the contribution of RPE, ALiBi, RoPE, or BTA to the attention scores. For RoPE, this notation represents the difference between rotated and unrotated query--key scores; that contribution depends on the input representations. BTA generates an input-conditioned bias by combining RBF experts initialized with length scales \(7,14,28,42\) and periodic experts initialized with period \(24\) and length scales \(0.35,0.55,0.9\).

The 64 encoder outputs are averaged and passed through a linear decoder to produce \(\widehat{\boldsymbol{Y}}\in\mathbb{R}^{4}\). We train each model to minimize mean squared forecasting error using AdamW, a batch size of 64, a learning rate of \(10^{-3}\), weight decay of \(10^{-4}\), and at most 30 epochs. The best model is selected by validation MSE, with early stopping after seven epochs without improvement. For BTA, we additionally select the router width from \(\{16,32,64,128\}\) and the number of selected experts from \(\{1,3,5,7\}\), using validation data only. All methods use the same data splits and random seed.

We train the simplified Transformer with four different PE methods including RPE \citep{liutkus2021relative}, ALiBi \citep{press2021train}, RoPE \citep{su2024roformer}, and BTA on the binary-regime synthetic dataset. The performance of the simplified Transformer with the four different PE methods is summarized in Table~\ref{tab:synthetic_forecasting} and the forecasting result given an input time series is shown in Figure~\ref{binary_regime_h4_forecast_examples}.
\begin{table}[!htbp]
\centering
\caption{The forecasting performance on the binary-regime GP dataset.
MSE is reported on standardized observations; lower is better.
Smooth and periodic denote the two subsets of the test set.}
\label{tab:synthetic_forecasting}
\fontfamily{ptm}\fontsize{10}{11}\selectfont
\setlength{\tabcolsep}{6pt}
\renewcommand{\arraystretch}{1}
\begin{tabular}{lccccc}
\toprule
PE Method & Best epoch & Val. MSE & Test MSE & Smooth MSE & Periodic MSE \\
\midrule
RPE   & 30 & 0.2578 & 0.2337 & 0.1792 & 0.2882 \\
RoPE  & 30 & 0.2364 & 0.2124 & 0.1354 & 0.2893 \\
ALiBi & 30 & 0.2573 & 0.2336 & 0.1801 & 0.2871 \\
BTA   & 30 & \textbf{0.2049} & \textbf{0.1788}
         & \textbf{0.0993} & \textbf{0.2583} \\
\bottomrule
\end{tabular}
\end{table}

\begin{figure}[!htbp]
  \centering
  \includegraphics[width=0.99\textwidth]{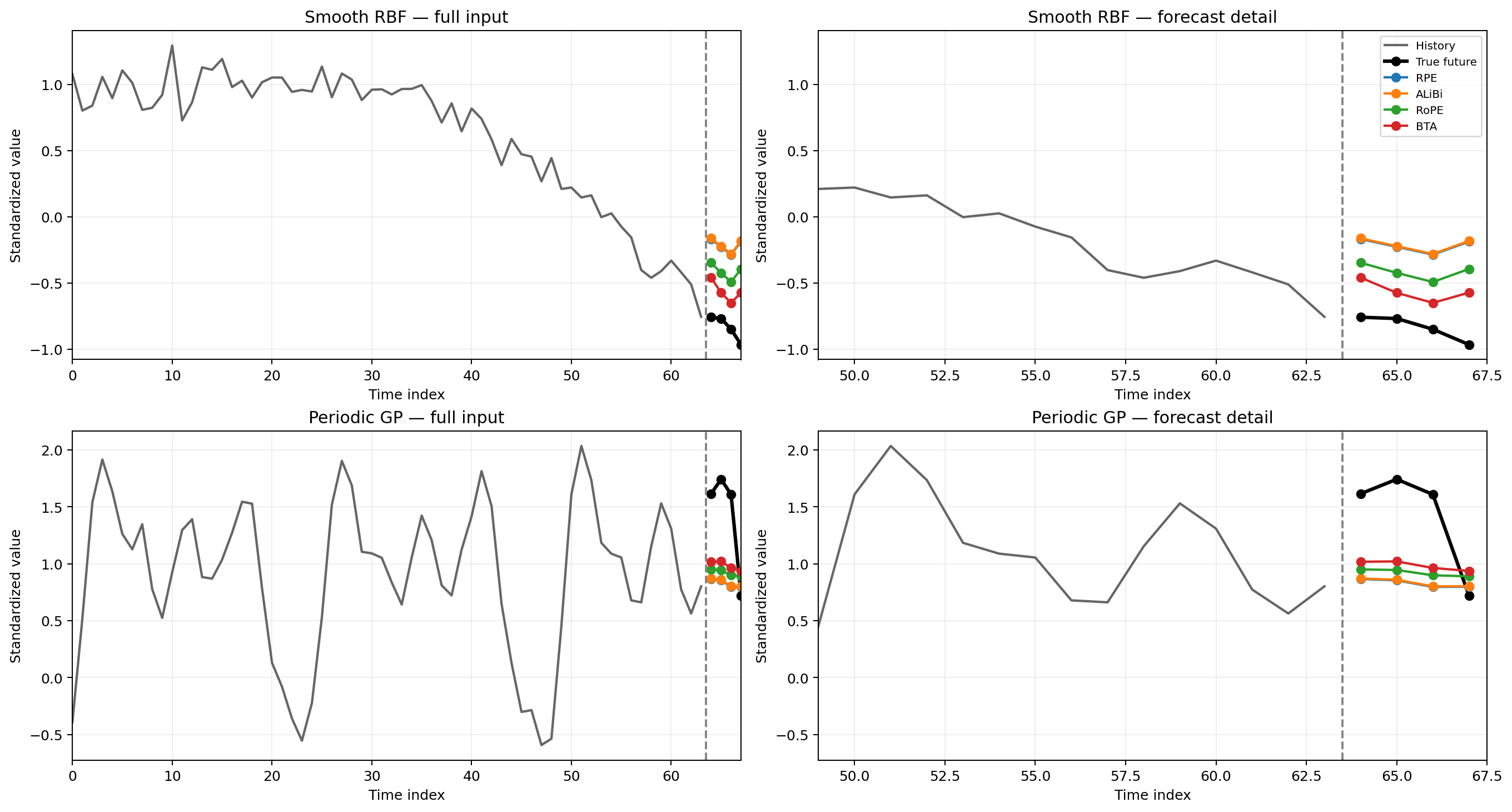}
  \caption{Forecasting results from the smooth
RBF GP (top row) and periodic GP (bottom row). The left panels show all
64 historical observations and the four future values; the right panels
magnify the forecast boundary. Gray lines denote the observed history,
black markers the true future, and colored markers the predictions from the simplified Transformer with the four PE methods:
RPE, ALiBi, RoPE, and BTA. The dashed line marks the end of the input
window.}
  \label{binary_regime_h4_forecast_examples}
\end{figure}

In addition, Figure~\ref{Img_PE_bias} visualizes the positional bias induced by the four different PE methods. 
\begin{figure}[!htbp]
  \centering
  \includegraphics[width=0.99\textwidth]{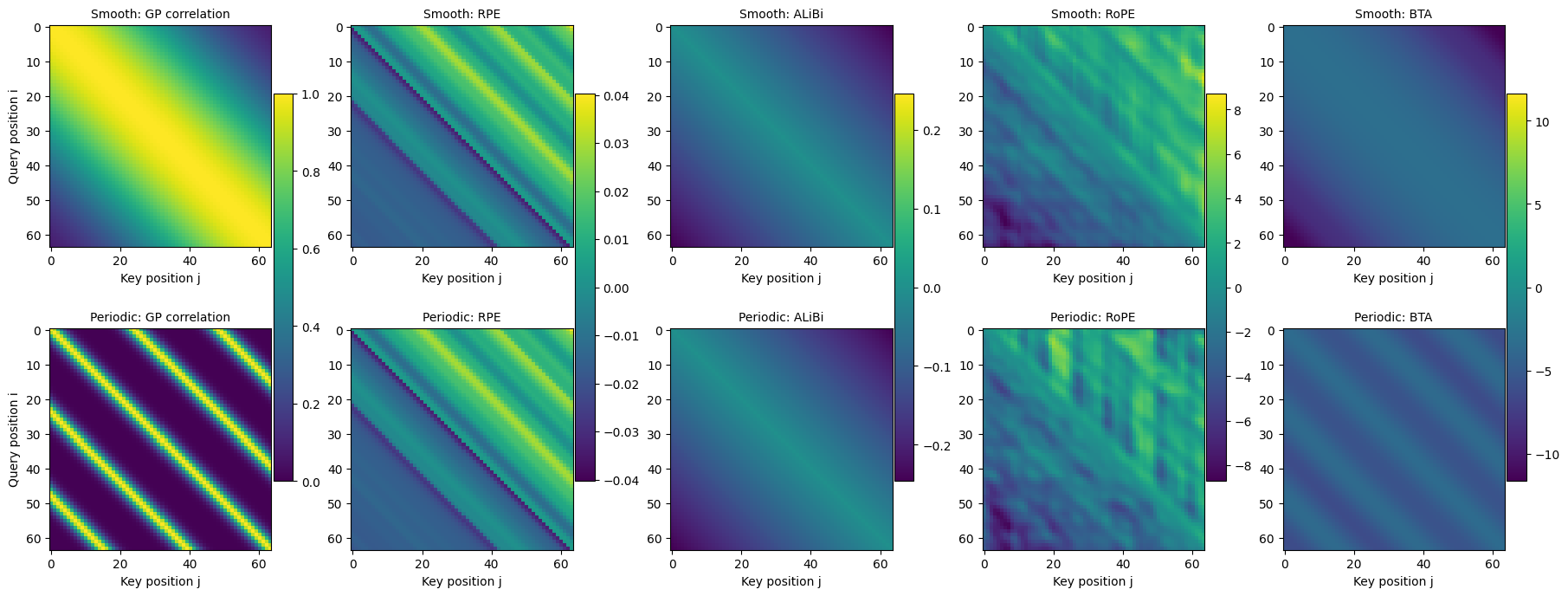}
  \caption{The figure compares positional biases under smooth and periodic Gaussian process regimes. RPE and ALiBi retain the same bias pattern across both regimes, while BTA produces different biases for the two inputs, illustrating its ability to adapt temporal evidence to the regime.}
  \label{Img_PE_bias}
\end{figure}

\subsection{Synthetic Validation of the Population-Level ELBO}
\label{sec:synthetic_elbo}

We design a binary-regime synthetic dataset to illustrate the
population-level ELBO in Theorem~2 and examine the value of adaptive
positional encoding (PE) when temporal dependencies change across
forecasting windows. Each regime is generated by a zero-mean, weakly
stationary Gaussian process (GP): the smooth regime uses an RBF
covariance, while the periodic regime uses a periodic covariance with
a 24-step cycle. Each model receives a 64-step history and predicts
the next four observations. We train the same single-head Transformer
with RPE, ALiBi, RoPE, or BTA; BTA uses its validation-selected router
configuration.

For each trained model, we extract the positional score matrix
$b_\phi(\boldsymbol{x}^{-})$ and construct one-sided score profiles
over lags $1,\ldots,K$. We set $K=32$ for the 64-step input windows.
Here $K$ specifies the maximum lag used to estimate temporal evidence,
not the input length: it leaves $T-K=32$ query positions with complete
one-sided lag profiles and includes the periodic regime's 24-step
cycle. The forecaster continues to use the full 64-step history.
A method-specific affine calibration is estimated from training
windows and then held fixed on test windows. As in Appendix~D, the
mean and query-wise covariance of the calibrated profiles define
\[
q_\phi(Z\mid\boldsymbol{x}^{-})
=
\mathcal{N}\!\left(
\boldsymbol{\mu}_\phi(\boldsymbol{x}^{-}),
\boldsymbol{K}_\phi(\boldsymbol{x}^{-})
\right).
\]
We add $\epsilon_q I$ to the posterior covariance, with
$\epsilon_q=10^{-4}$. For each regime $r$, the GP autocovariance
determines the mean and covariance of the lag-wise empirical
autocovariance statistics, yielding the moment-matched
temporal-evidence prior
$p_r(Z)=\mathcal{N}(\boldsymbol{\gamma}_r,
\boldsymbol{K}_{0,r})$.

We estimate the forecasting term by sampling $S=16$ values from
$q_\phi(Z\mid\boldsymbol{x}^{-}_n)$ for each test window. A sampled
value perturbs the trained model's observed positional scores at the
corresponding lags; the trained parameters and observed content
scores remain fixed. At
$Z=\boldsymbol{\mu}_\phi(\boldsymbol{x}^{-}_n)$, the perturbation
vanishes and the original forecast is recovered. Using the same
Gaussian observation variance $\sigma^2=1$ for every method, we
compute
\begin{equation}
\widehat{\mathcal{L}}(\phi)
=
\frac{1}{N}\sum_{n=1}^{N}
\left[
\frac{1}{S}\sum_{s=1}^{S}
\log p_{\theta,\phi}
\!\left(
\boldsymbol{y}_n\mid\boldsymbol{x}^{-}_n,Z_{ns}
\right)
-
D_{\mathrm{KL}}\!\left(
q_\phi(Z\mid\boldsymbol{x}^{-}_n)
\,\|\,p_{r_n}(Z)
\right)
\right],
\quad
Z_{ns}\sim q_\phi(Z\mid\boldsymbol{x}^{-}_n),
\label{eq:synthetic_elbo_estimate}
\end{equation}
where $r_n$ denotes the regime of window $n$. The first term
measures forecasting fit under sampled temporal evidence; the
second measures how closely the PE-induced evidence agrees with
the window's regime-specific GP prior. Their average is a
Monte Carlo estimate of the population-level variational objective
for this trained, input-dependent perturbation model.

\begin{table}[!htbp]
\centering
\caption{Forecasting performance and estimated ELBO on the
binary-regime synthetic test set with $\epsilon_q=10^{-4}$.
The Test MSE column repeats the point-forecast results in Table~\ref{tab:synthetic_forecasting}. The remaining columns report the Monte Carlo diagnostics under sampled temporal evidence defined in Equation~\ref{eq:synthetic_elbo_estimate}; expected log likelihood is not computed from the point-forecast MSE. Lower MSE and higher ELBO are better.}
\label{tab:synthetic_elbo}
\fontfamily{ptm}\fontsize{10}{11}\selectfont
\setlength{\tabcolsep}{6pt}
\renewcommand{\arraystretch}{1}
\begin{tabular}{lrrrr}
\toprule
PE & Test MSE & Expected log likelihood & Mean KL & ELBO \\
\midrule
RPE   & 0.2337 & $-4.1431$ & 348.0968 & $-352.2399$ \\
ALiBi & 0.2336 & $-4.1429$ & 271.2189 & $-275.3619$ \\
RoPE  & 0.2124 & $-4.1105$ &  51.4227 &  $-55.5333$ \\
BTA   & \textbf{0.1788} & $\mathbf{-4.0492}$ &
        \textbf{46.6581} & $\mathbf{-50.7073}$ \\
\bottomrule
\end{tabular}
\end{table}

As shown in Table~\ref{tab:synthetic_elbo}, BTA obtains both the
lowest forecasting MSE and the highest estimated ELBO. Its expected
log likelihood is higher and its posterior--prior KL is lower than
those of the baseline PE methods. Here RPE and ALiBi use input-independent biases, whereas RoPE has a content-dependent positional score under our definition of static PE. The result is consistent with the
hypothesis underlying Theorem~2 and Corollary~3: when different
windows have smooth or periodic temporal dependence, an adaptive
PE can supply more appropriate temporal evidence than a shared
static positional structure. Thus, the synthetic dataset provides
an empirical illustration of the population-level ELBO framework
and its motivation for adaptive PE in non-stationary forecasting.
The reported values pertain to the explicitly defined
input-dependent perturbation model in
Equation~\ref{eq:synthetic_elbo_estimate}; they are not a direct
numerical evaluation of the static covariance-reduced formula
in Theorem~2.

\section{Data Description}
\label{app:datasets}
We evaluate non-stationary LTSF models on eight widely used multivariate time-series benchmarks spanning diverse sampling frequencies, dimensionalities, and temporal dynamics. Table~\ref{tab:dataset_statistics} summarizes their characteristics and forecasting protocols. Following established settings~\citep{liu2025timebridge}, we use chronological train/validation/test splits and prediction horizons $H\in\{96,192,336,720\}$. All data-dependent preprocessing parameters are estimated exclusively from the training set and applied unchanged to the validation and test sets. We additionally report the Augmented Dickey--Fuller (ADF) statistic as an indicator of stationarity: under a fixed test specification, a less negative ADF statistic indicates weaker evidence against the unit-root null and thus comparatively stronger non-stationary behavior.

\begin{itemize}
    \item \textbf{ETT}~\citep{zhou2021informer} contains transformer temperature and power-load measurements from two electricity-transformer stations in China between 2016 and 2018. It comprises two hourly datasets (ETTh1 and ETTh2) and two 15-minute datasets (ETTm1 and ETTm2), each with seven variables.
    
    \item \textbf{Weather}~\citep{wu2023timesnet} contains 21 meteorological variables recorded in Germany at 10-minute intervals during 2020.
    
    \item \textbf{Electricity}~\citep{wu2023timesnet} contains hourly electricity-consumption measurements from 321 customers between 2012 and 2014.
    
    \item \textbf{Traffic}~\citep{wu2023timesnet} contains hourly freeway-occupancy measurements from 862 sensors in the San Francisco Bay Area during 2015--2016.
    
    \item \textbf{Solar-Energy}~\citep{liu2024itransformer} contains 10-minute power-generation measurements from 137 photovoltaic plants in Alabama during 2006.
\end{itemize}

\begin{table}[!htbp]
\centering
\caption{
Dataset characteristics and benchmark split-size summary, following Table~7 of TimeBridge~\citep{liu2025timebridge}.
$C$ denotes the number of variables. Training, Validation, Testing, and ADF entries reproduce that source. Split sizes are provided as benchmark metadata; the number of usable forecasting windows depends on the input length, prediction horizon, and handling of historical overlap. Interval denotes the sampling period.
}
\label{tab:dataset_statistics}
\fontfamily{ptm}\fontsize{9}{10}\selectfont
\setlength{\tabcolsep}{4.2pt}
\renewcommand{\arraystretch}{1.08}
\begin{tabular}{@{}lcrrrcr@{}}
\toprule
Dataset & $C$ & Training & Validation & Testing & Interval & ADF stat. \\
\midrule
ETTh1        & 7   & 8,545  & 2,881  & 2,881  & 1 hour & $-5.91$  \\
ETTh2        & 7   & 8,545  & 2,881  & 2,881  & 1 hour & $-4.13$  \\
ETTm1        & 7   & 34,465 & 11,521 & 11,521 & 15 min & $-14.98$ \\
ETTm2        & 7   & 34,465 & 11,521 & 11,521 & 15 min & $-5.66$  \\
Electricity  & 321 & 18,317 & 2,633  & 5,261  & 1 hour & $-8.44$  \\
Traffic      & 862 & 12,185 & 1,757  & 3,509  & 1 hour & $-15.02$ \\
Weather      & 21  & 36,792 & 5,271  & 10,540 & 10 min & $-26.68$ \\
Solar-Energy & 137 & 36,601 & 5,161  & 10,417 & 10 min & $-37.23$ \\
\bottomrule
\end{tabular}
\end{table}

\section{Transformer-based backbone architectures}
\label{app:backbone_architecture}

To evaluate whether BTA generalizes across substantially different Transformer architectures rather than benefiting a particular backbone, we select the four Transformer-based backbone architectures because of the following reasons:
\begin{itemize}
    \item Informer represents an early efficient Transformer for long-sequence forecasting; \item PatchTST represents the widely adopted patch-based architecture; 
    \item TimeBridge is specifically designed to model non-stationary temporal dependencies; \item Timer-XL represents recent generative time-series foundation models.
\end{itemize}

Moreover, Table~\ref{tab:backbone_comparison} summarizes the key difference between the four backbone architectures. 

\begin{table}[!htbp]
\centering
\caption{Comparison of the Transformer backbones used in our experiments.}
\label{tab:backbone_comparison}
\fontfamily{ptm}\fontsize{9}{10}\selectfont
\setlength{\tabcolsep}{4pt}
\renewcommand{\arraystretch}{1.08}
\resizebox{\textwidth}{!}{
\begin{tabular}{c|cccc}
\toprule
Backbone
& Informer~\citep{zhou2021informer}
& PatchTST~\citep{nie2023time}
& TimeBridge~\citep{liu2025timebridge}
& Timer-XL~\citep{liu2025timer} \\
\midrule

Representative role
& Early LTSF Transformer
& Patch-based Transformer
& Non-stationary Transformer
& Pretrained TSFM \\

Architecture
& Encoder--decoder
& Encoder-only
& Encoder-only
& Decoder-only \\

Attention design
& ProbSparse
& Self-attention
& Integrated + PatchDownsampling + Cointegrated
& Causal TimeAttention \\

Tokenization
& Point-wise
& Patch-wise
& Patch-wise
& Patch-wise \\

Positional mechanism
& Sinusoidal absolute PE
& Learnable absolute PE
& None
& Rotary PE + variable bias \\

Training paradigm
& Task-specific
& Task-specific
& Task-specific
& Pretraining / task-specific \\

\bottomrule
\end{tabular}
}
\end{table}

\section{LTSF reference methods}
\label{app:ltsf_reference_methods}

To compare BTA with representative methods for non-stationary LTSF, we select ten models as baselines. They are divided into two groups. First, we include five methods specifically designed to address non-stationarity, distribution shifts, or heterogeneous temporal dynamics. Second, we include five widely used LTSF models spanning Transformer-, MLP-, convolution-, and linear-based architectures to evaluate BTA against representative general-purpose forecasting methods.

\begin{enumerate}
    \item \textbf{Methods for non-stationary forecasting:}
    \begin{itemize}
        \item \textbf{DTAF}~\citep{lu2026towards} addresses non-stationarity in both temporal and frequency domains through temporal stabilization and frequency differencing.

        \item \textbf{PULSE}~\citep{liupulse} models evolving non-stationary dynamics through phase-anchored disentanglement, generative phase evolution, and statistic-aware augmentation.
        
        \item \textbf{TimeBridge}~\citep{liu2025timebridge} employs Integrated and Cointegrated Attention to separately model stable short-term dependencies and non-stationary long-term cointegration.
        
        \item \textbf{DUET}~\citep{qiu2025duet} uses dual clustering over temporal distributions and channels to capture heterogeneous temporal patterns and dynamic inter-variable dependencies.
        
        \item \textbf{Non-stationary Transformer}~\citep{liu2022non} combines Series Stationarization with De-stationary Attention to improve predictability while preserving informative non-stationary dependencies.
        
    \end{itemize}

    \item \textbf{General LTSF methods:}
    \begin{itemize}
        \item \textbf{iTransformer}~\citep{liu2024itransformer} treats individual variates as tokens and applies self-attention across variables to capture multivariate dependencies.
        
        \item \textbf{TimeMixer}~\citep{wang2024timemixer} employs multiscale decomposition and mixing to jointly model fine-scale seasonal patterns and coarse-scale trends.
        
        \item \textbf{PatchTST}~\citep{nie2023time} represents time series as patch-level tokens and uses channel-independent Transformers to capture long-range temporal dependencies efficiently.
        
        \item \textbf{TimesNet}~\citep{wu2023timesnet} transforms one-dimensional time series into two-dimensional representations based on multiple periods to model intra- and inter-period temporal variations.
        
        \item \textbf{DLinear}~\citep{zeng2023transformers} decomposes a time series into trend and seasonal components and forecasts them using simple linear mappings.
    \end{itemize}
\end{enumerate}

\section{More Details of BTA Implementation}
\label{app:bta_implementation}

This section provides additional implementation details of BTA, including
backbone integration, temporal-kernel construction, MoE routing,
hyperparameter selection, and training.

\subsection{Backbone Integration}

BTA is designed to augment rather than replace the native attention mechanism
of each backbone. For the $h$th attention head of the $n$th input, the MoE
router is constructed from the content-based attention scores
\[
s_{ij}^{(n,h),\mathrm{content}}
=
\frac{
(\boldsymbol{q}_i^{(n,h)})^\top
\boldsymbol{k}_j^{(n,h)}
}{
\sqrt{d_k}
}.
\]
The adaptive positional bias generated by BTA is then added to the
backbone-specific native attention logits:
\[
s_{ij}^{(n,h),\mathrm{BTA}}
=
s_{ij}^{(n,h),\mathrm{native}}
+
b_{ij}^{(n,h),\mathrm{BTA}},
\]
where $s_{ij}^{(n,h),\mathrm{native}}$ includes the native positional
mechanism of the backbone when applicable. The resulting logits are
normalized by the original attention softmax.

We integrate BTA into Informer, PatchTST, TimeBridge, and Timer-XL through
the same additive modification. These backbones differ substantially in
tokenization, attention design, and positional mechanisms, but no other
backbone component is modified. Thus, BTA preserves the original architecture
while introducing an additional input- and head-adaptive temporal bias.

\subsection{Temporal Kernel Bank}

We construct an over-complete temporal kernel bank using radial basis function
(RBF) and periodic (PER) kernels to represent complementary local,
long-range, and seasonal dependencies. Let
$\tau_{ij}=|i-j|$ denote the relative token lag. For point-wise models,
$\tau_{ij}$ is measured in time steps, whereas for patch-based models it is
measured in patches.

\paragraph{RBF kernels.}
The RBF log-kernel is
\[
\psi_r^{\mathrm{RBF}}(\tau_{ij};\ell_r)
=
-\frac{\tau_{ij}^{2}}{2\ell_r^{2}},
\qquad r=1,\ldots,5,
\]
where $\ell_r$ controls the temporal scale. We parameterize each RBF scale
relative to the available input context,
\[
\ell_r=\rho_r L,
\qquad
\rho_r\in\mathcal{R}_{\mathrm{RBF}},
\]
where $L$ denotes the number of input tokens or patches and
$\mathcal{R}_{\mathrm{RBF}}$ contains five fixed scale coefficients.
This normalization allows the same kernel construction to be used across
backbones with different input lengths and tokenization schemes.

\paragraph{Periodic kernels.}
To model recurring temporal structure, we first identify the two dominant
periods
\[
p_e\in\{p_1,p_2\},
\qquad e\in\{1,2\},
\]
using Fourier analysis of the training data only. For each dominant period,
we instantiate three periodic kernels with different peak widths:
\[
\psi_{r,e}^{\mathrm{PER}}
(\tau_{ij};\eta_r,p_e)
=
-\frac{
2\sin^2(\pi\tau_{ij}/p_e)
}{
\eta_r^2
},
\qquad
r=1,2,3,
\quad
e=1,2,
\]
where $\eta_r$ is a dimensionless width parameter. The corresponding
temporal peak width can equivalently be expressed as
\[
\ell'_{r,e}=\eta_r p_e.
\]
Thus, the periodic kernel bank contains
$2\times3=6$ candidate kernels, spanning two dominant periods and three
temporal widths for each period.

The complete kernel bank therefore contains
\[
M
=
5\ \text{RBF kernels}
+
6\ \text{periodic kernels}
=
11
\]
candidate temporal patterns. All kernel parameters are fixed before model
training; BTA learns only how to combine these candidate kernels for each
input and attention head.

\subsection{MoE Router Configuration}

For each input $n$ and attention head $h$, we aggregate the content-based
attention scores associated with the same relative lag to construct the
one-dimensional temporal-evidence (TE) routing profile
\[
e_{\tau}^{(n,h)}
=
\frac{1}{|\mathcal{I}_{\tau}|}
\sum_{(i,j)\in\mathcal{I}_{\tau}}
s_{ij}^{(n,h),\mathrm{content}},
\qquad
\mathcal{I}_{\tau}
=
\{(i,j):|i-j|=\tau\}.
\]
This profile summarizes how strongly the learned content representations
associate query--key pairs at different temporal lags.

To capture complementary lag-domain and frequency-domain information, the
router receives
\[
\left[
\boldsymbol{e}^{(n,h)};
\left|
\mathcal{F}
\left(
\boldsymbol{e}^{(n,h)}
\right)
\right|
\right],
\]
where $\mathcal{F}(\cdot)$ denotes the discrete Fourier transform and
$|\mathcal{F}(\cdot)|$ denotes its magnitude spectrum.

The router is a one-hidden-layer MLP with $M=11$ outputs, corresponding to
the five RBF and six periodic candidate kernels. Let
$\boldsymbol{g}^{(n,h)}\in\mathbb{R}^{M}$ denote its output logits.
We retain the Top-$k$ logits and apply softmax over the selected experts,
setting the remaining mixture weights to zero. This yields
\[
\boldsymbol{\omega}^{(n,h)}
=
[
\omega_1^{(n,h)},\ldots,\omega_M^{(n,h)}
],
\qquad
\sum_{c=1}^{M}\omega_c^{(n,h)}=1.
\]
The resulting input- and head-specific mixture weights are used in
Equation~\ref{equ_bta_bias} to construct the adaptive positional bias.

\subsection{Hyperparameter Configuration}

The principal BTA hyperparameters are the router hidden dimension $d_r$ and
the number of retained experts $k$. We search
\[
k\in\{1,3,5,7\}
\]
and use backbone-specific search grids for $d_r$, as summarized in
Table~\ref{tab:bta_configuration}.

Each router-dimension grid is represented by
$(d_{\mathrm{start}},d_{\mathrm{ub}},\Delta d)$ and contains candidate values
\[
d_r
=
d_{\mathrm{start}}+m\Delta d,
\qquad
m\in\mathbb{Z}_{\geq0},
\qquad
d_r\leq d_{\mathrm{ub}},
\]
where $d_{\mathrm{ub}}$ denotes an upper bound rather than necessarily an
included endpoint. Hyperparameters are selected according to validation-set
forecasting performance and fixed before test-set evaluation.

\subsection{Training and Optimization}

For each backbone, BTA follows the corresponding vanilla training protocol.
The temporal kernels remain fixed throughout training, whereas the MoE router
and trainable backbone parameters are optimized jointly. Informer, PatchTST,
and TimeBridge are trained task-specifically. Timer-XL is initialized from its
pretrained checkpoint and adapted to each forecasting dataset through
full-shot fine-tuning.

For each experimental run, the random seed is fixed for reproducibility.
The robustness analysis in
Table~\ref{tab:backbone_bta_mse_sensitivity} additionally repeats each
configuration across three predefined random seeds and reports the mean and
sample standard deviation.

\begin{table}[!htbp]
\centering
\caption{
BTA configuration across the four Transformer backbones.
Router hidden-dimension search ranges are reported as
$(d_{\mathrm{start}},d_{\mathrm{ub}},\Delta d)$, where
$d_{\mathrm{ub}}$ is the search upper bound and $\Delta d$ is the step size.
For PatchTST, patch length $P=16$, stride $S=8$, and end padding give $N=\lfloor(T-P)/S\rfloor+2=12$ input patches for $T=96$, following the patching convention of \citet{nie2023time}.
}
\label{tab:bta_configuration}
\fontfamily{ptm}\fontsize{9}{10}\selectfont
\setlength{\tabcolsep}{3.5pt}
\renewcommand{\arraystretch}{1.10}
\resizebox{\textwidth}{!}{
\begin{tabular}{c|cccc}
\toprule
\textbf{Configuration}
& \textbf{Informer}
& \textbf{PatchTST}
& \textbf{TimeBridge}
& \textbf{Timer-XL} \\
\midrule

Token level
& Point-wise
& Patch-wise
& Patch-wise
& Patch-wise \\

Sequence length
& 96
& 96
& 720
& 672 \\

Patch length
& 1
& 16
& 24 or 48
& 96 \\

Input tokens
& 96
& 12
& 30 or 15
& 7 \\

Lag unit
& Time step
& Patch
& Patch
& Patch \\

Kernel families
& RBF + Periodic
& RBF + Periodic
& RBF + Periodic
& RBF + Periodic \\

RBF kernels
& 5
& 5
& 5
& 5 \\

Dominant periods
& 2
& 2
& 2
& 2 \\

Periodic widths per period
& 3
& 3
& 3
& 3 \\

Periodic kernels
& 6
& 6
& 6
& 6 \\

Total kernels $M$
& 11
& 11
& 11
& 11 \\

Router architecture
& 1-layer MLP
& 1-layer MLP
& 1-layer MLP
& 1-layer MLP \\

Top-$k$
& \multicolumn{4}{c}{$\{1,3,5,7\}$} \\

Router hidden-dimension grid
$(d_{\mathrm{start}},d_{\mathrm{ub}},\Delta d)$
& $(32,1024,64)$
& $(32,512,32)$
& $(32,512,32)$
& $(32,2048,64)$ \\

Optimization
& Task-specific training
& Task-specific training
& Task-specific training
& Pretrained + full-shot fine-tuning \\

\bottomrule
\end{tabular}
}
\end{table}

\section{Main Experiment Appendix}

\paragraph{BTA generalizes across different Transformer backbones.}
Table~\ref{tab:backbone_bta_forecasting_results} reports the per-horizon
results underlying Table~\ref{tab:average_forecasting_results_table1}.
Across three backbones with different tokenization and attention designs, BTA
consistently improves the vanilla attention mechanism while leaving the
remaining backbone architecture unchanged. Averaged over the four forecasting
horizons, BTA reduces MSE by $5.79$--$14.05\%$ for Informer,
$4.34$--$17.51\%$ for PatchTST, and $1.74$--$11.80\%$ for Timer-XL.
Across all $96$ backbone--dataset--horizon--metric settings, BTA improves
$95$, with the only exception being the MAE of Informer on ETTm1 at
$H=192$ ($-0.49\%$). These consistent gains across heterogeneous
architectures demonstrate the broad applicability of BTA as a plug-and-play
attention enhancement.

\begin{table}[!htbp]
\centering
\caption{\textbf{Per-horizon forecasting results of BTA across three
Transformer backbones.}
Results are reported for $H\in\{96,192,336,720\}$.
Within each backbone and row, the lower MSE and MAE between Vanilla and BTA
are shown in \textbf{bold}.
\textit{Avg.} denotes the mean over the four horizons, and Gain (\%) is the
relative error reduction
$100\times(\mathrm{Vanilla}-\mathrm{BTA})/\mathrm{Vanilla}$, computed from
unrounded values for the \textit{Avg.} rows.
Backbone configurations are summarized in Table~\ref{tab:bta_configuration}.}
\label{tab:backbone_bta_forecasting_results}
\fontfamily{ptm}\fontsize{10}{11}\selectfont
\setlength{\tabcolsep}{1.25pt}
\renewcommand{\arraystretch}{1.02}
\resizebox{0.99\linewidth}{!}{%
\begin{tabular}{cc|cc|cc|cc|cc|cc|cc|cc|cc|cc}
\toprule
\multicolumn{2}{c|}{\textbf{Model}}
& \multicolumn{6}{c|}{\textbf{Informer} \citep{zhou2021informer}}
& \multicolumn{6}{c|}{\textbf{PatchTST} \citep{nie2023time}}
& \multicolumn{6}{c}{\textbf{Timer-XL} \citep{liu2025timer}} \\
\midrule
\multicolumn{2}{c|}{\textbf{Self-Attention}}
& \multicolumn{2}{c|}{Vanilla}
& \multicolumn{2}{c|}{BTA}
& \multicolumn{2}{c|}{Gain (\%)}
& \multicolumn{2}{c|}{Vanilla}
& \multicolumn{2}{c|}{BTA}
& \multicolumn{2}{c|}{Gain (\%)}
& \multicolumn{2}{c|}{Vanilla}
& \multicolumn{2}{c|}{BTA}
& \multicolumn{2}{c}{Gain (\%)} \\
\midrule
\multicolumn{2}{c|}{\textbf{Dataset}}
& \textbf{MSE} & \textbf{MAE}
& \textbf{MSE} & \textbf{MAE}
& \textbf{$\Delta$MSE} & \textbf{$\Delta$MAE}
& \textbf{MSE} & \textbf{MAE}
& \textbf{MSE} & \textbf{MAE}
& \textbf{$\Delta$MSE} & \textbf{$\Delta$MAE}
& \textbf{MSE} & \textbf{MAE}
& \textbf{MSE} & \textbf{MAE}
& \textbf{$\Delta$MSE} & \textbf{$\Delta$MAE} \\
\midrule
\multirow{5}{*}{\rotatebox{90}{ETTh1}} & 96 & 0.992 & 0.785 & \textbf{0.878} & \textbf{0.719} & 11.49 & 8.41 & 0.460 & 0.447 & \textbf{0.367} & \textbf{0.394} & 20.22 & 11.86 & 0.364 & 0.397 & \textbf{0.343} & \textbf{0.379} & 5.77 & 4.53 \\
& 192 & 1.008 & 0.801 & \textbf{0.895} & \textbf{0.726} & 11.21 & 9.36 & 0.512 & 0.447 & \textbf{0.423} & \textbf{0.425} & 17.38 & 4.92 & 0.405 & 0.424 & \textbf{0.376} & \textbf{0.402} & 7.16 & 5.19 \\
& 336 & 1.179 & 0.870 & \textbf{1.038} & \textbf{0.793} & 11.96 & 8.85 & 0.546 & 0.496 & \textbf{0.441} & \textbf{0.436} & 19.23 & 12.10 & 0.427 & 0.439 & \textbf{0.391} & \textbf{0.413} & 8.43 & 5.92 \\
& 720 & 1.204 & 0.882 & \textbf{1.112} & \textbf{0.837} & 7.64 & 5.10 & 0.544 & 0.517 & \textbf{0.470} & \textbf{0.461} & 13.60 & 10.83 & 0.439 & 0.459 & \textbf{0.406} & \textbf{0.433} & 7.52 & 5.66 \\
\cmidrule(lr){2-20}
& \textit{Avg.} & 1.096 & 0.835 & \textbf{0.981} & \textbf{0.769} & 10.50 & 7.88 & 0.516 & 0.477 & \textbf{0.425} & \textbf{0.429} & 17.51 & 10.02 & 0.409 & 0.430 & \textbf{0.379} & \textbf{0.407} & 7.28 & 5.35 \\
\midrule
\multirow{5}{*}{\rotatebox{90}{ETTm1}} & 96 & 0.672 & 0.614 & \textbf{0.586} & \textbf{0.557} & 12.80 & 9.28 & 0.352 & 0.374 & \textbf{0.315} & \textbf{0.355} & 10.51 & 5.08 & 0.290 & 0.341 & \textbf{0.284} & \textbf{0.337} & 2.07 & 1.17 \\
& 192 & 0.692 & \textbf{0.610} & \textbf{0.671} & 0.613 & 3.03 & -0.49 & 0.390 & 0.393 & \textbf{0.358} & \textbf{0.381} & 8.21 & 3.05 & 0.337 & 0.369 & \textbf{0.332} & \textbf{0.367} & 1.48 & 0.54 \\
& 336 & 0.942 & 0.749 & \textbf{0.831} & \textbf{0.685} & 11.78 & 8.54 & 0.421 & 0.414 & \textbf{0.373} & \textbf{0.401} & 11.40 & 3.14 & 0.374 & 0.392 & \textbf{0.367} & \textbf{0.390} & 1.87 & 0.51 \\
& 720 & 0.998 & 0.787 & \textbf{0.936} & \textbf{0.754} & 6.21 & 4.19 & 0.462 & 0.449 & \textbf{0.436} & \textbf{0.443} & 5.63 & 1.34 & 0.437 & 0.428 & \textbf{0.430} & \textbf{0.426} & 1.60 & 0.47 \\
\cmidrule(lr){2-20}
& \textit{Avg.} & 0.826 & 0.690 & \textbf{0.756} & \textbf{0.652} & 8.47 & 5.47 & 0.406 & 0.408 & \textbf{0.371} & \textbf{0.395} & 8.80 & 3.07 & 0.360 & 0.383 & \textbf{0.353} & \textbf{0.380} & 1.74 & 0.65 \\
\midrule
\multirow{5}{*}{\rotatebox{90}{Weather}} & 96 & 0.313 & 0.375 & \textbf{0.253} & \textbf{0.335} & 19.17 & 10.67 & 0.186 & 0.227 & \textbf{0.170} & \textbf{0.212} & 8.60 & 6.61 & 0.157 & 0.205 & \textbf{0.153} & \textbf{0.197} & 2.55 & 3.90 \\
& 192 & 0.377 & 0.408 & \textbf{0.299} & \textbf{0.369} & 20.69 & 9.56 & 0.234 & 0.265 & \textbf{0.219} & \textbf{0.256} & 6.41 & 3.40 & 0.206 & 0.250 & \textbf{0.200} & \textbf{0.243} & 2.91 & 2.80 \\
& 336 & 0.448 & 0.449 & \textbf{0.389} & \textbf{0.413} & 13.17 & 8.02 & 0.284 & 0.301 & \textbf{0.274} & \textbf{0.295} & 3.52 & 1.99 & 0.259 & 0.291 & \textbf{0.255} & \textbf{0.285} & 1.54 & 2.06 \\
& 720 & 0.712 & 0.604 & \textbf{0.649} & \textbf{0.577} & 8.85 & 4.47 & 0.356 & 0.349 & \textbf{0.351} & \textbf{0.347} & 1.40 & 0.57 & 0.337 & 0.344 & \textbf{0.332} & \textbf{0.338} & 1.48 & 1.74 \\
\cmidrule(lr){2-20}
& \textit{Avg.} & 0.463 & 0.459 & \textbf{0.398} & \textbf{0.424} & 14.05 & 7.73 & 0.265 & 0.286 & \textbf{0.254} & \textbf{0.278} & 4.34 & 2.80 & 0.240 & 0.273 & \textbf{0.235} & \textbf{0.266} & 1.98 & 2.48 \\
\midrule
\multirow{5}{*}{\rotatebox{90}{Solar}} & 96 & 0.200 & 0.253 & \textbf{0.191} & \textbf{0.249} & 4.50 & 1.58 & 0.234 & 0.286 & \textbf{0.208} & \textbf{0.249} & 11.11 & 12.94 & 0.174 & 0.215 & \textbf{0.166} & \textbf{0.214} & 4.60 & 0.47 \\
& 192 & 0.240 & 0.273 & \textbf{0.224} & \textbf{0.260} & 6.67 & 4.76 & 0.267 & 0.310 & \textbf{0.241} & \textbf{0.272} & 9.74 & 12.26 & 0.199 & 0.234 & \textbf{0.187} & \textbf{0.232} & 6.03 & 0.85 \\
& 336 & 0.244 & 0.273 & \textbf{0.234} & \textbf{0.262} & 4.10 & 4.03 & 0.290 & 0.315 & \textbf{0.258} & \textbf{0.284} & 11.03 & 9.84 & 0.225 & 0.255 & \textbf{0.200} & \textbf{0.248} & 11.11 & 2.75 \\
& 720 & 0.248 & 0.274 & \textbf{0.229} & \textbf{0.271} & 7.66 & 1.09 & 0.289 & 0.317 & \textbf{0.257} & \textbf{0.282} & 11.07 & 11.04 & 0.275 & 0.285 & \textbf{0.217} & \textbf{0.273} & 21.09 & 4.21 \\
\cmidrule(lr){2-20}
& \textit{Avg.} & 0.233 & 0.268 & \textbf{0.220} & \textbf{0.261} & 5.79 & 2.89 & 0.270 & 0.307 & \textbf{0.241} & \textbf{0.272} & 10.74 & 11.48 & 0.218 & 0.247 & \textbf{0.193} & \textbf{0.242} & 11.80 & 2.22 \\
\bottomrule
\end{tabular}%
}
\end{table}

To examine the robustness of BTA, we train the three backbones integrated with BTA across the four datasets and four different forecasting horizons. Table~\ref{tab:backbone_bta_mse_sensitivity} shows that BTA generally achieves lower mean MSE than vanilla self-attention.
The multi-seed results also show limited variation across random
initializations, supporting the robustness of BTA's forecasting
performance. 

\begin{table}[!htbp]
\centering
\setlength{\tabcolsep}{3pt}
\renewcommand{\arraystretch}{1.02}
\caption{\textbf{Forecasting MSE across three Transformer backbones and
three random seeds.}
Results are reported as mean $\pm$ sample standard deviation for
$H\in\{96,192,336,720\}$. Within each backbone, dataset, and horizon, the method
with the lower mean MSE has both its mean and standard deviation
shown in bold.}
\label{tab:backbone_bta_mse_sensitivity}
\fontfamily{ptm}\fontsize{10}{11}\selectfont
\resizebox{0.99\textwidth}{!}{%
\begin{tabular}{cc|cc|cc|cc}
\toprule
\multicolumn{2}{c|}{\textbf{Model}}
& \multicolumn{2}{c|}{\textbf{Informer} \citep{zhou2021informer}}
& \multicolumn{2}{c|}{\textbf{PatchTST} \citep{nie2023time}}
& \multicolumn{2}{c}{\textbf{Timer-XL} \citep{liu2025timer}} \\
\midrule
\textbf{Dataset} & \textbf{Horizon}
& \textbf{Vanilla MSE} & \textbf{BTA MSE}
& \textbf{Vanilla MSE} & \textbf{BTA MSE}
& \textbf{Vanilla MSE} & \textbf{BTA MSE} \\
\midrule

\multirow{4}{*}{\rotatebox{90}{ETTh1}}
& 96
& $0.994\pm0.003$ & $\mathbf{0.876}\pm\mathbf{0.004}$
& $0.462\pm0.005$ & $\mathbf{0.366}\pm\mathbf{0.004}$
& $0.366\pm0.005$ & $\mathbf{0.345}\pm\mathbf{0.003}$ \\

& 192
& $1.012\pm0.012$ & $\mathbf{0.895}\pm\mathbf{0.003}$
& $0.511\pm0.004$ & $\mathbf{0.422}\pm\mathbf{0.006}$
& $0.404\pm0.011$ & $\mathbf{0.379}\pm\mathbf{0.006}$ \\

& 336
& $1.174\pm0.009$ & $\mathbf{1.036}\pm\mathbf{0.020}$
& $0.548\pm0.016$ & $\mathbf{0.447}\pm\mathbf{0.012}$
& $0.422\pm0.018$ & $\mathbf{0.397}\pm\mathbf{0.010}$ \\

& 720
& $1.204\pm0.001$ & $\mathbf{1.111}\pm\mathbf{0.006}$
& $0.543\pm0.005$ & $\mathbf{0.470}\pm\mathbf{0.002}$
& $0.437\pm0.006$ & $\mathbf{0.406}\pm\mathbf{0.003}$ \\

\midrule

\multirow{4}{*}{\rotatebox{90}{ETTm1}}
& 96
& $0.672\pm0.001$ & $\mathbf{0.586}\pm\mathbf{0.001}$
& $0.352\pm0.001$ & $\mathbf{0.315}\pm\mathbf{0.001}$
& $0.290\pm0.001$ & $\mathbf{0.283}\pm\mathbf{0.001}$ \\

& 192
& $0.693\pm0.003$ & $\mathbf{0.669}\pm\mathbf{0.005}$
& $0.391\pm0.005$ & $\mathbf{0.359}\pm\mathbf{0.002}$
& $0.337\pm0.004$ & $\mathbf{0.332}\pm\mathbf{0.002}$ \\

& 336
& $0.940\pm0.008$ & $\mathbf{0.831}\pm\mathbf{0.001}$
& $0.421\pm0.002$ & $\mathbf{0.373}\pm\mathbf{0.000}$
& $0.372\pm0.004$ & $\mathbf{0.371}\pm\mathbf{0.003}$ \\

& 720
& $0.998\pm0.007$ & $\mathbf{0.937}\pm\mathbf{0.007}$
& $0.460\pm0.013$ & $\mathbf{0.436}\pm\mathbf{0.001}$
& $0.440\pm0.010$ & $\mathbf{0.435}\pm\mathbf{0.004}$ \\

\midrule

\multirow{4}{*}{\rotatebox{90}{Weather}}
& 96
& $0.315\pm0.002$ & $\mathbf{0.254}\pm\mathbf{0.002}$
& $0.184\pm0.002$ & $\mathbf{0.172}\pm\mathbf{0.002}$
& $0.157\pm0.002$ & $\mathbf{0.150}\pm\mathbf{0.003}$ \\

& 192
& $0.378\pm0.001$ & $\mathbf{0.294}\pm\mathbf{0.005}$
& $0.233\pm0.001$ & $\mathbf{0.222}\pm\mathbf{0.004}$
& $0.206\pm0.006$ & $\mathbf{0.197}\pm\mathbf{0.003}$ \\

& 336
& $0.449\pm0.001$ & $\mathbf{0.389}\pm\mathbf{0.001}$
& $0.282\pm0.003$ & $\mathbf{0.275}\pm\mathbf{0.003}$
& $0.261\pm0.002$ & $\mathbf{0.251}\pm\mathbf{0.004}$ \\

& 720
& $0.712\pm0.004$ & $\mathbf{0.651}\pm\mathbf{0.002}$
& $0.355\pm0.002$ & $\mathbf{0.352}\pm\mathbf{0.001}$
& $0.336\pm0.004$ & $\mathbf{0.327}\pm\mathbf{0.006}$ \\

\midrule

\multirow{4}{*}{\rotatebox{90}{Solar}}
& 96
& $0.200\pm0.000$ & $\mathbf{0.191}\pm\mathbf{0.000}$
& $0.234\pm0.000$ & $\mathbf{0.208}\pm\mathbf{0.000}$
& $0.174\pm0.000$ & $\mathbf{0.166}\pm\mathbf{0.000}$ \\

& 192
& $0.240\pm0.000$ & $\mathbf{0.225}\pm\mathbf{0.001}$
& $0.268\pm0.001$ & $\mathbf{0.242}\pm\mathbf{0.001}$
& $0.200\pm0.001$ & $\mathbf{0.186}\pm\mathbf{0.001}$ \\

& 336
& $0.244\pm0.001$ & $\mathbf{0.234}\pm\mathbf{0.000}$
& $0.290\pm0.001$ & $\mathbf{0.258}\pm\mathbf{0.000}$
& $0.225\pm0.000$ & $\mathbf{0.201}\pm\mathbf{0.001}$ \\

& 720
& $0.248\pm0.001$ & $\mathbf{0.228}\pm\mathbf{0.001}$
& $0.289\pm0.000$ & $\mathbf{0.256}\pm\mathbf{0.001}$
& $0.275\pm0.001$ & $\mathbf{0.218}\pm\mathbf{0.001}$ \\

\bottomrule
\end{tabular}%
}
\vspace{1pt}
\end{table}

\paragraph{BTA Improves TimeBridge for Non-stationary LTSF.}
Table~\ref{tab:long_term_forecasting} compares TimeBridge+BTA (\emph{Ours})
with the original TimeBridge and other representative LTSF methods under the
same input length $T=720$. Integrating BTA consistently improves TimeBridge
across all eight datasets, reducing the overall average MSE and MAE by
$2.39\%$ and $1.89\%$, respectively. Moreover, TimeBridge+BTA achieves the
lowest MSE in 29 of 32 horizon-level settings and the lowest MAE in 23,
substantially exceeding all competing methods. These results demonstrate that
BTA provides a consistent additional benefit to an already competitive
non-stationary forecasting backbone.

\begin{table}[!htbp]
\centering
\caption{\textbf{Per-horizon comparison for non-stationary LTSF.}
\emph{Ours} denotes TimeBridge+BTA, enabling direct comparison with the
original TimeBridge to isolate the contribution of BTA. All methods use input
length $T=720$ and prediction horizons $H\in\{96,192,336,720\}$.
\textit{Avg.} denotes the mean over the four horizons. The best and
second-best results for each dataset, horizon, and metric are shown in \textbf{bold} and
\underline{underlined}, respectively. \textit{1st Count} reports the number
of horizon-level settings in which each method achieves the best result,
excluding \textit{Avg.} rows and counting ties for all tied methods.}
\label{tab:long_term_forecasting}
\renewcommand{\small}{\fontsize{9}{10}\selectfont}
\fontfamily{ptm}\fontsize{10}{11}\selectfont
\setlength{\tabcolsep}{1.25pt}
\renewcommand{\arraystretch}{1.02}
\resizebox{\linewidth}{!}{%
\begin{tabular}{lc|cc|cc|cc|cc|cc|cc|cc|cc|cc|cc|cc}
\toprule

\multicolumn{2}{c}{\multirow{2}{*}{\textbf{Model}}}
& \multicolumn{2}{c}{\multirow{2}{*}{\qquad\textbf{Ours\qquad}}}
& \multicolumn{2}{c}{\textbf{PULSE}}
& \multicolumn{2}{c}{\textbf{DTAF}}
& \multicolumn{2}{c}{\textbf{DUET}}
& \multicolumn{2}{c}{\textbf{TimeBridge}}
& \multicolumn{2}{c}{\textbf{iTransformer}}
& \multicolumn{2}{c}{\textbf{TimeMixer}}
& \multicolumn{2}{c}{\textbf{PatchTST}}
& \multicolumn{2}{c}{\textbf{TimesNet}}
& \multicolumn{2}{c}{\textbf{DLinear}}
& \multicolumn{2}{c}{\textbf{Stationary}} \\

\multicolumn{2}{c}{}
& \multicolumn{2}{c}{}
& \multicolumn{2}{c}{\small \citep{liupulse}}
& \multicolumn{2}{c}{\small \citep{lu2026towards}}
& \multicolumn{2}{c}{\small \citep{qiu2025duet}}
& \multicolumn{2}{c}{\small \citep{liu2025timebridge}}
& \multicolumn{2}{c}{\small \citep{liu2024itransformer}}
& \multicolumn{2}{c}{\small \citep{wang2024timemixer}}
& \multicolumn{2}{c}{\small \citep{nie2023time}}
& \multicolumn{2}{c}{\small \citep{wu2023timesnet}}
& \multicolumn{2}{c}{\small \citep{zeng2023transformers}}
& \multicolumn{2}{c}{\small \citep{liu2022non}} \\

\midrule

\multicolumn{2}{c|}{\textbf{Metric}}
& \textbf{MSE} & \textbf{MAE}
& \textbf{MSE} & \textbf{MAE}
& \textbf{MSE} & \textbf{MAE}
& \textbf{MSE} & \textbf{MAE}
& \textbf{MSE} & \textbf{MAE}
& \textbf{MSE} & \textbf{MAE}
& \textbf{MSE} & \textbf{MAE}
& \textbf{MSE} & \textbf{MAE}
& \textbf{MSE} & \textbf{MAE}
& \textbf{MSE} & \textbf{MAE}
& \textbf{MSE} & \textbf{MAE} \\

\midrule
\multirow{5}{*}{\rotatebox{90}{ETTm1}} & 96 & \textbf{0.275} & \underline{0.330} & 0.303 & 0.347 & \underline{0.277} & \textbf{0.329} & 0.279 & 0.336 & 0.284 & 0.337 & 0.300 & 0.353 & 0.293 & 0.345 & 0.293 & 0.346 & 0.338 & 0.375 & 0.299 & 0.343 & 0.393 & 0.404 \\
& 192 & \textbf{0.315} & \textbf{0.359} & 0.349 & 0.379 & 0.329 & \underline{0.361} & \underline{0.318} & \textbf{0.359} & 0.323 & 0.365 & 0.345 & 0.382 & 0.335 & 0.372 & 0.333 & 0.370 & 0.371 & 0.387 & 0.335 & 0.365 & 0.481 & 0.447 \\
& 336 & \underline{0.351} & 0.385 & 0.375 & 0.398 & \textbf{0.345} & \textbf{0.374} & \textbf{0.345} & \underline{0.378} & 0.370 & 0.394 & 0.374 & 0.398 & 0.368 & 0.386 & 0.369 & 0.390 & 0.410 & 0.411 & 0.369 & 0.386 & 0.555 & 0.489 \\
& 720 & 0.408 & 0.411 & 0.424 & 0.426 & \textbf{0.399} & \textbf{0.407} & \underline{0.400} & \underline{0.408} & 0.415 & 0.418 & 0.429 & 0.430 & 0.426 & 0.417 & 0.416 & 0.420 & 0.478 & 0.450 & 0.425 & 0.421 & 0.677 & 0.548 \\
\cmidrule(lr){3-24}
& \textit{Avg.} & \underline{0.337} & 0.371 & 0.363 & 0.388 & 0.338 & \textbf{0.368} & \textbf{0.336} & \underline{0.370} & 0.348 & 0.379 & 0.362 & 0.391 & 0.356 & 0.380 & 0.353 & 0.381 & 0.399 & 0.406 & 0.357 & 0.379 & 0.527 & 0.472 \\
\midrule
\multirow{5}{*}{\rotatebox{90}{ETTm2}} & 96 & \textbf{0.155} & \textbf{0.242} & 0.163 & 0.248 & 0.160 & 0.245 & 0.164 & 0.253 & \underline{0.157} & \underline{0.243} & 0.175 & 0.266 & 0.165 & 0.256 & 0.166 & 0.256 & 0.187 & 0.267 & 0.167 & 0.260 & 0.266 & 0.324 \\
& 192 & \textbf{0.215} & \textbf{0.283} & 0.227 & 0.291 & \textbf{0.215} & \underline{0.285} & \underline{0.220} & 0.293 & 0.224 & 0.286 & 0.242 & 0.312 & 0.225 & 0.298 & 0.223 & 0.296 & 0.249 & 0.309 & 0.224 & 0.303 & 0.434 & 0.405 \\
& 336 & \textbf{0.266} & \textbf{0.320} & 0.283 & 0.330 & \underline{0.269} & \underline{0.321} & 0.277 & 0.331 & \underline{0.269} & \underline{0.321} & 0.282 & 0.340 & 0.277 & 0.332 & 0.274 & 0.329 & 0.321 & 0.351 & 0.281 & 0.342 & 0.808 & 0.563 \\
& 720 & \textbf{0.346} & \textbf{0.375} & 0.374 & 0.387 & 0.351 & \underline{0.378} & 0.358 & 0.386 & \underline{0.348} & \underline{0.378} & 0.378 & 0.398 & 0.360 & 0.387 & 0.362 & 0.385 & 0.497 & 0.403 & 0.397 & 0.421 & 0.758 & 0.555 \\
\cmidrule(lr){3-24}
& \textit{Avg.} & \textbf{0.246} & \textbf{0.305} & 0.262 & 0.314 & \underline{0.249} & \underline{0.307} & 0.255 & 0.316 & 0.250 & \underline{0.307} & 0.269 & 0.329 & 0.257 & 0.318 & 0.256 & 0.317 & 0.314 & 0.333 & 0.267 & 0.332 & 0.567 & 0.462 \\
\midrule
\multirow{5}{*}{\rotatebox{90}{ETTh1}} & 96 & \textbf{0.347} & \textbf{0.385} & 0.362 & \underline{0.390} & 0.369 & 0.396 & 0.358 & 0.391 & \underline{0.354} & \underline{0.390} & 0.386 & 0.405 & 0.372 & 0.401 & 0.370 & 0.400 & 0.384 & 0.402 & 0.375 & 0.399 & 0.547 & 0.501 \\
& 192 & \textbf{0.386} & \textbf{0.414} & 0.408 & 0.418 & 0.404 & 0.419 & 0.399 & \underline{0.416} & \underline{0.389} & \textbf{0.414} & 0.424 & 0.440 & 0.413 & 0.430 & 0.413 & 0.429 & 0.557 & 0.436 & 0.405 & \underline{0.416} & 0.603 & 0.529 \\
& 336 & \textbf{0.404} & \textbf{0.426} & 0.443 & 0.433 & 0.449 & 0.453 & 0.418 & 0.431 & \underline{0.408} & \underline{0.430} & 0.449 & 0.460 & 0.438 & 0.450 & 0.422 & 0.440 & 0.491 & 0.469 & 0.439 & 0.443 & 0.711 & 0.585 \\
& 720 & \underline{0.437} & \underline{0.465} & \textbf{0.433} & \textbf{0.448} & 0.441 & 0.471 & 0.441 & 0.471 & 0.443 & 0.468 & 0.495 & 0.487 & 0.486 & 0.484 & 0.447 & 0.468 & 0.521 & 0.500 & 0.472 & 0.490 & 0.714 & 0.601 \\
\cmidrule(lr){3-24}
& \textit{Avg.} & \textbf{0.394} & \underline{0.423} & 0.412 & \textbf{0.422} & 0.416 & 0.435 & 0.404 & 0.427 & \underline{0.399} & 0.426 & 0.439 & 0.448 & 0.427 & 0.441 & 0.413 & 0.434 & 0.488 & 0.452 & 0.423 & 0.437 & 0.644 & 0.554 \\
\midrule
\multirow{5}{*}{\rotatebox{90}{ETTh2}} & 96 & \textbf{0.267} & \textbf{0.329} & 0.277 & \underline{0.330} & \underline{0.271} & \textbf{0.329} & 0.273 & 0.339 & \underline{0.271} & 0.331 & 0.297 & 0.348 & 0.281 & 0.351 & 0.274 & 0.337 & 0.340 & 0.374 & 0.289 & 0.353 & 0.392 & 0.420 \\
& 192 & \textbf{0.332} & \textbf{0.371} & 0.356 & 0.381 & \underline{0.333} & \textbf{0.371} & \underline{0.333} & 0.376 & 0.335 & \underline{0.372} & 0.371 & 0.403 & 0.349 & 0.387 & 0.348 & 0.384 & 0.402 & 0.414 & 0.383 & 0.418 & 0.519 & 0.480 \\
& 336 & \textbf{0.360} & \textbf{0.398} & 0.399 & 0.417 & 0.368 & \underline{0.402} & 0.372 & 0.411 & 0.371 & \underline{0.402} & 0.404 & 0.428 & \underline{0.366} & 0.413 & 0.377 & 0.416 & 0.452 & 0.452 & 0.448 & 0.465 & 0.572 & 0.511 \\
& 720 & \textbf{0.392} & \textbf{0.429} & 0.413 & 0.436 & 0.480 & 0.482 & \underline{0.398} & 0.435 & 0.399 & \underline{0.432} & 0.424 & 0.444 & 0.401 & 0.436 & 0.406 & 0.441 & 0.462 & 0.468 & 0.605 & 0.551 & 0.600 & 0.536 \\
\cmidrule(lr){3-24}
& \textit{Avg.} & \textbf{0.338} & \textbf{0.382} & 0.361 & 0.391 & 0.363 & 0.396 & \underline{0.344} & 0.390 & \underline{0.344} & \underline{0.384} & 0.374 & 0.406 & 0.349 & 0.397 & 0.351 & 0.395 & 0.414 & 0.427 & 0.431 & 0.447 & 0.521 & 0.487 \\
\midrule
\multirow{5}{*}{\rotatebox{90}{Weather}} & 96 & \textbf{0.142} & \textbf{0.182} & 0.150 & 0.196 & 0.145 & 0.185 & 0.149 & 0.195 & \underline{0.144} & \underline{0.184} & 0.159 & 0.208 & 0.147 & 0.198 & 0.149 & 0.198 & 0.172 & 0.220 & 0.152 & 0.237 & 0.183 & 0.230 \\
& 192 & \textbf{0.185} & \textbf{0.225} & 0.200 & 0.245 & 0.190 & \underline{0.228} & 0.194 & 0.239 & \underline{0.187} & \textbf{0.225} & 0.200 & 0.248 & 0.192 & 0.243 & 0.194 & 0.241 & 0.219 & 0.261 & 0.220 & 0.282 & 0.253 & 0.291 \\
& 336 & \textbf{0.237} & \textbf{0.267} & 0.260 & 0.291 & 0.241 & 0.269 & 0.254 & 0.285 & \underline{0.239} & \underline{0.268} & 0.253 & 0.289 & 0.247 & 0.284 & 0.245 & 0.282 & 0.280 & 0.306 & 0.265 & 0.319 & 0.326 & 0.341 \\
& 720 & \textbf{0.309} & \textbf{0.319} & 0.347 & 0.353 & 0.313 & \underline{0.321} & 0.324 & 0.336 & \underline{0.311} & \underline{0.321} & 0.321 & 0.338 & 0.318 & 0.330 & 0.314 & 0.334 & 0.365 & 0.359 & 0.323 & 0.362 & 0.417 & 0.397 \\
\cmidrule(lr){3-24}
& \textit{Avg.} & \textbf{0.218} & \textbf{0.248} & 0.239 & 0.271 & 0.222 & 0.251 & 0.230 & 0.264 & \underline{0.220} & \underline{0.250} & 0.233 & 0.271 & 0.226 & 0.264 & 0.226 & 0.264 & 0.259 & 0.287 & 0.240 & 0.300 & 0.295 & 0.315 \\
\midrule
\multirow{5}{*}{\rotatebox{90}{ELC}} & 96 & \textbf{0.124} & \underline{0.217} & 0.134 & 0.228 & \underline{0.129} & 0.221 & 0.131 & 0.224 & \textbf{0.124} & \textbf{0.216} & 0.138 & 0.237 & 0.153 & 0.256 & \underline{0.129} & 0.222 & 0.168 & 0.272 & 0.140 & 0.237 & 0.172 & 0.276 \\
& 192 & \textbf{0.144} & \underline{0.237} & 0.150 & 0.243 & 0.146 & \textbf{0.235} & 0.156 & 0.247 & \underline{0.145} & 0.241 & 0.157 & 0.256 & 0.168 & 0.269 & 0.147 & 0.240 & 0.184 & 0.289 & 0.152 & 0.249 & 0.185 & 0.287 \\
& 336 & \textbf{0.158} & \textbf{0.252} & 0.166 & 0.261 & 0.165 & \underline{0.254} & 0.164 & 0.257 & 0.166 & 0.263 & 0.167 & 0.264 & 0.189 & 0.291 & \underline{0.163} & 0.259 & 0.198 & 0.300 & 0.169 & 0.267 & 0.196 & 0.300 \\
& 720 & \textbf{0.180} & \textbf{0.272} & 0.188 & 0.289 & 0.200 & 0.283 & 0.195 & 0.285 & \underline{0.182} & \underline{0.277} & 0.194 & 0.286 & 0.228 & 0.320 & 0.197 & 0.290 & 0.220 & 0.320 & 0.203 & 0.301 & 0.221 & 0.319 \\
\cmidrule(lr){3-24}
& \textit{Avg.} & \textbf{0.152} & \textbf{0.245} & 0.160 & 0.255 & 0.160 & \underline{0.248} & 0.162 & 0.253 & \underline{0.154} & 0.249 & 0.164 & 0.261 & 0.185 & 0.284 & 0.159 & 0.253 & 0.193 & 0.295 & 0.166 & 0.264 & 0.194 & 0.296 \\
\midrule
\multirow{5}{*}{\rotatebox{90}{Traffic}} & 96 & \textbf{0.318} & \textbf{0.225} & 0.414 & 0.268 & 0.370 & \underline{0.232} & 0.351 & 0.237 & \underline{0.340} & 0.240 & 0.363 & 0.265 & 0.369 & 0.257 & 0.360 & 0.249 & 0.593 & 0.321 & 0.410 & 0.282 & 0.612 & 0.338 \\
& 192 & \textbf{0.336} & \textbf{0.237} & 0.444 & 0.278 & 0.392 & \underline{0.243} & 0.368 & 0.246 & \underline{0.343} & 0.250 & 0.385 & 0.273 & 0.400 & 0.272 & 0.379 & 0.256 & 0.617 & 0.336 & 0.423 & 0.287 & 0.613 & 0.340 \\
& 336 & \textbf{0.348} & \textbf{0.245} & 0.458 & 0.287 & 0.408 & \underline{0.252} & 0.393 & 0.260 & \underline{0.363} & 0.257 & 0.396 & 0.277 & 0.407 & 0.272 & 0.392 & 0.264 & 0.629 & 0.336 & 0.436 & 0.296 & 0.618 & 0.328 \\
& 720 & \textbf{0.382} & \textbf{0.260} & 0.525 & 0.312 & 0.441 & 0.269 & 0.430 & 0.272 & \underline{0.395} & \underline{0.263} & 0.445 & 0.312 & 0.461 & 0.316 & 0.432 & 0.286 & 0.640 & 0.350 & 0.466 & 0.315 & 0.653 & 0.355 \\
\cmidrule(lr){3-24}
& \textit{Avg.} & \textbf{0.346} & \textbf{0.242} & 0.460 & 0.286 & 0.403 & \underline{0.249} & 0.386 & 0.254 & \underline{0.360} & 0.253 & 0.397 & 0.282 & 0.409 & 0.279 & 0.391 & 0.264 & 0.620 & 0.336 & 0.434 & 0.295 & 0.624 & 0.340 \\
\midrule
\multirow{5}{*}{\rotatebox{90}{Solar}} & 96 & \textbf{0.150} & \underline{0.198} & 0.176 & 0.229 & 0.182 & 0.218 & 0.175 & \textbf{0.197} & \underline{0.161} & 0.224 & 0.188 & 0.242 & 0.179 & 0.232 & 0.178 & 0.229 & 0.219 & 0.314 & 0.216 & 0.287 & 0.381 & 0.398 \\
& 192 & \textbf{0.171} & \underline{0.211} & 0.192 & 0.245 & 0.182 & 0.235 & 0.193 & \textbf{0.210} & \underline{0.177} & 0.237 & 0.193 & 0.258 & 0.201 & 0.259 & 0.189 & 0.246 & 0.231 & 0.322 & 0.244 & 0.305 & 0.395 & 0.386 \\
& 336 & \textbf{0.184} & \textbf{0.226} & 0.200 & 0.247 & \underline{0.189} & \underline{0.229} & 0.197 & 0.230 & 0.196 & 0.232 & 0.195 & 0.259 & 0.190 & 0.256 & 0.198 & 0.249 & 0.246 & 0.337 & 0.263 & 0.319 & 0.410 & 0.394 \\
& 720 & \textbf{0.197} & \underline{0.236} & 0.208 & 0.249 & 0.228 & 0.256 & 0.210 & \textbf{0.215} & 0.210 & 0.239 & 0.223 & 0.281 & \underline{0.203} & 0.261 & 0.209 & 0.256 & 0.280 & 0.363 & 0.264 & 0.324 & 0.377 & 0.376 \\
\cmidrule(lr){3-24}
& \textit{Avg.} & \textbf{0.176} & \underline{0.218} & 0.194 & 0.243 & 0.195 & 0.235 & 0.194 & \textbf{0.213} & \underline{0.186} & 0.233 & 0.200 & 0.260 & 0.193 & 0.252 & 0.194 & 0.245 & 0.244 & 0.334 & 0.247 & 0.309 & 0.391 & 0.389 \\
\midrule
\multicolumn{2}{c|}{1st Count} & 29 & 23 & 1 & 1 & 3 & 6 & 1 & 4 & 1 & 3 & 0 & 0 & 0 & 0 & 0 & 0 & 0 & 0 & 0 & 0 & 0 & 0 \\
\bottomrule
\end{tabular}%
}
\end{table}

\section{More case studies}
\label{app:solar_qualitative_analysis}

Figures~\ref{fig:solar_input_conditioned_kernels},
\ref{fig:weather_input_conditioned_kernels}, and
\ref{Img_Timer_XL_case_study} further illustrate how BTA adapts positional
biases to heterogeneous temporal dynamics. On Solar, both windows exhibit
strong periodic structure, but BTA selects different temporal scales: irregular
peak amplitudes lead to a combination of broader periodic and long-range local
kernels, whereas more regularly recurring peaks favor narrower periodic
kernels. On Weather, BTA adapts primarily across kernel families, emphasizing
a long-range local kernel for a window with a recent level shift and a
$3$-day periodic kernel for a window with recurring temporal structure.
A similar behavior is observed with Timer-XL: an abrupt recent change favors
short-range local kernels, while a more stable window assigns greater weight
to broader local and periodic dependencies. These input-specific kernel
mixtures induce distinct temporal-evidence profiles and hence adaptive
positional biases, demonstrating how BTA enables Transformer forecasters to
adjust their temporal dependencies to non-stationary input dynamics.

\begin{figure}[!htbp]
  \centering
  \includegraphics[width=0.99\textwidth]{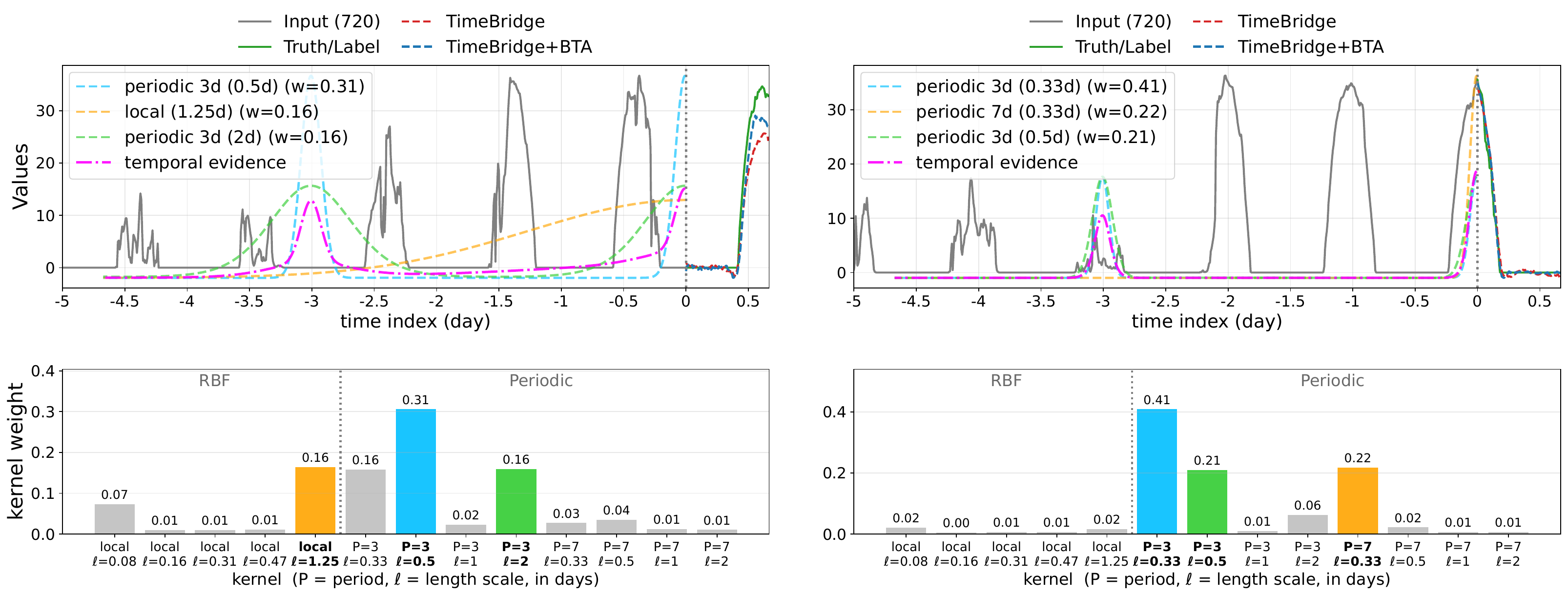}
  \caption{\textbf{Input-adaptive positional bias on Solar.}
Two Solar windows with a 720-step input and 96-step forecasting horizon;
panels follow Figure~\ref{fig:ettm1_input_conditioned_kernels}.
Both windows are dominated by periodic kernels, but BTA adapts their temporal
scales to the input dynamics. The left window, with more irregular peak
amplitudes, combines a $3$-day periodic kernel with $\ell=0.5$ days
($w=0.31$), a long-range local kernel with $\ell=1.25$ days ($w=0.16$),
and a broader $3$-day periodic kernel with $\ell=2$ days ($w=0.16$).
In contrast, the right window concentrates its weight on narrower periodic
kernels, led by $3$-day kernels with $\ell=0.33$ and $0.5$ days
($w=0.41$ and $0.21$) and a $7$-day kernel with $\ell=0.33$ days
($w=0.22$). The resulting mixtures induce distinct temporal-evidence
profiles and positional biases. BTA reduces MSE by $58\%$
and $34\%$, respectively ($22.48\rightarrow9.41$ and
$1.46\rightarrow0.96$).}
  \label{fig:solar_input_conditioned_kernels}
\end{figure}

\begin{figure}[!htbp]
  \centering
  \includegraphics[width=0.99\textwidth]{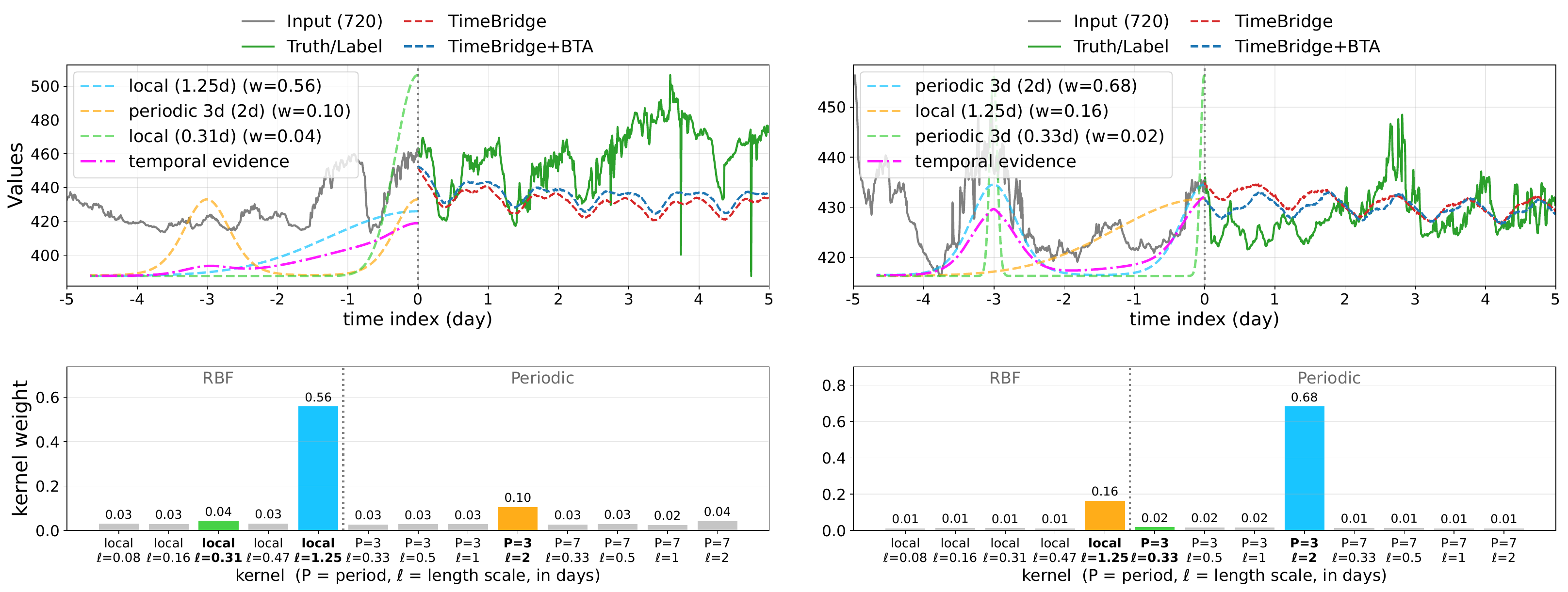}
  \caption{\textbf{Input-adaptive positional bias on Weather.}
Two Weather windows with a 720-step input and 720-step forecasting horizon;
panels follow Figure~\ref{fig:ettm1_input_conditioned_kernels}.
The left window exhibits a recent level shift without a clear recurring
pattern, for which BTA assigns $0.69$ of the total weight to local RBF
kernels, dominated by the long-range $\ell=1.25$-day kernel ($w=0.56$).
In contrast, the right window exhibits stronger recurring structure and is
dominated by a $3$-day periodic kernel with $\ell=2$ days ($w=0.68$),
while the long-range local kernel receives $w=0.16$.
The contrasting kernel mixtures produce different temporal-evidence profiles
and positional biases, illustrating how BTA adapts between local and periodic
dependencies across input windows.}
  \label{fig:weather_input_conditioned_kernels}
\end{figure}

\begin{figure}[!htbp]
  \centering
  \includegraphics[width=0.99\textwidth]{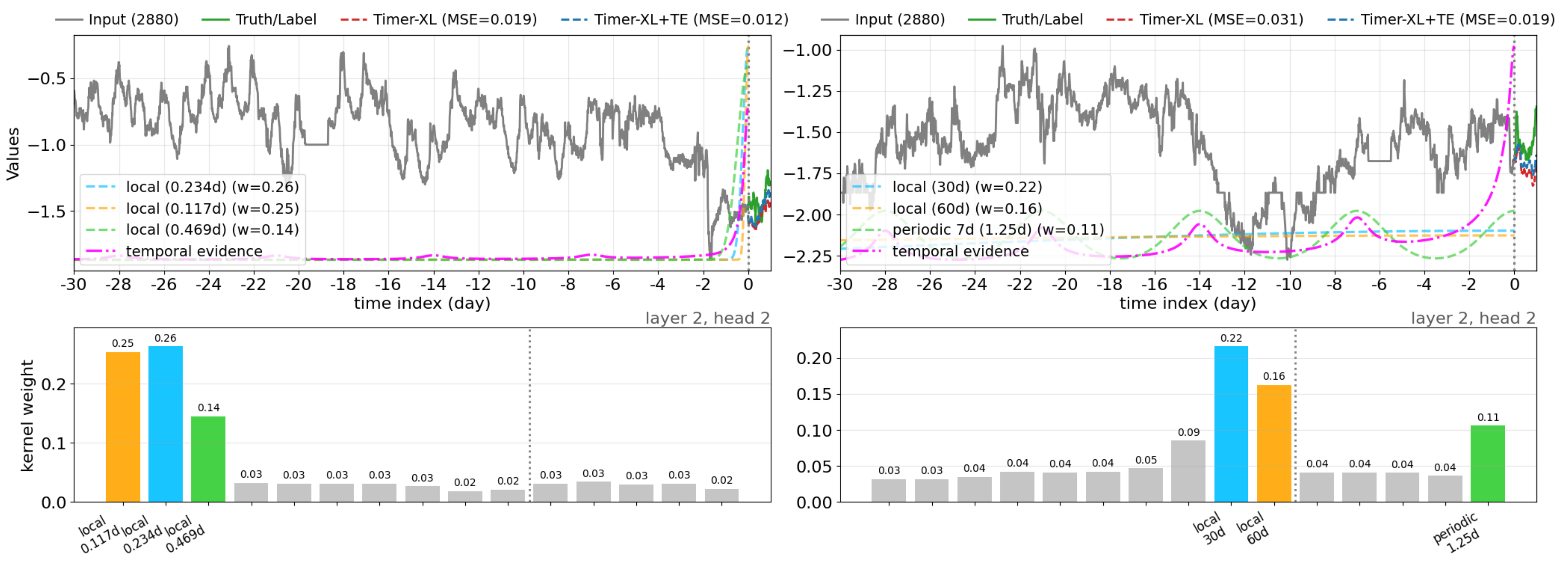}
  \caption{\textbf{Input-adaptive positional bias in Timer-XL on ETTm1.}
Two 30-day input windows are used to forecast the following day.
The top panels compare Timer-XL and Timer-XL+BTA predictions and visualize
the dominant temporal kernels selected by BTA; the bottom panels show the
corresponding router weights for layer~2, head~2.
For the left window, which exhibits an abrupt recent change, BTA emphasizes
short-range local kernels, with its three largest weights assigned to local
scales of $0.234$, $0.117$, and $0.469$ days.
For the more stable right window, BTA shifts toward substantially broader
local dependencies together with a periodic component.
These input-dependent mixtures induce different positional biases and reduce
the forecasting MSE from $0.019$ to $0.012$ and from $0.031$ to $0.019$,
respectively, demonstrating that BTA can adapt the temporal dependencies of a
pretrained Transformer to non-stationary input dynamics.}
  \label{Img_Timer_XL_case_study}
\end{figure}

\section{More representation analysis on Solar}
\label{app:solar_kernel_weights}

Figure~\ref{fig:solar_kernel_weights_by_prediction_length} further confirms that BTA adapts its
temporal inductive bias to the forecasting horizon on Solar. The learned
kernel weights vary systematically across prediction lengths: shorter horizons
place greater emphasis on relatively local dependencies, whereas longer
horizons increasingly favor broader local structure, with the
$\ell=1.25$-day RBF kernel becoming dominant at $H=720$. The allocation
among periodic kernels also changes with the horizon, with longer forecasts
placing greater weight on broader $3$-day periodic patterns. These results
further demonstrate that BTA does not rely on a fixed positional structure,
but instead adjusts its temporal dependencies to the forecasting horizon,
providing an interpretable view of the adaptive positional bias for
non-stationary LTSF.

\begin{figure}[!htbp]
  \centering
  \includegraphics[width=0.90\textwidth]{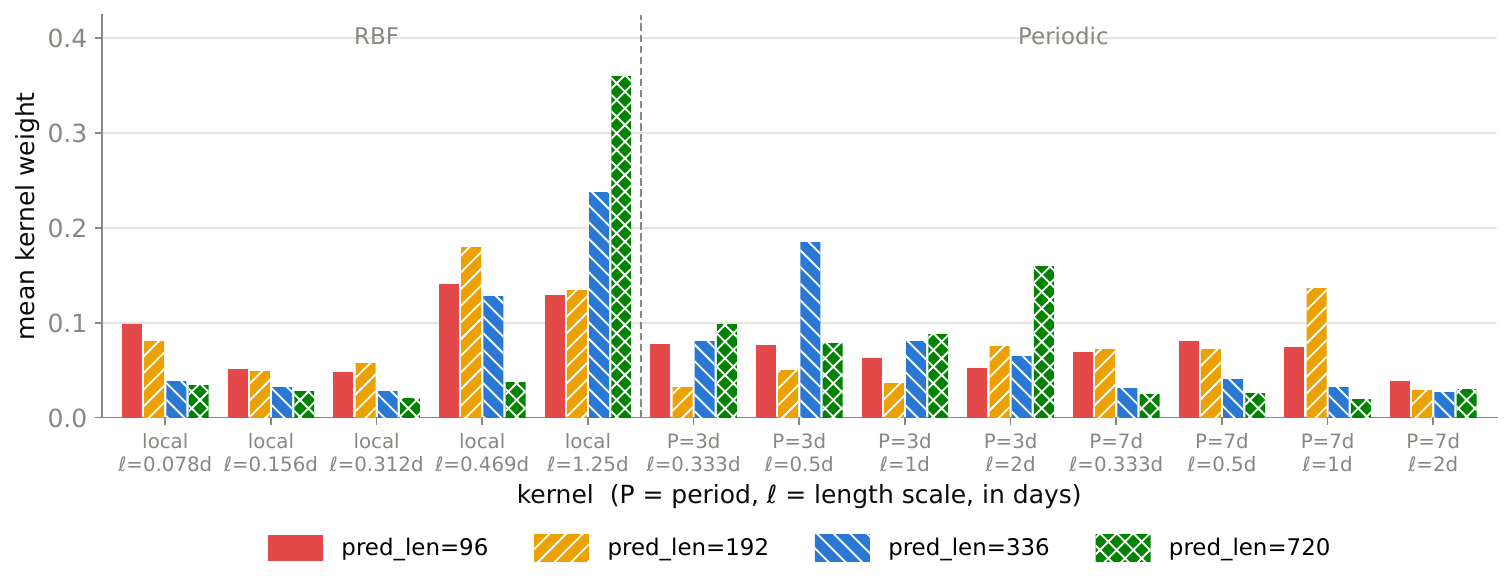}
  \caption{\textbf{BTA reveals horizon-dependent temporal dependencies on Solar.}
Mean kernel weights across Solar test samples and attention heads for
$H\in\{96,192,336,720\}$. The dashed line separates local RBF and periodic
kernels. As the forecasting horizon increases, BTA shifts weight toward
broader local dependencies and redistributes weight among periodic kernels,
with the broadest local RBF kernel becoming dominant at $H=720$.}
  \label{fig:solar_kernel_weights_by_prediction_length}
\end{figure}

 \section{Additional Hyperparameter Analysis} 
 \label{app:additional_hyperparameter_analysis} 
 Figure~\ref{fig:hyperparameter_additional} complements the ETTm1
hyperparameter analysis in Figure~\ref{fig:hyperparameter_analysis} with
results on ETTh1 and Solar.

\paragraph{Top-$k$.}
ETTh1 is nearly insensitive to the tested Top-$k$ configurations, with a
maximum relative MSE difference below $0.04\%$. Solar exhibits greater
variation, with $k=3$ attaining the lowest displayed MSE and the remaining
settings increasing MSE by up to $8.20\%$. However, the ETTh1 $k=1$ run uses
a different batch size from the $k\in\{3,5,7\}$ runs, while the Solar $k=3$
run uses a different number of patch-downsampling layers. These results should
therefore be interpreted as descriptive sensitivity analyses rather than
controlled ablations of $k$ alone.

\paragraph{Router hidden dimension.}
Sensitivity to router capacity is also dataset-dependent. ETTh1 attains its
lowest MSE at $d_r=192$, with at most a $1.23\%$ increase across the tested
dimensions, whereas Solar attains its minimum at $d_r=96$ and increases by
up to $9.30\%$ at $d_r=224$. Together with the ETTm1 results, these findings
suggest that moderate router capacity is generally sufficient, while the
optimal hidden dimension depends on the dataset.

\begin{figure}[!htbp]
\centering
\includegraphics[width=0.8\textwidth]{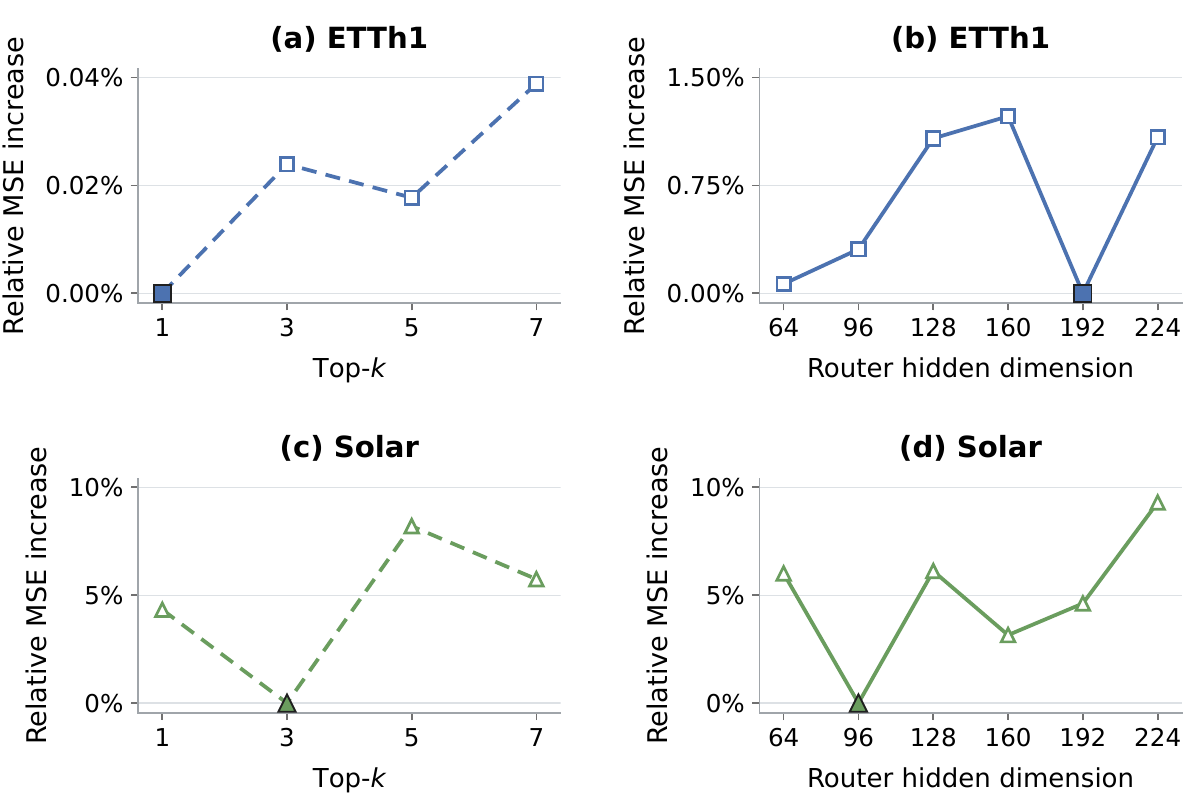}
\caption{\textbf{Additional hyperparameter sensitivity analysis on ETTh1 and
Solar.}
The top and bottom rows report results on ETTh1 and Solar, respectively,
with Top-$k$ varied in the left column and router hidden dimension $d_r$
varied in the right column. Each panel reports the relative MSE increase with
respect to the lowest MSE within that panel and uses an independent vertical
scale; filled markers indicate the minimum. Dashed curves denote the Top-$k$
sweeps, for which auxiliary configuration differences prevent a strictly
controlled comparison of $k$ alone.}
\label{fig:hyperparameter_additional}
\end{figure}

 \section{Efficiency Analysis}
\label{app:additional_efficiency_analysis}

\begin{table}[!htbp]
\centering
\caption{\textbf{Computational efficiency on ETTh1, ETTm1, and Weather.}
\emph{Ours} denotes TimeBridge+BTA. FLOPs are measured per input window, and
Size denotes the number of trainable parameters. Results are reported for
forecasting horizons $H\in\{96,192,336,720\}$.}
\label{tab:efficiency_appendix}
\fontfamily{ptm}\fontsize{10}{11}\selectfont
\renewcommand{\arraystretch}{1}
\setlength{\tabcolsep}{2.2pt}
\resizebox{0.7\columnwidth}{!}{
\begin{tabular}{cc cc cc cc cc}
\toprule
\multicolumn{2}{c}{\multirow{2}{*}{Model}}
& \multicolumn{2}{c}{\multirow{2}{*}{Ours}}
& \multicolumn{2}{c}{TimeBridge}
& \multicolumn{2}{c}{DUET}
& \multicolumn{2}{c}{DTAF} \\
\multicolumn{2}{c}{} & \multicolumn{2}{c}{}
& \multicolumn{2}{c}{\cite{liu2025timebridge}}
& \multicolumn{2}{c}{\cite{qiu2025duet}}
& \multicolumn{2}{c}{\cite{lu2026towards}} \\
\midrule
\multicolumn{2}{c}{Metric}
& FLOPs & Size & FLOPs & Size & FLOPs & Size & FLOPs & Size \\
\midrule
\multirow{4}{*}{ETTh1}
& 96  & 274.9M & 0.786M & 274.6M & 0.772M & 35.7M & 3.605M & 19.0M & 0.582M \\
& 192 & 280.0M & 1.155M & 279.7M & 1.141M & 46.7M & 5.132M & 26.8M & 1.121M \\
& 336 & 287.8M & 1.708M & 287.5M & 1.693M & 65.4M & 5.730M & 38.4M & 1.964M \\
& 720 & 308.4M & 3.182M & 308.1M & 3.168M & 72.6M & 6.981M & 69.4M & 4.176M \\
\midrule
\multirow{4}{*}{ETTm1}
& 96  & 47.4M & 0.252M & 46.9M & 0.230M & 46.1M & 5.083M & 19.0M & 0.587M \\
& 192 & 48.7M & 0.344M & 48.2M & 0.322M & 19.2M & 2.423M & 26.8M & 1.165M \\
& 336 & 50.6M & 0.482M & 50.1M & 0.460M & 47.8M & 5.206M & 38.6M & 1.980M \\
& 720 & 55.8M & 0.851M & 55.3M & 0.829M & 10.5M & 1.432M & 69.9M & 4.253M \\

\midrule

\multirow{4}{*}{Weather}
& 96  & 288.9M & 0.575M & 288.0M & 0.556M & 73.1M  & 2.794M & 57.7M  & 0.581M \\
& 192 & 296.6M & 0.759M & 295.7M & 0.740M & 206.5M & 6.710M & 80.9M  & 1.145M \\
& 336 & 308.2M & 1.036M & 307.3M & 1.017M & 209.6M & 6.784M & 115.8M & 1.974M \\
& 720 & 339.2M & 1.773M & 338.3M & 1.754M & 217.8M & 6.981M & 208.7M & 4.187M \\

\bottomrule
\end{tabular}
}
\par\smallskip\noindent M denotes millions.
\end{table}

\paragraph{Controlled comparison with the backbone.}
Because BTA is integrated into TimeBridge while retaining the remaining
backbone components and training protocol, the comparison between
TimeBridge+BTA and TimeBridge directly quantifies the additional computational
cost of BTA. Table~\ref{tab:efficiency_appendix} shows that this overhead is
small relative to the forecasting gains. On ETTh1, BTA adds approximately
$0.3$\,M FLOPs per input window ($0.10\%$ on average) and approximately $14$\,K
trainable parameters ($1.1\%$), while reducing the average MSE from $0.399$
to $0.394$ ($1.3\%$). On ETTm1, it adds approximately $0.5$\,M FLOPs
($1.00\%$) and $22$\,K parameters ($6.0\%$), while reducing the average
MSE from $0.348$ to $0.337$ ($3.2\%$). On Weather, the additional
computational cost is likewise small, while BTA further improves the
TimeBridge backbone. These results show that BTA achieves forecasting gains
with modest additional computation and capacity. These overhead estimates use the displayed counts in Table~\ref{tab:efficiency_appendix}; percentages are averaged over the four horizons.

\paragraph{Horizon-independent overhead.}
Because BTA operates on the input-side attention representation and introduces
no additional forecasting-head component, its absolute FLOP and parameter
overheads remain approximately constant across prediction horizons. 
For example, on ETTh1 at
$H=720$, BTA increases FLOPs by only $0.10\%$ and trainable parameters by
$0.44\%$.

\paragraph{Comparison with larger reference models.}
The comparison with other methods shows that larger parameter counts do not necessarily yield lower forecasting error. On ETTh1, DUET
uses $3.6$--$7.0$\,M trainable parameters, approximately
$2.2$--$4.6\times$ those of TimeBridge+BTA across the four horizons, yet
has a higher average MSE ($0.404$ versus $0.394$). This comparison shows that BTA remains competitive with larger reference models. It does not by itself isolate the effect of model capacity; a capacity-matched ablation would be needed to establish that causal attribution.

\section{More Ablation Study}
\label{app:additional_ablation_study}
Table~\ref{table_ablation_study_informer_weather} shows that each component
of BTA contributes to forecasting performance on Weather. Adding local and
periodic kernels improves vanilla self-attention, while replacing static
kernel weights with input-adaptive MoE routing yields further gains.
Incorporating the Fourier spectrum of the TE profile achieves the best overall
performance, reducing average MSE and MAE by $14.05\%$ and $7.73\%$,
respectively. These results demonstrate the complementary benefits of a richer
temporal kernel bank, input-adaptive routing, and frequency-domain information.

\begin{table}[!h]
\centering
\setlength{\tabcolsep}{3pt}
\renewcommand{\arraystretch}{0.95}
\caption{\textbf{Ablation study of BTA on Weather with Informer as the backbone.}
Starting from vanilla self-attention, we progressively introduce local kernels,
periodic kernels, input-adaptive MoE routing based on the temporal-evidence
(TE) profile, and its Fourier spectrum. Gain (\%) denotes the relative
reduction in MSE and MAE with respect to vanilla self-attention.}
\label{table_ablation_study_informer_weather}
\fontfamily{ptm}\fontsize{7}{8}\selectfont
\resizebox{0.8\linewidth}{!}{%
\begin{tabular}{c|c|cc|cc|cc|cc|cc}
\toprule

\multicolumn{2}{c|}{Attention}
& \multicolumn{2}{c|}{Self-attention}
& \multicolumn{8}{c}{Bayesian Time-aware Attention} \\

\midrule

\multicolumn{2}{c|}{Optimization}
& \multicolumn{2}{c|}{Static}
& \multicolumn{4}{c|}{Static}
& \multicolumn{2}{c|}{MoE (TE only)}
& \multicolumn{2}{c}{MoE (TE+Fourier)} \\

\midrule

\multicolumn{2}{c|}{Kernels}
& \multicolumn{2}{c|}{None}
& \multicolumn{2}{c|}{Local}
& \multicolumn{2}{c|}{Local+Periodic}
& \multicolumn{2}{c|}{Local+Periodic}
& \multicolumn{2}{c}{Local+Periodic} \\

\midrule

\multicolumn{2}{c|}{Metric}
& MSE & MAE
& MSE & MAE
& MSE & MAE
& MSE & MAE
& MSE & MAE \\

\midrule

\multirow{6}{*}{\rotatebox{90}{Weather}}
& 96
& 0.313 & 0.375
& 0.274 & 0.353
& 0.270 & 0.348
& 0.266 & 0.339
& \textbf{0.253} & \textbf{0.335} \\

& 192
& 0.377 & 0.408
& 0.333 & 0.394
& 0.320 & 0.384
& 0.313 & 0.377
& \textbf{0.299} & \textbf{0.369} \\

& 336
& 0.448 & 0.449
& 0.467 & 0.458
& 0.460 & 0.457
& 0.412 & 0.428
& \textbf{0.389} & \textbf{0.413} \\

& 720
& 0.712 & 0.604
& 0.696 & 0.597
& 0.689 & 0.590
& 0.662 & 0.583
& \textbf{0.649} & \textbf{0.577} \\

\cmidrule(lr){2-12}

& Avg
& 0.463 & 0.459
& 0.443 & 0.451
& 0.435 & 0.445
& 0.413 & 0.432
& \textbf{0.398} & \textbf{0.424} \\

& Gain(\%)
& -- & --
& 4.32 & 1.85
& 6.00 & 3.10
& 10.65 & 5.94
& $\boldsymbol{14.05}$ & $\boldsymbol{7.73}$ \\

\bottomrule
\end{tabular}
}
\par\smallskip\noindent\begin{minipage}{\linewidth}
\fontfamily{ptm}\fontsize{10}{12}\selectfont
Both MoE variants use local and periodic kernels. TE denotes the temporal-evidence profile; Fourier denotes its Fourier spectrum. Gain is relative to vanilla self-attention.
\end{minipage}
\end{table}

 \end{document}